\documentclass[journal]{IEEEtran}
\ifCLASSINFOpdf
\else
 \fi

\usepackage[usenames]{color}
\usepackage[colorlinks = true]{hyperref}
\usepackage[english]{babel}
\usepackage{bm,bbm,amsmath,amssymb,amsthm,graphicx}
\usepackage{enumitem}
\usepackage{url,cite}
\usepackage{pifont}
\usepackage[caption=false,font=normalsize,labelfont=sf,textfont=sf]{subfig}
\usepackage{algorithm}
\usepackage{algpseudocode}
\usepackage{tikz}
\usetikzlibrary{calc,patterns,decorations.pathmorphing,decorations.markings}

\newcommand  \stack[2]   {\overset{\text{#1}}{#2}}

\newcommand{\MAT}{\left[ \begin{array}}
\newcommand{\mat}{\end{array} \right]}

\newtheorem{Theorem}{Theorem}[section]

\newtheorem{Q}{Question}[section]
\newtheorem{Proposition}{Proposition}
\newtheorem{Remark}{Remark}[section]

\def \st {\operatorname*{s.t. }}

\def \a {\bm{a}}
\def \A {\mathbf{A}}
\def \AA {\mathcal{A}}

\def \b {\bm{b}}
\def \B {\mathbf{B}}

\def \CC {\mathcal{C}}

\def \CCC {\mathbb{C}}

\def \e {\bm{e}}

\def \EEE{\mathbb{E}}

\def \ft {\widetilde{f}}

\def \H {\mathbf{H}}
\def \HH {\mathcal{H}}

\def \HHb {\overline{\mathcal{H}}}
\def \HHt {\widetilde{\mathcal{H}}}
\def \I {\mathbf{I}}

\def \LL {\mathcal{L}}

\def \OO {\mathcal{O}}

\def \Pbf {\mathbf{P}}

\def \q {\bm{q}}
\def \Q {\mathbf{Q}}
\def \QQ {\mathcal{Q}}

\def \QQb {\overline{\mathcal{Q}}}

\def \RRR {\mathbb{R}}
\def \s {\bm{s}}
\def \S {\mathbf{S}}
\def \SS {\mathcal{S}}

\def \T {\mathbf{T}}
\def \TT {\mathcal{T}}

\def \u {\bm{u}}

\def \U {\mathbf{U}}

\def \v {\bm{v}}
\def \V {\mathbf{V}}

\def \w {\bm{w}}

\def \WW {\mathcal{W}}

\def \x {\bm{x}}
\def \xs {\bm{x}^\star}

\def \X {\mathbf{X}}
\def \XX {\mathcal{X}}

\def \Xs {\mathbf{X}^\star}

\def \Z {\mathbf{Z}}

\def \Zt {\widetilde{\mathbf{Z}}}

\def \sumk{\sum_{k=1}^K}
\def \sumn{\sum_{n=1}^N}
\def \summ{\sum_{m=1}^{M}}

\def \balpha {\boldsymbol{\alpha}}
\def \balphat {\widetilde{\boldsymbol{\alpha}}}

\def \bPsi {\boldsymbol{\Psi}}

\def \bSigma {\boldsymbol{\Sigma}}

\def \zero {\mathbf{0}}
\def \one {\mathbf{1}}

\begin{document}

\title{Digital Beamforming Robust to Time-Varying Carrier Frequency Offset}

\author{Shuang~Li,
        Payam Nayeri,
        and Michael B. Wakin
\thanks{SL is with the Department of Mathematics, University of California, Los Angeles, CA 90095. Email: shuangli@math.ucla.edu. PN and MBW are with the Department of Electrical Engineering, Colorado School of Mines, Golden, CO 80401. Email: \{pnayeri,mwakin\}@mines.edu.}
}

\maketitle

\begin{abstract}
Adaptive interference cancellation is rapidly becoming a necessity for our modern wireless communication systems, due to the proliferation of wireless devices that interfere with each other. To cancel interference, digital beamforming algorithms adaptively adjust the weight vector of the antenna array, and in turn its radiation pattern, to minimize interference while maximizing the desired signal power. While these algorithms are effective in ideal scenarios, they are sensitive to signal corruptions. In this work, we consider the case when the transmitter and receiver in a communication system cannot be synchronized, resulting in a carrier frequency offset that corrupts the signal. We present novel beamforming algorithms that are robust to signal corruptions arising from this time-variant carrier frequency offset. In particular, we bring in the Discrete Prolate Spheroidal Sequences (DPSS's) and propose two atomic-norm-minimization (ANM)-based methods in both 1D and 2D frameworks to design a weight vector that can be used to cancel interference when there exist unknown time-varying frequency drift in the pilot and interferer signals. Both algorithms do not assume a pilot signal is known. Noting that solving ANM optimization problems via semi-definite programs can be a computational burden, we also present a novel fast algorithm to approximately solve our 1D ANM optimization problem. Finally, we confirm the benefits of our proposed algorithms and show the advantages over existing approaches with a series of experiments.


\end{abstract}


\IEEEpeerreviewmaketitle


\section{Introduction}
\label{sec:intr}


Modern communication systems need to operate in a crowded spectrum where multiple transceivers communicate in close physical proximity and in many cases across very close channels. Interference has become a major issue, degrading communication reliability and throughput~\cite{ZLiIoT2019},~\cite{ChettriIoT2020}. To survive this hostile environment, communication systems need to adapt to the dynamic environment by adaptively cancelling interference sources and maximizing signal-to-interference-plus-noise ratio (SINR)~\cite{widrow1967adaptive},~\cite{van2004optimum}. A digital beamformer can adaptively adjust its weight vector using an array processing algorithm, which in turn shapes its radiation pattern and can place pattern nulls in the direction of interference sources, while maximizing the desired signal power~\cite{Mailloux2017}. Digital beamforming (DBF) is thus seen as a necessity for future wireless systems. While multiple approaches to cost-effective DBF hardware have been proposed over the years~\cite{DBFYang2018,DBFRoh2014,gaydos2018SDR}, one of the major challenges remaining is creating robust DBF algorithms that can operate effectively in the hostile wireless signal environment.

A wide variety of algorithms have been developed for DBF~\cite{van2004optimum,monzingo2011introduction,Blind3}, each differing in what is assumed known and how signals are modeled. Very broadly, these algorithms can be categorized into statistical versus deterministic methods. Statistical methods~\cite{van2004optimum,monzingo2011introduction} utilize a stochastic model for the desired signal, interferers, and/or noise, while deterministic methods~\cite{Blind3} do not rely on stochastic models and tend to involve eigendecompositions and related concepts from linear algebra. DBF methods can also be categorized into blind versus non-blind methods. Although there is no universal definition of blind methods~\cite{Blind1,Blind2,Blind3,Blind4}, the gist is that such methods do not rely on explicitly knowing aspects such as the desired signal, its spectral properties, the array geometry, the signal direction, and so on. Depending on what is assumed, blind beamforming problems can also have connections to blind source separation problems~\cite{Blind3,Blind4}. Non-blind methods may require explicit knowledge of a pilot signal or its spectral properties as well as other knowledge of the array geometry.

In this paper, we consider the DBF problem in the following context. Our goal is to design a weight vector (see Section~\ref{sec:prob}) that allows us to localize and separate a desired signal (``the signal'' for short) from unwanted interferers impinging on an array. Our treatment can be characterized as:
\begin{itemize}
\item Deterministic: We do not assume a stochastic model for the signal, interferers, or noise, nor do we require long-term observations for estimating auto- or cross-correlation statistics.
\item Narrowband: We assume that the signal and interferers have narrowband spectra relative to the sampling frequency. Sinusoids, which are a natural choice for pilot signals (discussed below) are one possibility for such narrowband signals, but sinusoidal structure is not explicitly required.
\item Blind: Our algorithms do not require a known pilot signal. We present two DBF algorithms in this paper. Our 1D algorithm does not require the array antennas to be equally spaced, nor does it require the antenna positions to be known. Our 2D algorithm does require the antenna positions to be known and equally spaced, but it can handle the case where the signal and interferers have overlapping temporal frequencies.
\end{itemize}
Our approach to solving this DBF problem is inspired by several existing approaches, which we describe next before giving further details on our approach.

The minimum variance distortionless response (MVDR) beamformer is a prototypical non-blind statistical beamformer~\cite{van2004optimum}. Modeling the signal and noise as stochastic processes, and given knowledge of the correlation statistics (spectral matrices) for these processes in addition to knowledge of the signal direction, MVDR is designed to output the highest possible SINR. As such it is considered the benchmark for comparing the performance of beamforming algorithms~\cite{van2004optimum,PSMI}. In practice, however, some or all of these quantities may not be available. As we discuss in Section~\ref{sec:smiintro}, when a known pilot signal and the empirical data covariance are used to estimate the signal direction and spectral properties, MVDR reduces to a technique known as sample matrix inversion (SMI).\footnote{Other algorithms in this vein include least mean square (LMS) and recursive least square (RLS). Both converge to the optimal solution but are generally slower than SMI~\cite{van2004optimum,gaydos2019adaptive}.} Computing the SMI beamformer requires only the empirical data matrix and knowledge of the pilot signal (but not the array geometry or pilot signal direction). As such, it can also be interpreted as a semi-blind deterministic method~\cite{Blind1}.

DBF algorithms such as SMI can effectively cancel interferers as long as the knowledge of the incoming pilot signal is sufficiently accurate. However, in non-ideal scenarios, signals can become corrupted and cause these algorithms fail. To address these concerns, DBF algorithms need to be robust to causes that corrupt the signal. Again, the literature on robust DBF methods is vast. Several algorithms have been developed to address beamformer mismatches in particular in the form of steering vector (SV) errors. SV errors arise from direction-of-arrival mismatch, or perturbed arrays with errors in element location, phase, or gain, and cause signal corruptions that degrades the system performance. The most widely adopted solution to these problems is diagonal loading and equivalent approaches~\cite{PSMI,SMIDL1,SMIDL2,SMIVL,SMIIVL}. Diagonal loading provides robustness to SV errors by effectively designing for a higher white noise level than is actually present, and as such is usually practical when signal-to-interference ratio (SIR) is more important than SINR. To mitigate some of these issues, variable loading has been introduced; it can improve robustness to SV errors while maintaining a desired SIR or SINR~\cite{SMIVL,SMIIVL}.

A different category of beamformer mismatch, which is the focus of this work, occurs when the transmitter and receiver cannot be synchronized and thus a carrier frequency offset exists between the anticipated pilot signal and that which arrives at the array. In~\cite{gaydos2019adaptive}, it was shown that when it is not possible to measure and correct the carrier frequency offset, algorithms such as SMI will fail. A solution to this problem using genetic algorithms was demonstrated in~\cite{gaydos2019adaptive}. More recently, a new mathematical approach based on atomic norm minimization (ANM) was proposed for this problem~\cite{li2020adaptiveACES,li2020adaptiveAWPL}. The ANM approach is naturally capable of extracting the offset frequencies that cause signal corruption, and as such, robustness is achieved with maximum SINR and without any degradation to other system performance characteristics. Both of these algorithms provided a robust solution to beamforming when transmitter and receiver are unsynchronized, which is a realistic non-ideal scenario occurring more frequently in our ever growing modern wireless infrastructure. This robust beamforming problem is categorically different than what has been studied in the past, and to the best of our knowledge, no solution to this problem is available outside~\cite{gaydos2019adaptive,li2020adaptiveACES,li2020adaptiveAWPL}.

While the works~\cite{gaydos2019adaptive,li2020adaptiveACES,li2020adaptiveAWPL} provided a robust solution to the carrier frequency offset problem in unsynchronized systems, they only considered the time-invariant case. In reality, however, frequency offset is a dynamic problem, and for practical implementation a robust time-variant solution is needed. In this work, we accommodate an unknown time-varying frequency offset in the pilot signal. Indeed, we drop entirely the requirement that the pilot signal be known, and we assume only that the pilot signal is narrowband. As we show (see Section~\ref{sec:DPSSintro}), this assumption accommodates (but does not require) the common choice of a sinusoidal pilot signal, and it allows for a moderate amount of unknown time-varying drift in the carrier frequency of this sinusoid. In this way, our treatment is more general (blind) and robust compared to classical SMI.

Our approach for solving the blind beamforming problem relies on an optimization framework using ANM. ANM allows for a signal or data matrix to be decomposed into sparse components arising from a certain {\em atomic set} of prototype signals or matrices. In compressive sensing, finite atomic sets of on-grid sinusoids have been used for recovering signals from partial information~\cite{candes2006compressive,candes2008introduction}. In line spectrum estimation, continuous atomic sets of off-grid sinusoids have been used for precisely identifying the active frequencies in a finite-length sampled signal without suffering from the ``leakage'' effects that confront conventional Fourier-based methods~\cite{tang2013compressed}. In array processing, such atomic sets have been used to provide joint sparse frequency decompositions of multi-channel data matrices~\cite{li2015off,yang2016exact}. Unfortunately, the atomic sets in these works are not rich enough to represent the signals encountered when time-varying frequency drift is present.

In this paper, we expand the array processing ANM frameworks using a novel atomic set for data matrices. Our work expands on the joint sparse frequency model~\cite{li2015off,yang2016exact} by modulating each sinusoid by an unknown vector in a low-dimensional basis of Discrete Prolate Spheroidal Sequences (DPSS's)~\cite{slepian1978prolate,davenport2012compressive,zhu2017approximating}. DPSS's have been shown to provide excellent compressed representations of sampled narrowband signals. We therefore bring DPSS's into the ANM framework to capture the unknown time-varying frequency drift component in the pilot and interferer signals.

We present two novel ANM frameworks in this paper. Our first ``1D'' framework does not require the antennas to be equally spaced, and moreover it does not require the antenna positions to be known. Our second ``2D'' framework does require the antenna positions to be known and equally spaced, but it can handle the case where the pilot signal and interferers have overlapping temporal frequencies. The 2D framework exploits sparse structure in both the temporal and angular frequency domains, whereas the 1D framework exploits sparse sinusoidal structure only in the temporal frequency domain. Both frameworks can accommodate time-varying frequency drift component in the pilot and interferer signals.

Finally, we note that most ANM optimization problems can be formulated as a semi-definite program (SDP), solving this SDP can be a computational burden. In this paper, we also present a novel fast algorithm for approximately solving our 1D ANM optimization problem. We provide extensive experimental tests to confirm the benefits of our 1D and 2D ANM frameworks compared to conventional SMI, and we illustrate the computational speedup provided by our fast 1D algorithm. We leave the work of developing a fast algorithm for the 2D ANM method as our future work.

{\bf {\em Notations.}} Throughout this paper, we use superscripts $^\top$, $^*$, $^H$, and $^\dagger$ to denote transpose, conjugate,  conjugate transpose, and pseudo-inverse, respectively. We use $\odot$ and $\otimes$ to denote the elementwise multiplication and Kronecker product, respectively. For a positive integer $K$, denote set $[K] \triangleq \{1,\cdots,K\}$. We use $\one_L\in\RRR^L$ and $\I_L\in\RRR^{L\times L}$ to represent a length-$L$ vector with entries being all 1's and an $L\times L$ identity matrix, respectively. Finally, we let $\Pbf_{\Q} \triangleq \I - \Q\Q^H$ denote the orthogonal projection onto the orthogonal complement of the column space of an orthogonal matrix $\Q$.

The remainder of this paper is organized as follows. We introduce some background on SMI, ANM, and DPSS in Section~\ref{sec:prel} and formulate the problem of interest in Section~\ref{sec:prob}. We then present the proposed 1D method ``ANM+DPSS+SMI'', fast 1D method ``IVDST+DPSS+SMI'', and 2D method ``2D-ANM+DPSS+SMI'' in Sections~\ref{sec:method_sdp}, \ref{sec:method_fast}, and \ref{sec:2d}, respectively. We conduct a series of experiments to test the proposed methods in Section~\ref{sec:nume}. Finally, we conclude this work and discuss a few possible future directions in Section~\ref{sec:conc}.

\section{Preliminaries}
\label{sec:prel}

\subsection{Sample matrix inversion (SMI)}
\label{sec:smiintro}

Consider an $N$-element linear antenna array. The optimal MVDR beamformer computes the complex valued weighting vector $\w\in\CCC^N$ for the array elements such that the interference signals are minimized subject to a constraint~\cite{van2004optimum}. The algorithm minimizes $\EEE\left\{\w^H \S_x \w \right\}$ subject to $\w^H \textbf{asv} = 1$, 
where the constraint ensures that a plane wave with a known array steering vector, $\textbf{asv}$, 
will not be distorted. Here, $\EEE$ is the expected value, and $\S_x$ is the signal spectral matrix. In practice $\S_x$ is not available but the maximum likelihood of the matrix can be estimated from incoming data as
$\S_x= \frac{1}{M}\sum_{m=1}^{M} \x_m \x_m^H$,
where $\{\x_m\}_{m=1}^M$ denote the signal snapshots and $M$ is the number of snapshots.
The array steering vector is $\textbf{asv}=e^{jk_0 \sin(\theta_s)\q}$, where $\q\in\RRR^{1\times N}$ denotes the element positions, $\theta_s$ is the signal direction, and $k_0$ is the plane wave wavenumber. In SMI, where a pilot signal is available, the signal steering vector, $\v_s$, is used in place of the array steering vector, and its maximum likelihood is estimated by the cross-correlation between the signal, $\x$, and the known pilot signal, $\s\in\CCC^M$, as
$\v_s= \frac{1}{M}\sum_{m=1}^{M} \x_m s_m^H,$
where $s_m$ is the $m$-th entry of $\s$.
The SMI element weights are then given by $\w^H \triangleq \frac{\v_s^H\S_x^{-1}}{\v_s^H\S_x^{-1}\v_s}$. In this equation, the presence of the $\S_x^{-1}$ term gives rise to the name sample matrix inversion.
By rearranging the signal snapshots $\{\x_m\}_{m=1}^M$ as columns of a data matrix $\X^H$, it follows that the SMI element weights can be rewritten as
\begin{equation}
\w = (\X^H\X)^{-1}\X^H \s = \X^\dagger \s, \label{eq:SMIorig}
\end{equation}
where we ignored the normalization factor $\v_s^H\S_x^{-1}\v_s$. Though we arrived at the solution~\eqref{eq:SMIorig} using statistical arguments, we note that $\w = \X^\dagger \s$ can also be interpreted as the least-squares solution to the purely deterministic problem of finding the weight vector that best solves $\X\w = \s$, i.e., best localizes the known pilot signal $\s$~\cite{Blind1}.


\subsection{Atomic norm minimization (ANM)}


The atomic norm is a generalization of the $\ell_1$ norm (commonly used in compressive sensing and sparse signal recovery) and nuclear norm (commonly used in low-rank matrix recovery) from finite, discrete dictionaries to infinite, continuously-parameterized dictionaries~\cite{chandrasekaran2012convex,tang2013compressed}. As in $\ell_1$ norm minimization and nuclear norm minimization, atomic norm minimization (ANM) can promote sparsity when solving an inverse problem. Thus, it has been widely studied in many signal processing applications, including line spectral estimation and array signal processing, for localizing the off-grid temporal or angular frequencies of sinusoidal components~\cite{chandrasekaran2012convex,candes2014towards,tang2013compressed,li2015off,yang2016exact,li2018atomic,xie2017radar,li2019atomic}.

To be more precise, consider the following spectrally sparse signal $\x\in\CCC^M$ with $K$ different active frequencies
\begin{align}
    \x = \sumk c_k \a(f_k) \in \CCC^M, \label{eq:origsignal}
\end{align}
where the $c_k$'s denote complex coefficients, and
\begin{align}
 \a(f) \triangleq [1~e^{j2\pi f} ~\cdots~ e^{j2\pi f (M-1)}]^\top\in\CCC^M
 \label{eq:def_af}
 \end{align}
  is a vector containing $M$ uniform samples of a complex exponential signal with frequency $f\in[0,1)$.
One can then define
$$
\AA \triangleq \{\a(f): f\in[0,1)\}
$$
as an {\em atomic set} that consists of all possible such complex exponentials. Note that the spectrally sparse signal $\x$ in~\eqref{eq:origsignal} can be represented with only $K$ atoms from the above atomic set $\AA$.
The corresponding induced {\em atomic norm} is then defined as
$$
\|\x\|_{\AA} \triangleq \inf\left\{ \sumk |c_k|: \x = \sumk c_k \a(f_k)\right\}.
$$

\subsection{Discrete prolate spheroidal sequences (DPSS’s)}
\label{sec:DPSSintro}

For integers $M$ and $L \in \{1,2,\dots,M\}$ and a positive number $W < \frac{1}{2}$, let $\S_{M,W}\in \RRR^{M\times L}$ denote a matrix containing the first $L$ of the $M$-dimensional discrete prolate spheroidal sequence (DPSS) vectors with digital half bandwidth $W$~\cite{slepian1978prolate,davenport2012compressive,zhu2017approximating}. DPSS's are distinctive in that each column of $\S_{M,W}$, if zero padded, has a Discrete-Time Fourier Transform (DTFT) highly concentrated in the normalized frequency domain $[-W,W]$. In fact, the first column vector is precisely the time-limited signal (with $M$ nonzero entries) whose DTFT is optimally concentrated in $[-W,W]$, the second column vector is the time-limited signal whose DTFT is most concentrated in $[-W,W]$ subject to being orthogonal to the first column, etc.
  
DPSS's are particularly valuable for providing efficient representations of sampled sinusoids and narrowband signals. For example, as indicated in the following theorem, any complex exponential signal $\a(f)$ defined in~\eqref{eq:def_af} with $f\in[-W, W]$ can be well represented in the column span of the DPSS basis $\S_{M,W}$ when $L \approx 2MW$~\cite{davenport2012compressive,zhu2017approximating}.

\begin{Theorem}\cite{zhu2017approximating}
Fix $W\in(0,\frac{1}{2})$ and $\epsilon\in(0,1)$. Let $L = 2MW(1+\epsilon)$. Then there exist constants $C_1,C_2$ such that
\begin{equation*}\begin{split}
||\Pbf_{\S_{M,W}}\a(f)||_2 \leq C_1M^{5/4}e^{-C_2M}, ~ \forall~ f\in[-W, W]
\end{split}\end{equation*}
holds for all $M \geq M_0$. Here, $C_1,C_2$ are numerical constants that may depend on $W$ and $\epsilon$. 
\label{thm:pure tone by DPSS}
\end{Theorem}

In summary, this result states that ($i$) the effective dimensionality of the subspace spanned by $\{ \a(f) \}_{f\in[-W, W]}$ is $2MW$, and ($ii$) the first $L \approx 2MW$ DPSS vectors provide a basis for this subspace. Sinusoids concentrated in a narrow band away from baseband, e.g., $f\in[a, b]$, can be equally well-captured by taking $W = \frac{b-a}{2}$ to be the half bandwidth and modulating each of the DPSS vectors by $\a(f^o)$ where $f^0 = \frac{a+b}{2}$ is the center frequency. Finally, the high quality of DPSS approximations applies not only to sampled sinusoids, but also to sampled narrowband signals (since these can be viewed as linear combinations of sampled sinusoids within a narrow range of frequencies). As discussed more formally in~\cite{davenport2012compressive,zhu2017approximating}, the finite-length sample vectors arising from sampling random baseband (or bandpass) analog signals can be well approximated using a DPSS (or modulated DPSS) basis. Again the effective dimensionality is $L \approx 2MW$, where $M$ is the length of the sample vector and $W$ is the digital half bandwidth.

\section{Problem Formulation}
\label{sec:prob}

In this work, we consider a conventional linear array with $N$ antenna elements. We collect $M$ snapshots on each antenna element and model the data matrix as
\begin{align}
\Xs = \sumk \s_k \textbf{asv}(\theta_k) \in \CCC^{M\times N},	
\label{eq:data}
\end{align}
where $\s_1\in \CCC^M$ denotes the desired signal and $\{\s_k\in \CCC^M\}_{k=2}^K$ denote the interferers. Here,
\begin{align*}
          \textbf{asv}(\theta_k) \triangleq e^{jk_0 \sin(\theta_k)\q}
\end{align*}
is a length-$N$ row vector representing the array steering vector associated with the desired signal (if $k=1$) or interferers (if $k>1$). $\theta_1$ and $\{\theta_k\}_{k=2}^K$ denote the angle to the desired signal and interferers, respectively, $\q\in\RRR^{1\times N}$ denotes the element positions, and $k_0$ is the plane wave wavenumber.

We do not assume the desired signal, the interferers, or any of their directions are known. Moreover, our 1D algorithm does not require the array antennas to be equally spaced, nor does it require the element positions $\q$ to be known.

Given the data matrix $\Xs$, our goal is to construct a weight vector $\w$ such that $\Xs \w \approx \s_1$ as closely as possible. Inspired by SMI (which assumes $\s_1$ is known), our approach involves first estimating $\s_1$ from $\Xs$ and then plugging this estimate (which we call $\widetilde{\s}_1$) into the SMI expression:
\begin{align}
\w = {\Xs}^\dagger \widetilde{\s}_1.
\label{eq:weight_fix}
\end{align}
To facilitate the estimation of $\s_1$ from $\Xs$, we assume that the desired signal and all of the interferers are narrowband signals whose DTFT's are each concentrated in a small subset of the digital bandwidth $[-\frac{1}{2},\frac{1}{2}]$. More specifically, employing the DPSS's discussed in Section~\ref{sec:DPSSintro}, we assume that each $\s_k$ can be modeled as follows:
\begin{equation}
\s_k \approx \a(f_k^o) \odot (\S_{M,W} \balpha_k) ~\text{for}~ k=1,2,\dots,K,
\label{eq:skassume}
\end{equation}
where $f_k^o$ is a possibly unknown carrier frequency, $\a(f_k^o)$ is a sampled complex exponential with frequency $f_k^o$ as defined in~\eqref{eq:def_af}, $\S_{M,W}\in\RRR^{M\times L}$ is an $L$-dimensional DPSS basis (we discuss the choice of $W$ and $L \approx 2MW$), and $\balpha_k \in \CCC^L$ denote the corresponding unknown coefficients. This narrowband assumption is quite flexible and accommodates the following cases:
\begin{itemize}
\item {\bf Case 1:} $\s_k$ is a pure sinusoid (complex exponential) with a fixed frequency. In this case, \eqref{eq:skassume} holds with equality by taking $f_k^o$ to be the frequency of the sinusoid and letting $\balpha_k = 0$.
\item {\bf Case 2:} $\s_k$ is a sinusoid (such as a pilot signal when $k=1$), modulated by a frequency offset that drifts slightly over time. In this case, $\s_k$ can be viewed as a frequency modulated (FM) signal. Carson's bandwidth rule for FM signals~\cite{wakin2012study} guarantees that a majority of the signal power will concentrate in a narrow band around the signal's carrier frequency, with the half bandwidth $W$ proportional to the peak frequency deviation plus the temporal bandwidth of the modulating signal. Therefore, $\a(f_k^o)$ will again capture the dominant sinusoidal component, and following the discussion in Section~\ref{sec:DPSSintro}, $\S_{M,W} \balpha_k$ will approximately capture the deviations from this frequency.
\item {\bf Case 3:} $\s_k$ is a bandpass signal, oversampled with respect to its Nyquist rate. In this case, the DTFT of $\s_k$ will concentrate in some subset $[a,b]$ of the digital bandwidth $[-\frac{1}{2},\frac{1}{2}]$. Following the discussion in Section~\ref{sec:DPSSintro}, the approximation in~\eqref{eq:skassume} then holds by taking $f_k^o = \frac{a+b}{2}$ to be the midpoint of this band and $W$ to be greater than or equal to its half width $\frac{b-a}{2}$.
\item {\bf Case 4:} $\s_k$ is a bandpass signal, oversampled with respect to its Nyquist rate, modulated by a frequency offset that drifts slightly over time. The approximation in~\eqref{eq:skassume} then holds by combining the arguments from Case~2 and Case~3: because multiplication in time corresponds to convolution in frequency, the product of two narrowband signals will remain (relatively) narrowband.
\end{itemize}
The reader will notice that we do not focus on quantifying the necessary degree of accuracy in the approximation~\eqref{eq:skassume} in this work. Nevertheless, this assumption is quite broad. We use numerical experiments in Section~\ref{sec:nume} to confirm that this approximation is sufficiently valid to enable the success of our proposed technique in a range of interesting cases.

Finally, we note that the above idea of using~\eqref{eq:weight_fix} to design a weight vector has already been studied in~\cite{li2020adaptiveACES,li2020adaptiveAWPL} under the assumption that each $\s_k$ is a pure sinusoid modulated by a fixed (not time-varying) frequency offset. Those works propose two ANM-based methods to estimate the frequency offset from either $\Xs$ or its noisy measurements. However, as is shown in the simulations, the weight vector obtained by directly applying this idea cannot cancel the interferers when there exist {\em time-varying} frequency offsets. To address this problem, we extend the ANM formulation using the novel assumption in~\eqref{eq:skassume}.


\section{Proposed Method: ``ANM+DPSS+SMI''}
\label{sec:method_sdp}


We now describe our procedure for constructing the estimate $\widetilde{\s}_1$ based on the data matrix $\Xs$. We will estimate $\s_1$ by estimating both $f_1^o$ and $\balpha_1$ and plugging these estimates into~\eqref{eq:skassume}. Indeed, we can estimate any and all of the frequencies $f_k^o$ and vectors $\balpha_k$ using the procedure below; we assume only that, among these estimates, the correct one can be associated with the desired signal (the $k=1$ term can be correctly selected).

Inspired by a recent work on multi-band signal recovery~\cite{helland2019super}, we rewrite the data matrix in~\eqref{eq:data} as
\begin{align*}
\Xs & = \sumk \s_k \textbf{asv}(\theta_k)
 \stack{\ding{172}}{\approx} \sumk \left( \a(f_k^o) \odot (\S_{M,W} \balpha_k)  \right)\textbf{asv}(\theta_k)	\\
& = \sumk \|\balpha_k\|_2 \|\textbf{asv}(\theta_k) \|_2 ( ( \S_{M,W} \text{sign}(\balpha_k)) \\
&~~~~~~~~~~~~~~~~~~~~~~~~~~~~~~~~~~\odot \a(f_k^o)   ) \text{sign}(\textbf{asv}(\theta_k))\\
& \stack{\ding{173}}{=} \sumk c_k \left( ( \S_{M,W} \balphat_k)\odot \a(f_k^o)   \right) \b_k^H \\
&= \sumk c_k [( ( \S_{M,W} \balphat_k)\odot \a(f_k^o)) b_{1k}^* \\
&~~~~~~~~~~~~~~~~~~~~~~~~~~~~~~\cdots ( ( \S_{M,W} \balphat_k)\odot \a(f_k^o)  ) b_{Nk}^* ]\\
& \stack{\ding{174}}{=}\sumk c_k \left[\left( \S_{M,W}\odot \left( \a(f_k^o) \balphat_k^\top b_{1k}^*\right)\right)\one_L  \right.\\
&~~~~~~~~~~~~~~~~~~~~~~~\cdots \left. \left( \S_{M,W}\odot \left( \a(f_k^o) \balphat_k^\top b_{Nk}^*\right)\right)\one_L \right]\\
& = \left[\left( \S_{M,W}\odot \left(\sumk c_k \a(f_k^o) \balphat_k^\top b_{1k}^*\right)\right)\one_L \right. \\
&~~~~~~~~~~~ \left. \cdots \left( \S_{M,W}\odot \left(\sumk c_k \a(f_k^o) \balphat_k^\top b_{Nk}^*\right)\right)\one_L \right],
\end{align*}
where the sign vector is defined as $\text{sign}(\x)\triangleq \frac{\x}{\|\x\|_2}$. \ding{172} follows from our assumption~\eqref{eq:skassume}. \ding{173} follows by denoting $c_k = \|\balpha_k\|_2 \|\textbf{asv}(\theta_k) \|_2 $, $\balphat_k = \text{sign}(\balpha_k)$, and $\b_k^H = \text{sign}(\textbf{asv}(\theta_k))$. Note that $\{\balphat_k\}_{k=1}^K$ and $\{\b_k\}_{k=1}^K$ are normalized vectors with $\|\balphat_k\|_2 = \|\b_k\|_2 = 1$. We use $b_{nk}$ to denote the $n$-th entry of $\b_k$. $\one_L\in\RRR^L$ is a length-$L$ vector with entries being all 1's. Finally, \ding{174} follows from the following equalities
\begin{align*}
( \S_{M,W} \balphat_k)\odot \a(f_k^o) &= \text{diag}(\a(f_k^o)) \S_{M,W} \balphat_k  \\
&= \text{diag}(\a(f_k^o)) \S_{M,W} \text{diag}(\balphat_k) \one_L\\
&= 	\left( \S_{M,W}\odot \left( \a(f_k^o) \balphat_k^\top \right)\right)\one_L.
\end{align*}

The above observation inspires us to define a tensor $\XX^\star\in\CCC^{M\times L\times N}$ with the $n$-th frontal slice being $\XX^\star_{::n} = \sumk c_k \a(f_k^o) \balphat_k^\top b_{nk}^*$, namely,
\begin{align}
\XX^\star \triangleq \sumk c_k \A(f_k^o) \circledast \H_k
\label{eq:def_tensor}
\end{align}
 with $\A(f_k^o) \triangleq [\a(f_k^o)\cdots \a(f_k^o)] \in \CCC^{M\times N}$ and $\H_k \triangleq \balphat_k \b_k^H \in \CCC^{L\times N}$. Here, we use $\A\circledast \B \in \CCC^{M\times L \times N}$ to denote the ``reshaped Khatri-Rao product'', which is defined as $[\A \circledast \B]_{::n} \triangleq \a_n \b_n^\top$ for any two matrices $\A = [\a_1 \cdots \a_N] \in \CCC^{M\times N}$ and $\B = [\b_1 \cdots \b_N] \in \CCC^{L\times N}$.
By defining a linear operator $\LL:\CCC^{M\times L\times N} \rightarrow \CCC^{M \times N}$ as
\begin{align}
\LL(\XX) \triangleq [(\S_{M,W}\odot \XX_{::1})\one_L \cdots  (\S_{M,W}\odot \XX_{::N})\one_L]	,
\label{eq:def_L}
\end{align}
we can rewrite the data matrix $\Xs$ as the linear measurements of a third-order tensor $\XX^\star$ obtained from the above linear operator $\LL:\CCC^{M\times L\times N} \rightarrow \CCC^{M \times N}$, namely,
\begin{align*}
\Xs = \LL(\XX^\star).	
\end{align*}

We now define an {\em atomic set} as
\begin{align*}
\AA \triangleq \left\{\A(f^o) \circledast \H:~f^o\in[0,1),~ \H\in\CCC^{L\times N},~ \|\H\|_F=1\right\}.	
\end{align*}
The induced {\em atomic norm} is then defined as
\begin{align*}
\|\XX\|_{\AA} \triangleq \inf \left\{ \sumk c_k: \XX = \sumk c_k \A(f_k^o) \circledast \H_k, c_k\geq 0    \right\},	
\end{align*}
which is approximately equivalent to the optimal value of the following semidefinite program (SDP)~\cite{helland2019super}:
\begin{align*}
\inf_{\substack{\U\in\CCC^{M\times N}\\ \WW\in\CCC^{L \times L \times N}}} &\frac 1 2 u + \frac  1 2 \sumn \text{Tr}(\WW_{::n}) \\
 \st~~~~
&\left[\begin{array}{cc}
\text{Toep}(\u_n) & \XX_{::n}\\
\XX_{::n}^H & \WW_{::n}
\end{array}\right] \succeq \zero,~ u_{1n}=u,~ \forall~n\in[N].
\end{align*}
Here, we use Tr$(\cdot)$ to denote the trace of a square matrix and Toep$(\u)$ to denote the Hermitian Toeplitz matrix with the vector $\u$ as its first column. $\u_n$ and $u_{1n}$ denote the $n$-th column of $\U$ and the first entry in $\u_n$, respectively.

It can be seen that the tensor defined in~\eqref{eq:def_tensor} can be represented with only $K$ atoms from the above atomic set. To promote sparsity and recover the tensor $\XX^\star$ from its linear measurements $\Xs = \LL(\XX^\star)$, we propose to solve the following ANM problem
\begin{align}
 \min_{\XX} ~\|\XX\|_{\AA} \quad \st~	\|\Xs - \LL(\XX)\|_F\leq \varepsilon,
\label{eq:ANM}
\end{align}
where $\varepsilon$ is a parameter that controls the data fidelity term. Note that the above ANM~\eqref{eq:ANM} is approximately equivalent to the following SDP
\begin{equation}
\begin{aligned}
\inf_{\substack{\U\in\CCC^{M\times N}\\ \WW\in\CCC^{L \times L \times N}\\ \XX\in\CCC^{M \times L \times N}}} &\frac 1 2 u + \frac  1 2 \sumn \text{Tr}(\WW_{::n}) \\
 \st~~~~
&\left[\begin{array}{cc}
\text{Toep}(\u_n) & \XX_{::n}\\
\XX_{::n}^H & \WW_{::n}
\end{array}\right] \succeq \zero,~ u_{1n}=u,~ \forall~n\in[N],\\
&\|\Xs - \LL(\XX)\|_F\leq \varepsilon,
\label{eq:SDPanm}
\end{aligned}
\end{equation}
which can be solved by any off-the-shelf SDP solver, such as CVX~\cite{grant2008cvx}. As primal-dual algorithms are used in CVX, it can return both the primal solution and the dual solution.

Denote $\QQ^\star$ as the dual solution to~\eqref{eq:ANM}. We construct a dual polynomial as
\begin{align}
q(f^o) \triangleq \| \bPsi(f^o,\QQ^\star) \|_F	
\label{eq:dualpoly}
\end{align}
with
\begin{align*}
\bPsi(f^o,\QQ^\star)  = \left[\begin{array}{ccc}
{\QQ^\star_{::1}}^H \a(f^o)&\cdots &	{\QQ^\star_{::N}}^H \a(f^o)
\end{array}
\right]\in \CCC^{L\times N}.
\end{align*}
As is shown in~\cite{helland2019super}, one can identify $\{f_k^o\}_{k=1}^K$ by localizing the places where $q(f^o)=1$. However, in this work, as we will see in the simulations, the dual polynomial has several plateaus. Therefore, we will use the k-means method to cluster the $f^o$ with $q(f^o) \geq \gamma_0$ and get the estimated frequencies $\ft_k^o$.

Denote $\widehat{\XX}$ as the primal optimal solution to~\eqref{eq:ANM}. Note that
\begin{align*}
\widehat{\XX}_{::1} &= \sumk c_k \a(\ft_k^o) \widehat{\balphat}_k^\top \widehat{b}_{1k}^* \\
& = [\a(\ft_1^o)\cdots \a(\ft_K^o)]
\left[\begin{array}{ccc}
c_1 \widehat{b}_{11}^*&&\\
&\ddots&\\
&&c_K \widehat{b}_{1K}^*	
\end{array}\right]
\left[\begin{array}{c}
\widehat{\balphat}_1^\top\\
\vdots\\
\widehat{\balphat}_K^\top	
\end{array}\right].
\end{align*}
Denote $\A_f \triangleq [\a(\ft_1^o)\cdots \a(\ft_K^o)]$. Once we have the estimated frequencies $\{\ft_k^o\}_{k=1}^K$, we can then estimate the sign of the coefficient $\widehat{\balphat}_1$ as
\begin{align*}
\text{sign}(\widehat{\balphat}_1)
= \frac{(\A_f^\dag \widehat{\XX}_{::1} )^\top\e_1}{\| (\A_f^\dag \widehat{\XX}_{::1} )^\top\e_1 \|_2},
\end{align*}
where $\e_1\in\RRR^K$ is the first column of the $K\times K$ identity matrix $\I_K$. Recall that $\balphat_k = \text{sign}(\balpha_k)$. Therefore, we get
\begin{align}
\text{sign}(\widehat{\balpha}_1)
=\frac{(\A_f^\dag \widehat{\XX}_{::1} )^\top\e_1}{\| (\A_f^\dag \widehat{\XX}_{::1} )^\top\e_1 \|_2}.
\label{eq:sign_a1}
\end{align}
Finally, we estimate $\s_1$ as 
\begin{align}
\widetilde{\s}_1 = \a(\ft_1^o) \odot (\S_{M,W} \text{sign}(\widehat{\balpha}_1))
\label{eq:sest1d}
\end{align}
and we plug this estimate into~\eqref{eq:weight_fix} to estimate the weight vector $\w$. We summarize the proposed ``ANM+DPSS+SMI'' method in Algorithm~\ref{alg:ANM_DPSS_SMI}.


\begin{algorithm}[H]
\caption{ANM+DPSS+SMI}\label{alg:ANM_DPSS_SMI}
\begin{algorithmic}[1]
\Procedure{Input}{the data matrix $\Xs$ and the number of DPSS basis vectors $L$}
\State create DPSS basis $\S_{M,W}\in\RRR^{M\times L}$
\State compute the primal solution $\widehat{\XX}$ and the dual solution $\QQ^\star$ by solving the SDP~\eqref{eq:SDPanm}
\State form the dual polynomial $q(f^o)$ as in~\eqref{eq:dualpoly}
\State use the k-means method to cluster the $f^o$ with $q(f^o) \geq \gamma_0$ and get $\ft_k^o$; set $k=1$ according to the desired frequency 
\State compute $\text{sign}(\widehat{\balpha}_1)$ according to~\eqref{eq:sign_a1}
\State estimate $\s_1$ as in~\eqref{eq:sest1d} and construct the weight vector as in~\eqref{eq:weight_fix}
\State \textbf{return} the weight vector $\w$
\EndProcedure
\end{algorithmic}
\end{algorithm}


\begin{Remark}{(Computational complexity.)} \label{rmk:comput}
Note that most SDP solvers, including SDPT3~\cite{tutuncu2001sdpt3}, use the interior-point method and need to solve a  system of linear equations when computing the Newton direction, which can be very expensive for large-size problems. Here, the $N$ positive semi-definite (PSD) constraints in~\eqref{eq:SDPanm} can be combined into one PSD constraint on a big block diagonal matrix of size $N(M+L) \times N(M+L)$. Therefore, the overall computational complexity of solving \eqref{eq:SDPanm} with SDP solvers is $\OO(N^{3.5}(M+L)^{3.5})$, which limits the use of the proposed ``ANM+DPSS+SMI'' method in large-scale problems.
\end{Remark}	

In Remark~\ref{rmk:comput}, one may notice that the combined PSD constraint is on a block diagonal matrix, rather than a dense matrix. Thus, a natural question arises:
\begin{itemize}
\item[] {\em Is it possible that the computational complexity of solving~\eqref{eq:SDPanm} is less than $\OO(N^{3.5}(M+L)^{3.5})$ due to the sparsity that exists in the PSD constraint?}  	
\end{itemize}
To the best of our knowledge, there is no closed-form expression for the complexity. However, we are able to empirically answer this question by conducting the following experiment. In particular, we collect the running time needed for solving the following SDP
\begin{equation}
\min~ x ~~ \st ~~ \A_x \succeq \zero,
\label{eq:SDP_simple}	
\end{equation}	
where $\A_x\in\RRR^{N_s\times N_s}$ is a matrix with entries from the set $\{0,x,\sqrt{N_s}\}$. First, we test the case when $\A_x$ is a dense matrix, i.e.,
\begin{align}
	\A_x = \left[\begin{array}{cccc}
x & \sqrt{N_s} & \cdots & \sqrt{N_s}	\\
\sqrt{N_s} & x & \cdots & \sqrt{N_s}	\\
\vdots & \vdots & \ddots & \vdots\\
\sqrt{N_s} & \sqrt{N_s} & \cdots & x
\end{array}
\right].
\label{eq:A_dense}
\end{align}
The time needed for solving~\eqref{eq:SDP_simple} with $\A_x$ given in~\eqref{eq:A_dense} for a variety of $N_s$ is illustrated in Figure~\ref{fig:test_SDP_time} (the blue line with circles). Second, we test several cases when $\A_x$ is a block diagonal matrix, specifically with $N_b$ of blocks on the diagonal where $N_b = 10,~ 20,~ 50$ and $100$. We generate each block on the diagonal as an $\frac{N_s}{N_b} \times \frac{N_s}{N_b}$ matrix with diagonal entries being $x$ and off-diagonal entries being $\sqrt{N_s}$. The running times are presented as red lines in Figure~\ref{fig:test_SDP_time}. It can be seen that solving the SDP~\eqref{eq:SDP_simple} with a sparse PSD constraint matrix is only slightly faster when using a dense PSD constraint matrix. The running time for both cases is approximately on the same order as the constraint matrix size, i.e., $\OO(N_s^{2.5})$, which is  better than the complexity described in Remark~\ref{rmk:comput}.

\begin{figure}[t]
\centering
\includegraphics[width=3.0in]{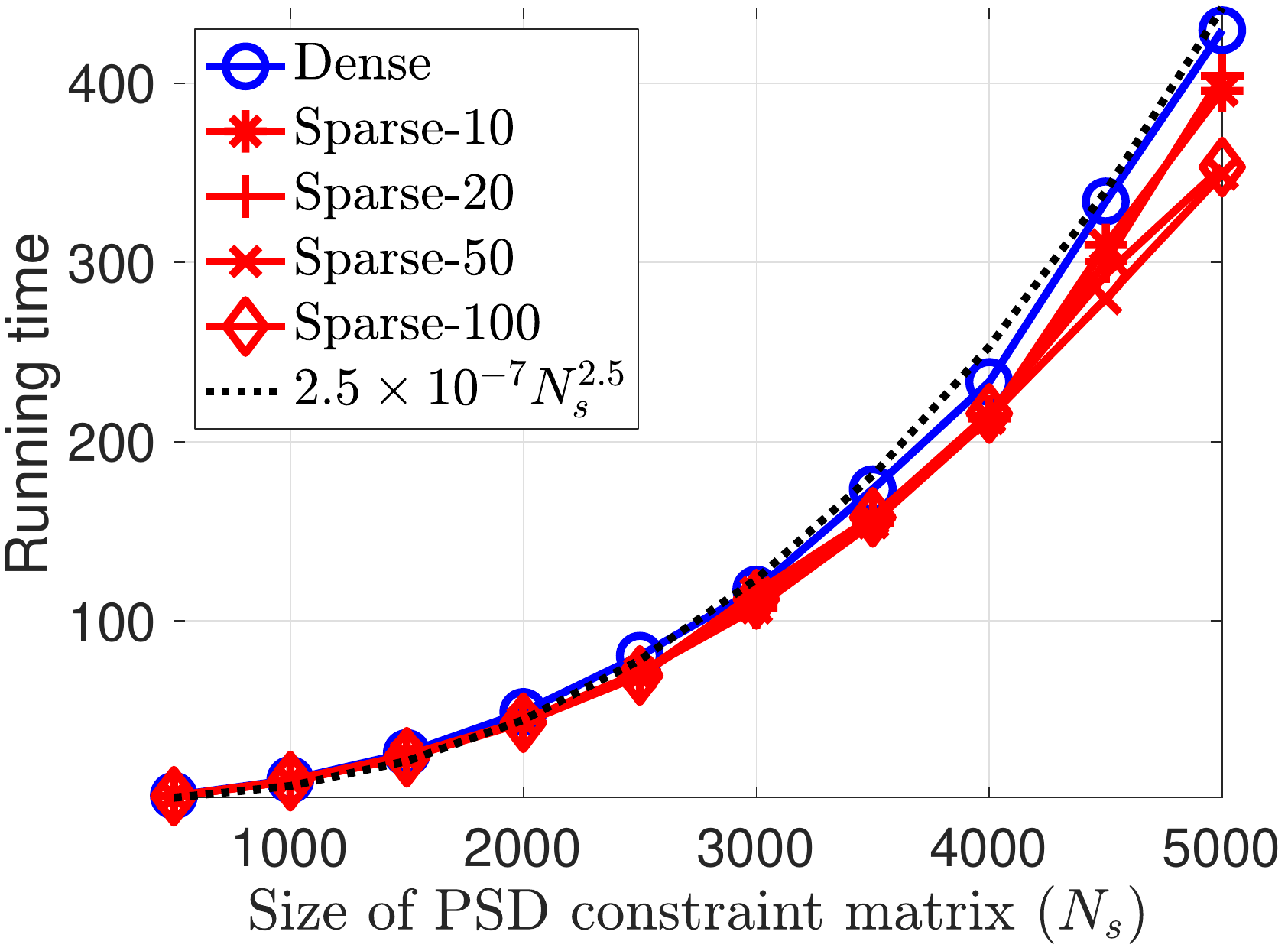}
\caption{The time (in seconds) needed for solving SDP~\eqref{eq:SDP_simple} with varying size of PSD constraint matrix (denoted as $N_s$).}
\label{fig:test_SDP_time}
\end{figure}

We conclude that the answer to the above question is ``no'', namely, that sparsity in the PSD constraint matrix will not significantly reduce the computational complexity of solving an SDP problem. To reduce the computational complexity, we introduce a fast version of the ``ANM+DPSS+SMI'' method in next section.

\section{Proposed Method: ``IVDST+DPSS+SMI''}
\label{sec:method_fast}

Recall that a key step in the proposed ``ANM+DPSS+SMI'' method is to compute the dual solution. Though CVX can return the dual solution by solving the primal SDP~\eqref{eq:SDPanm}, one can also directly solve the dual problem of ~\eqref{eq:SDPanm} to get the dual solution.
With some elementary calculations, we obtain the dual problem of~\eqref{eq:SDPanm} as follows\footnote{To simplify the derivation, we replace the inequality constraint $\|\Xs - \LL(\XX)\|_F\leq \varepsilon$ with an equality constraint $\Xs = \LL(\XX)$, i.e., we set $\varepsilon=0$.
}
\begin{equation}
\begin{aligned}
\sup_{\substack{\QQ\in\CCC^{M \times L \times N} \\ \HH\in\CCC^{M \times M \times N}}} &\langle \Xs, \LL(\QQ) \rangle_{\RRR} \\
 \st~~~~	
&\left[\begin{array}{cc}
\HH_{::n} & -\QQ_{::n}\\
-\QQ_{::n}^H & \I_L
\end{array}\right] \succeq \zero,\\
&\sum_{m = 1}^{M-j} \HH_{m(m+j)n} = 0,~\forall~j\in[M-1],~\forall~n\in[N],\\
&\sumn \summ \HH_{mmn} = 1.
\label{eq:dual_SDP}
\end{aligned}
\end{equation}

Note that the above dual problem~\eqref{eq:dual_SDP} is also an SDP. Therefore, it can be solved by any off-the-shelf SDP solver just as the primal SDP~\eqref{eq:SDPanm}. However, as is discussed in Remark~\ref{rmk:comput}, solving SDP problems directly with SDP solvers can result in high computational complexity.
To reduce the computational complexity, the authors in~\cite{bhaskar2013atomic} propose a fast method based on Alternating
Direction Method of Multipliers (ADMM)~\cite{bertsekas1989parallel,boyd2011distributed} to solve the SDP. In~\cite{hansen2019fast}, the authors reformulate the SDP as a conic program with much fewer dual variables to further reduce the computational cost. A fast iterative Vandermonde decomposition and shrinkage-thresholding (IVDST) algorithm based on the accelerated proximal gradient
(APG) technique is developed in~\cite{wang2018ivdst} to solve the SDP. Recently, this algorithm is also extended to the two-dimensional (2D) framework in a problem of 2D grid-free compressive beamforming~\cite{liu2020iterative}.

Inspired by the IVDST algorithm, we propose a faster method, named ``IVDST+DPSS+SMI'', to speed up the ``ANM+DPSS+SMI'' method introduced in Section~\ref{sec:method_sdp}. Rather than solving the primal SDP~\eqref{eq:SDPanm}, we focus on its dual problem~\eqref{eq:dual_SDP} in this section. Note that the dual problem~\eqref{eq:dual_SDP} can be rewritten as the following unconstrained problem
\begin{align*}
\min_{\QQ,\HH}~ h(\QQ) + g(\QQ,\HH)	
\end{align*}
with
\begin{align*}
h(\QQ) &\triangleq - \langle \Xs, \LL(\QQ) \rangle_{\RRR},\\
g(\QQ,\HH) & \triangleq g_1(\QQ,\HH) + g_2(\HH) + g_3(\HH),\\
g_1(\QQ,\HH) & \triangleq \lambda_1 \sumn \CC(\Z(\QQ_{::n},\HH_{::n})), \\
g_2(\HH) & \triangleq \lambda_2\left( \sumn \summ \HH_{mmn} - 1\right)^2,\\
g_3(\HH) & \triangleq \lambda_3 \sumn \sum_{j=1}^{M-1}\left(\sum_{m = 1}^{M-j} \HH_{m(m+j)n}\right)^2,
\end{align*}
where $\{\lambda_i\}_{i=1}^3$ are regularization parameters. $\Z(\QQ_{::n},\HH_{::n})$ is a matrix defined as
\begin{align*}
\Z(\QQ_{::n},\HH_{::n}) \triangleq \left[\begin{array}{cc}
\HH_{::n} & -\QQ_{::n}\\
-\QQ_{::n}^H & \I_L
\end{array}\right].	
\end{align*}
$\CC(\Z)$ is an indicator function used to enforce the PSD constraint and is defined as
\begin{align*}
\CC(\Z) \triangleq \begin{cases}
 0, \quad &\text{if}~\Z~\text{is PSD,}\\
 \infty,	 &\text{otherwise.}
 \end{cases}	
\end{align*}

Next, we introduce how to use the idea of APG to estimate the dual solution $\QQ^\star$, which is then used to design a weight vector $\w$.

\begin{description}[leftmargin=*]
\item[(1) {\em Initialization:}] According to the Schur complement condition, $\left[\begin{array}{cc}
\HH_{::n} & -\QQ_{::n}\\
-\QQ_{::n}^H & \I_L
\end{array}\right] \succeq \zero$ and $\I_L \succ \zero$ imply that
\begin{align*}
\HH_{::n} - 	\QQ_{::n}\QQ_{::n}^H \succeq \zero.
\end{align*}
Therefore, we initialize $\QQ_1=\QQ_0\in \CCC^{M\times L \times N}$ as a complex Gaussian random tensor, and initialize $\HH_1 = \HH_0\in \CCC^{M \times M \times N}$ as
\begin{align}
\{\HH_1\}_{::n} = \{\HH_0\}_{::n} = \{\QQ_0\}_{::n}\{\QQ_0\}_{::n}^H, \quad \forall~ n\in[N].	
\label{eq:init_H}
  \end{align}

\item[(2) {\em Smoothing:}] In each iteration, we add a momentum term to accelerate the convergence of the gradient vector. In particular, the updates in the $i$-th iteration are given as
\begin{equation}
\begin{aligned}
  &\QQb_i = \QQ_i + \frac{t_{i-1}-1}{t_i}(\QQ_i - \QQ_{i-1}), \\ &\HHb_i = \HH_i + \frac{t_{i-1}-1}{t_i}(\HH_i - \HH_{i-1}),
  \label{eq:QHbar}
\end{aligned}
\end{equation}
where $t_i = \frac{1+\sqrt{4 t_{i-1}^2 + 1}}{2}$ with $t_0 = 1$.

\item[(3) {\em Gradient descent:}] After smoothing, we update the parameters $\QQ$ along the gradient descent direction of $h(\QQ)$ with an appropriate small stepsize $\eta$. $\HH$ is fixed in this step since $h(\QQ)$ only depends on $\QQ$. Then, the updates in the $i$-th iteration are given as
\begin{equation}
\begin{aligned}
  \QQ_g &= \QQb_i - \eta \nabla h(\QQ)
  = \QQb_i + \eta \LL^*(\Xs), \\
   \HH_g &= \HHb_i,
   \label{eq:QHg}
  \end{aligned}
  \end{equation}
where $\LL^*: \CCC^{M \times N} \rightarrow \CCC^{M\times L\times N}$ denotes the adjoint operator of the linear operator $\LL$ defined in~\eqref{eq:def_L}. In particular, we have
\begin{align}
[\LL^*(\Xs)]_{::n} = \S_{M,W} \odot (\xs_n \one_L^H ), \quad \forall~n\in[N],
\label{eq:def_L*}	
\end{align}
 where $\xs_n$ denotes the $n$-th column of $\Xs$.

\item[(4) {\em Proximal mapping:}] In this step, we use the following proximal operator
\begin{align*}
&(\QQ_{i+1},\HH_{i+1}) = \text{prox}_{\eta g}(\QQ_g,\HH_g)	\\
=& \arg \min_{\QQ,\HH} \left\{g(\QQ,\HH) + \frac{1}{2\eta} \|(\QQ,\HH)-(\QQ_g,\HH_g)\|_F^2 \right\}.
\end{align*}
Note that there is no analytical solution to the above optimization problem, therefore, we use alternating projection method to approximate its solution. In particular, for any $n\in[N]$, we first update $\HH$ with
\begin{equation}
\begin{aligned}
	\HHt_{mmn} &= \frac{\{\HH_g\}_{mmn}}{\sumn \summ \{\HH_g\}_{mmn}},\\
  \HHt_{m(m+j)n} &= \{\HH_g\}_{m(m+j)n}- \frac{1}{M-j}\sum_{m=1}^{M-j} \{\HH_g\}_{m(m+j)n}, \\ &~~~~~~~~~~~~~~~~~~~~~~~~~~~~~~~~~~~~~~~\forall~j\in[M-1],
  \label{eq:Htilde}
\end{aligned}
\end{equation}
which ensures
\begin{align*}
&\sumn \summ \HHt_{mmn} = 1, \text{ and }
\\
&\sum_{m = 1}^{M-j} \HHt_{m(m+j)n} = 0,~\forall~j\in[M-1],~\forall~n\in[N],
\end{align*}
as is required in the constraints of~\eqref{eq:dual_SDP}.

Next, to promote the PSD constraint, we evaluate
\begin{align}
	\Z_n  = \left[\begin{array}{cc}
\HHt_{::n} & -\{\QQ_g\}_{::n}\\
-\{\QQ_g\}_{::n}^H & \I_L
\end{array}\right], \quad \forall~ n\in[N]
\label{eq:Zn}
\end{align}
and compute its eigen-decomposition
\begin{align*}
	(\V_n,\bSigma_n) = \text{eig}(\Z_n),\quad \forall~ n\in[N].
\end{align*}
To make $\{\Z_n\}_{n=1}^N$ PSD, we only keep the positive eigenvalues and the corresponding eigenvectors. Then, we get
\begin{align}
	\Zt_n = \V_n \max(\bSigma_n,0)\V_n^H, ~ \forall~n\in[N].
	\label{eq:Ztn}
\end{align}
With the obtained PSD matrices $\{\Zt_n\}_{n=1}^N$, we update $\HH$ and $\QQ$ with
\begin{equation}
\begin{aligned}
	\{\HH_{i+1}\}_{::n} &= \Zt_n(1:M,1:M), \\
	\{\QQ_{i+1}\}_{::n} &= -\Zt_n(1:M,M+1:M+L), \quad \forall~n\in[N].
	\label{eq:QHi1}
\end{aligned}
\end{equation}

\item[(5) {\em Computing weight:}] Denoting $I$ as the number of iterations, then $\QQ_I$ is an estimation of the dual solution $\QQ^\star$. Given $\QQ_I$, we can then construct a dual polynomial as in~\eqref{eq:dualpoly} and use the k-means method to cluster the $f^o$ with $q(f^o) \geq \gamma_0$ and get the estimated frequencies $\ft_k^o$. As is shown in~\eqref{eq:sign_a1}, to compute $\text{sign}(\widehat{\balpha}_1)$, one needs to first estimate $\widehat{\XX}_{::1}$, which is the primal solution to the SDP~\eqref{eq:SDPanm}. However, since we focus on solving the dual SDP~\eqref{eq:dual_SDP}, we do not have the primal solution. Fortunately, we can use $[\LL^*(\Xs)]_{::1}$ as a surrogate to $\widehat{\XX}_{::1}$. Then, we can compute $\text{sign}(\widehat{\balpha}_1)$ as
\begin{align}
\text{sign}(\widehat{\balpha}_1)
=\frac{(\A_f^\dag [\LL^*(\Xs)]_{::1} )^\top\e_1}{\| (\A_f^\dag [\LL^*(\Xs)]_{::1})^\top\e_1 \|_2}
\label{eq:sign_a1_fast}
\end{align}
with $\A_f = [\a(\ft_1^o)\cdots \a(\ft_K^o)]$. Finally, $\widetilde{\s}_1$ can be computed as in~\eqref{eq:sest1d} and the weight vector $\w$ can be computed using~\eqref{eq:weight_fix}.
\end{description}

We summarize the proposed ``IVDST+DPSS+SMI'' method in Algorithm~\ref{alg:IVDST_DPSS_SMI}.
\begin{algorithm}[H]
\caption{IVDST+DPSS+SMI}\label{alg:IVDST_DPSS_SMI}
\begin{algorithmic}[1]
\Procedure{Input}{the data matrix $\Xs$, the number of DPSS basis vectors $L$, the stepsize $\eta$, and the number of iterations $I$}
\State create DPSS basis $\S_{M,W}\in\RRR^{M\times L}$
\State compute $\LL^*(\Xs)$ according to~\eqref{eq:def_L*}
\State initialize $\QQ_1=\QQ_0$ as a complex Gaussian random tensor, and initialize $\HH_1 = \HH_0$ as in~\eqref{eq:init_H}
\For{$ i = 1, ..., I$}
\State {\em smoothing:} compute $\QQb_i$ and $\HHb_i$ via~\eqref{eq:QHbar}
\State {\em gradient descent:} compute $\QQ_g$ and $\HH_g$ via~\eqref{eq:QHg}
\State {\em proximal mapping:} compute $\HHt$ via~\eqref{eq:Htilde}, evaluate $\Z_n$ via~\eqref{eq:Zn} and compute its eigen-composition to get $(\V_n,\bSigma_n)$, then compute $\Zt_n$ via~\eqref{eq:Ztn}, finally, update $\QQ_{i+1}$ and $\HH_{i+1}$ via~\eqref{eq:QHi1}

\EndFor
\State \textbf{return} dual solution $\QQ_I$

\State given $\QQ_I$, form the dual polynomial $q(f^o)$ as in~\eqref{eq:dualpoly}
\State use the k-means method to cluster the $f^o$ with $q(f^o) \geq \gamma_0$ and get $\ft_k^o$; set $k=1$ according to the desired frequency
\State compute $\text{sign}(\widehat{\balpha}_1)$ according to~\eqref{eq:sign_a1_fast}
\State estimate $\s_1$ as in~\eqref{eq:sest1d} and construct the weight vector as in~\eqref{eq:weight_fix}
\State \textbf{return} the weight vector $\w$
\EndProcedure
\end{algorithmic}
\end{algorithm}

\begin{Remark}{(Computational complexity.)} \label{rmk:comput_fast}
Compared with the proposed ``ANM+DPSS+SMI'' method, the most expensive step in the ``IVDST+DPSS+SMI'' method is computing the eigen-decomposition of $\Z_n,~\forall~n\in[N]$. Therefore, the overall complexity is on the order of $\OO(N(M+L)^2)$.
\end{Remark}	

\section{Extension to 2D Case}
\label{sec:2d}

In this section, we take the structure of $\textbf{asv}(\theta)$ into consideration and define a {\em two dimensional (2D) atomic set} as
\begin{equation}
\begin{aligned}
\AA_{2D} \triangleq \left\{\A(f^o) \circledast (\balphat \textbf{asv}(\theta)):~f^o\in[0,1),~ \balphat \in\CCC^{L}, \right.\\
\left.  \|\balphat\|_2=1,~\theta\in[-90^\circ, 90^\circ] \right\}.
\label{eq:atomicset2D}	
\end{aligned}
\end{equation}
The induced {\em 2D atomic norm} is then defined as
\begin{equation}
\begin{aligned}
\|\XX\|_{\AA_{2D}} \triangleq \inf \left\{ \sumk c_k: \right. ~~~~~~~~~~~~~~~~~~~~~~~~~~~~~~~~~~~\\
~~~~\left.\XX = \sumk c_k \A(f_k^o) \circledast (\balphat_k \textbf{asv}(\theta_k)), c_k\geq 0    \right\},	
\label{eq:atomicnorm2D}
\end{aligned}
\end{equation}
As in Section~\ref{sec:method_sdp}, one can solve the following 2D-ANM problem
\begin{align}
 \min_{\XX} ~\|\XX\|_{\AA_{2D}} \quad \st~	\|\Xs - \LL(\XX)\|_F\leq \varepsilon
\label{eq:ANM2D}
\end{align}
to construct the weight vector.

Denote $\langle \QQ,\XX \rangle_{\RRR} = \text{Re}\{\langle \QQ,\XX \rangle\}$ as the real inner product between two tensors.
The {\em dual norm} of $\|\XX\|_{\AA_{2D}}$ is then defined as 
\begin{align*}
\|\QQ\|_{\AA_{2D}}^* =& \sup_{\|\XX\|_{\AA_{2D}}\leq 1} \langle \QQ,\XX \rangle_{\RRR} \\
=& \sup_{\substack{ f^o\in[0,1),\|\balphat\|_2=1\\ \theta\in[-90^\circ, 90^\circ] }} \langle \QQ,\A(f^o) \circledast (\balphat \textbf{asv}(\theta)) \rangle_{\RRR} \\	
=& \sup_{\substack{ f^o\in[0,1),\|\balphat\|_2=1\\ \theta\in[-90^\circ, 90^\circ] }} \sumn \langle \QQ_{::n},\a(f^o) e^{jk_0 \sin(\theta)\q_n }  \balphat ^\top  \rangle_{\RRR} \\	
=& \sup_{\substack{ f^o\in[0,1),\|\balphat\|_2=1\\ \theta\in[-90^\circ, 90^\circ] }}  \left\langle \balphat ^* , \sumn \QQ_{::n}^H\a(f^o) e^{jk_0 \sin(\theta)\q_n }   \right\rangle_{\RRR} \\
=& \sup_{\substack{ f^o\in[0,1)\\ \theta\in[-90^\circ, 90^\circ] }}  \left\|\sumn \QQ_{::n}^H\a(f^o) e^{jk_0 \sin(\theta)\q_n }   \right\|_2\\
=&  \sup_{\substack{ f^o\in[0,1)\\ \theta\in[-90^\circ, 90^\circ] }} q_{2D}(f^o, \theta),
\end{align*}
where
\begin{align}
	q_{2D}(f^o, \theta) \triangleq \left\|\sumn \QQ_{::n}^H\a(f^o) e^{jk_0\sin(\theta)\q_n }   \right\|_2
	\label{eq:dualpoly2D}
\end{align}
is defined as the {\em dual polynomial} and $\q_n$ denotes the $n$-th element position.

Next, we show that the above 2D atomic norm defined in~\eqref{eq:atomicnorm2D} can be approximated by the optimal value of the following semidefinite program (SDP):
\begin{equation}
\begin{aligned}
\inf_{\substack{\TT \in\CCC^{(2M-1)\times (2L-1)\times N}\\ \bm{t}\in\CCC^N}} & \frac{1}{2MN} \sumn \text{Tr}(\SS(\TT_{::n})) + \frac{1}{2N} \sumn t_n \\
 \st~~~~~~~~~
&\left[\begin{array}{cc}
\SS(\TT_{::n}) & \x_n\\
\x_n^H & t_n
\end{array}\right] \succeq \zero,\\
 &\x_n = \text{vec}(\XX_{::n}^\top),~ \forall~n\in[N],
 \label{eq:SDP_2D}
\end{aligned}
\end{equation}
where $\SS(\T)\in\CCC^{ML \times ML}$ is a block Toeplitz matrix generated by $\T\in\CCC^{(2M-1)\times (2L-1)}$. To be more precise, denote $T_{m,l}$ with $-M<m<M$ and $-L<l<L$ as the $(m,l)$-th entry of $\T$. Denote $\T_m\in\CCC^{L\times L}$ as a Toeplitz matrix constructed from the $m$-th row of $\T$, i.e.,
\begin{align*}
\T_m = 	\left[\begin{array}{cccc}
T_{m,0} & T_{m,-1} & \cdots & T_{m,-(L-1)}\\
T_{m,1} & T_{m,0} & \cdots & T_{m,-(L-2)}\\	
\vdots & \vdots & \ddots & \vdots\\
T_{m,L-1} & T_{m,L-2} & \cdots & T_{m,0}\\
\end{array}\right].
\end{align*}
Then, the block Toeplitz matrix $\SS(\T)$ is defined as
\begin{align*}
\SS(\T) = 	\left[\begin{array}{cccc}
\T_0 & \T_{-1} & \cdots & \T_{-(M-1)}\\
\T_1 & \T_0 & \cdots & \T_{-(M-2)}\\	
\vdots & \vdots & \ddots & \vdots\\
\T_{M-1} & \T_{M-2} & \cdots & \T_0\\
\end{array}\right].
\end{align*}

Denote SDP$(\XX)$ as the objective value corresponding to the optimal solution of the above SDP~\eqref{eq:SDP_2D}. The following proposition shows that SDP$(\XX)$ is a lower bound of $\|\XX\|_{\AA_{2D}}$.
\begin{Proposition}
With the 2D-atomic set and 2D-atomic norm defined in~\eqref{eq:atomicset2D} and~\eqref{eq:atomicnorm2D}, we have
\begin{align*}
\text{SDP}(\XX) \leq \|\XX\|_{\AA_{2D}}	
\end{align*}
for any $\XX\in \CCC^{M\times L \times N}$. 	
\end{Proposition}
\begin{proof}
Let
\begin{align*}
\XX = \sumk c_k \A(f_k^o) \circledast (\balphat_k \textbf{asv}(\theta_k))	
\end{align*}
with $f^o_k\in[0,1),~ \balphat_k \in\CCC^{L}, ~\|\balphat_k\|_2=1,~\theta_k\in[-90^\circ, 90^\circ]$. Then, for any $n\in[N]$, we have
\begin{align*}
\XX_{::n} = \sumk c_k \a(f_k^o) \balphat_k ^\top \textbf{asv}_n(\theta_k),		
\end{align*}
and
\begin{align*}
\x_n = \text{vec}(\XX_{::n}^\top) = \sumk c_k \a(f_k^o) \otimes 	\balphat_k \textbf{asv}_n(\theta_k),
\end{align*}
where $\textbf{asv}_n(\theta_k)$ denotes the $n$-th entry of $\textbf{asv}(\theta_k)$ and $\otimes$ denotes the Kronecker product.

For $t_n = \sumk c_k$ and a block Toeplitz matrix
\begin{align*}
\SS(\TT_{::n}) 
\!=\! \sumk \!c_k \!\left[ \a(f_k^o) \!\otimes 	\balphat_k \textbf{asv}_n(\theta_k) \right]	\left[ \a(f_k^o) \otimes 	\balphat_k \textbf{asv}_n(\theta_k) \right]^H\!\!,	
\end{align*}
we have
\begin{align*}
&\left[\begin{array}{cc}
\SS(\TT_{::n}) & \x_n\\
\x_n^H & t_n
\end{array}\right] \\
=& \sumk c_k
\left[\begin{array}{c}
\a(f_k^o) \otimes 	\balphat_k \textbf{asv}_n(\theta_k) \\
1
\end{array}\right]	
\left[\begin{array}{c}
\a(f_k^o) \otimes 	\balphat_k \textbf{asv}_n(\theta_k) \\
1
\end{array}\right]^H\\
\succeq & \zero.
\end{align*}
Thus, the above $\SS(\TT_{::n})$ and $t_n$ with any $n\in[N]$ are feasible to the SDP~\eqref{eq:SDP_2D}. It follows that
\begin{align*}
\text{SDP}(\XX) \leq & \frac{1}{2MN} \sumn \text{Tr}(\SS(\TT_{::n})) + \frac{1}{2N} \sumn t_n \\
=& \frac{1}{2MN} \sumn \sumk c_k \|\a(f_k^o) \otimes 	\balphat_k \textbf{asv}_n(\theta_k) \|_2^2 + \frac{1}{2} \sumk c_k 	\\
=& \sumk c_k = \|\XX\|_{\AA_{2D}}.
\end{align*}	
\end{proof}

Then, we propose the following SDP to approximate the solution of 2D-ANM in~\eqref{eq:ANM2D}:
\begin{equation}
\begin{aligned}
\inf_{\substack{\TT \in\CCC^{(2M-1)\times (2L-1)\times N}\\ \bm{t}\in\CCC^N\\ \XX\in\CCC^{M \times L \times N}}} & \frac{1}{2MN} \sumn \text{Tr}(\SS(\TT_{::n})) + \frac{1}{2N} \sumn t_n \\
 \st~~~~~~~~~
&\left[\begin{array}{cc}
\SS(\TT_{::n}) & \x_n\\
\x_n^H & t_n
\end{array}\right] \succeq \zero,\\
 &\x_n = \text{vec}(\XX_{::n}^\top),~ \forall~n\in[N],\\
&\|\Xs - \LL(\XX)\|_F\leq \varepsilon.
\label{eq:SDPanm2D}
\end{aligned}
\end{equation}
Again, solving the above SDP with CVX can return both the primal solution and the dual solution.
Given the dual solution, one can construct a dual polynomial as in~\eqref{eq:dualpoly2D} and identify $\{f_k^o,\theta_k\}_{k=1}^K$ by localizing the places where $q_{2D}(f^o,\theta)=1$, similar as in Section~\ref{sec:method_sdp}, and finally construct a weight vector.

We summarize the proposed ``2D-ANM+DPSS+SMI'' method in Algorithm~\ref{alg:2DANM_DPSS_SMI}.
\begin{algorithm}[H]
\caption{2D-ANM+DPSS+SMI}\label{alg:2DANM_DPSS_SMI}
\begin{algorithmic}[1]
\Procedure{Input}{the data matrix $\Xs$ and the number of DPSS basis vectors $L$}
\State create DPSS basis $\S_{M,W}\in\RRR^{M\times L}$
\State compute the primal solution $\widehat{\XX}$ and the dual solution $\QQ^\star$ by solving the SDP~\eqref{eq:SDPanm2D}
\State form the dual polynomial $q_{2D}(f^o)$ as in~\eqref{eq:dualpoly2D}
\State use the k-means method to cluster the $f^o$ with $q_{2D}(f^o,\theta) \geq \gamma_0$ and get $\ft_k^o$; set $k=1$ according to the desired frequency
\State compute $\text{sign}(\widehat{\balpha}_1)$ according to~\eqref{eq:sign_a1}
\State estimate $\s_1$ as in~\eqref{eq:sest1d} and construct the weight vector as in~\eqref{eq:weight_fix}
\State \textbf{return} the weight vector $\w$
\EndProcedure
\end{algorithmic}
\end{algorithm}

\begin{Remark}{(Advantages and disadvantages of ``2D-ANM+DPSS+SMI''.)} Unlike the 1D-based methods,
in the above proposed ``2D-ANM+DPSS+SMI'' method, one needs to know the element positions $\q\in\RRR^{1\times N}$, and the element positions need to be equispaced. However, the ``2D-ANM+DPSS+SMI'' method can successfully identify the frequency components in the signals even when there exist two identical frequencies from different directions while the 1D method must have a certain separation between all frequencies. See Figure~\ref{fig:test_ANM_DPSS_SMI_2D_dual}.
\end{Remark}

\section{Numerical Simulations}
\label{sec:nume}

In this section, we conduct a series of experiments to test the proposed three methods and compare them with the ``ANM+SMI'' method proposed in~\cite{li2020adaptiveAWPL} and the classical SMI method~\cite{van2004optimum}. We simulate a uniform linear array with $N = 4$ elements and half-wavelength element spacing, namely, the element position vector $\q = -(N-1)d/2:d:(N-1)d/2$ with $d = 0.5$. We assume one desired signal and two interferers; thus, $K=3$. The angles to the desired signal and two interferers are set as $\theta_1 = -20^\circ$, $\theta_2 = -60^\circ$, and $\theta_3 = 20^\circ$.

In all experiments, we generate each of $\s_1, \s_2, \s_3$ as a sinusoid (complex exponential) modulated by a frequency offset that drifts slightly over time. This corresponds to Case 2 in Section~\ref{sec:prob} and allows for an illustrative and challenging set of experiments. The ground truth sinusoid frequencies are given by $f_1^o=0.1$, $f_2^o = 0.3$ and $f_3^o = 0.5$. We consider four different types of frequency offsets: (a) static frequency offset, (b) linear frequency offset, (c) zigzag frequency offset, and (d) random frequency offset, which are shown in Figure~\ref{fig:test_ANM_IVDST_DPSS_SMI_freqdrift}. The three colors in each plot represent the time-varying offset relative to the ground truth frequencies $f_1^o$, $f_2^o$, and $f_3^o$. We take $M=120$ uniform time samples at each element and formulate the data matrix $\Xs$ as in~\eqref{eq:data}.

For the proposed ``IVDST+DPSS+SMI'' method, we set the stepsize as $\eta = 4$ and decrease it by multiplying a factor of $0.99$ after every $50$ iterations. The maximum number of iterations is set as $I=200$. We present the values of $L$ used to generate the DPSS basis in Table~\ref{TB:L}. 

The dual polynomials obtained from the ``ANM+SMI'' method, the ``ANM+DPSS+SMI'' method, and the ``IVDST+DPSS+SMI'' method are shown in Figure~\ref{fig:test_ANM_IVDST_DPSS_SMI_dual}. The corresponding radiation pattern (array factor) is illustrated in Figure~\ref{fig:test_ANM_IVDST_DPSS_SMI_radi}. Figure~\ref{fig:test_ANM_IVDST_DPSS_SMI_radi} also includes the radiation pattern for ``SMI'' which corresponds to the classical SMI method but accounting only for the sinusoidal component of $\s_1$ and not its time-varying offset; that is, setting $\w = (\X^H\X)^{-1}\X^H \s = \X^\dagger \a(f_1^o)$. (Our methods do not assume explicit knowledge of $f_1^o$.)

From Figures~\ref{fig:test_ANM_IVDST_DPSS_SMI_dual} and~\ref{fig:test_ANM_IVDST_DPSS_SMI_radi}, it can be seen that the ``ANM+SMI'' method~\cite{li2020adaptiveAWPL} works very well when the frequency offsets are static. However, our proposed methods ``IVDST+DPSS+SMI'' and ``ANM+DPSS+SMI'' significantly outperform the ``ANM+SMI'' and ``SMI'' methods when the frequency offsets are time-varying. Moreover, based on the dual polynomials shown in Figure~\ref{fig:test_ANM_IVDST_DPSS_SMI_dual}, it is much easier to cluster the frequencies with the fast ``IVDST+DPSS+SMI'' method, as the dual polynomials obtained from the ``ANM+DPSS+SMI'' method are sometimes relatively flat. To show the efficiency of the proposed ``IVDST+DPSS+SMI'' method, we present the time used by ``IVDST+DPSS+SMI'' and ``ANM+DPSS+SMI'' in Table~\ref{TB:time}. As expected, the ``IVDST+DPSS+SMI'' method runs much faster than the SDP-based ``ANM+DPSS+SMI'' method. 

\begin{figure}[ht]
\begin{minipage}{0.48\linewidth}
\centering
\includegraphics[width=1.67in]{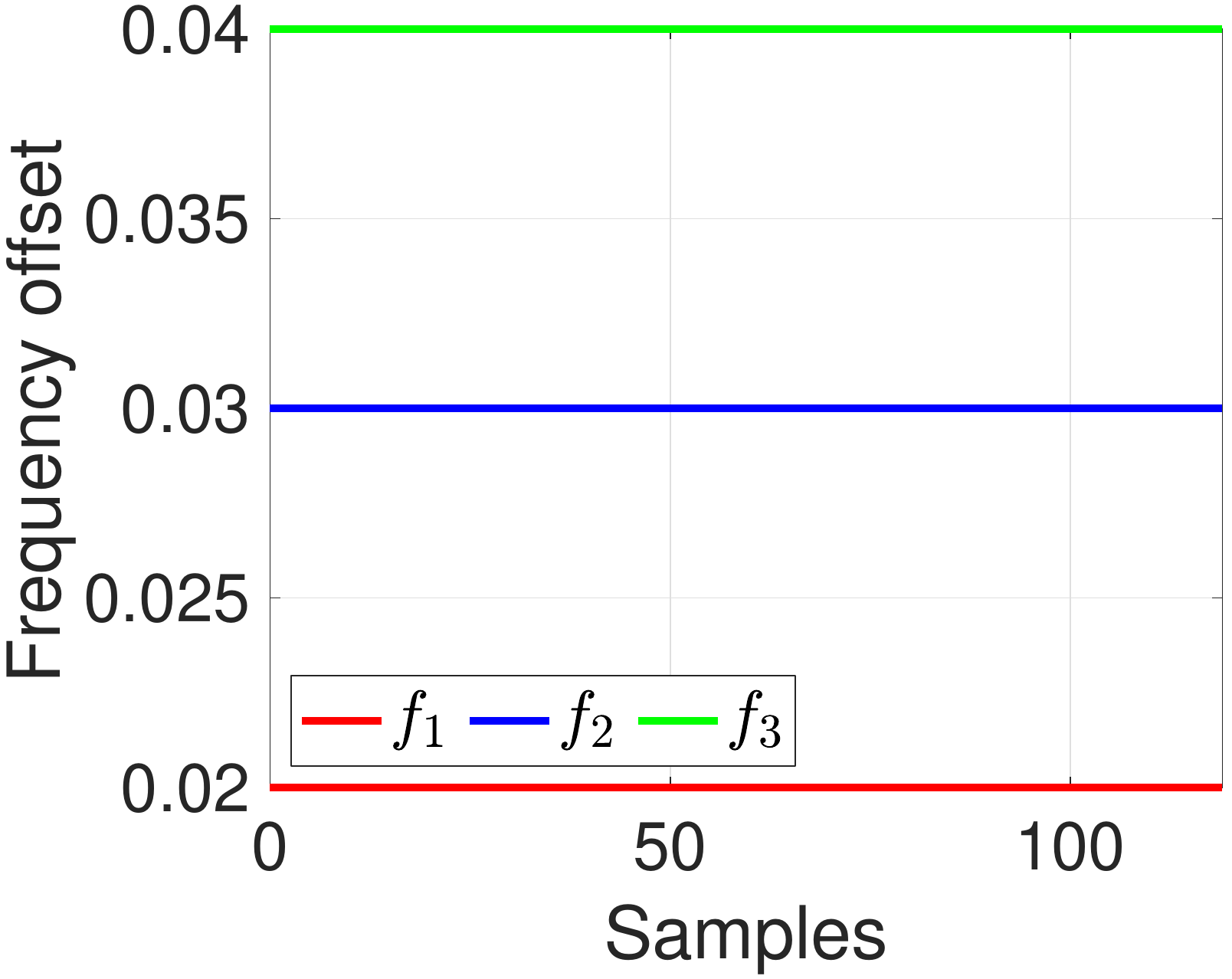}
\centerline{\footnotesize{(a) Static frequency offset}}
\end{minipage}
~
\begin{minipage}{0.48\linewidth}
\centering
\includegraphics[width=1.67in]{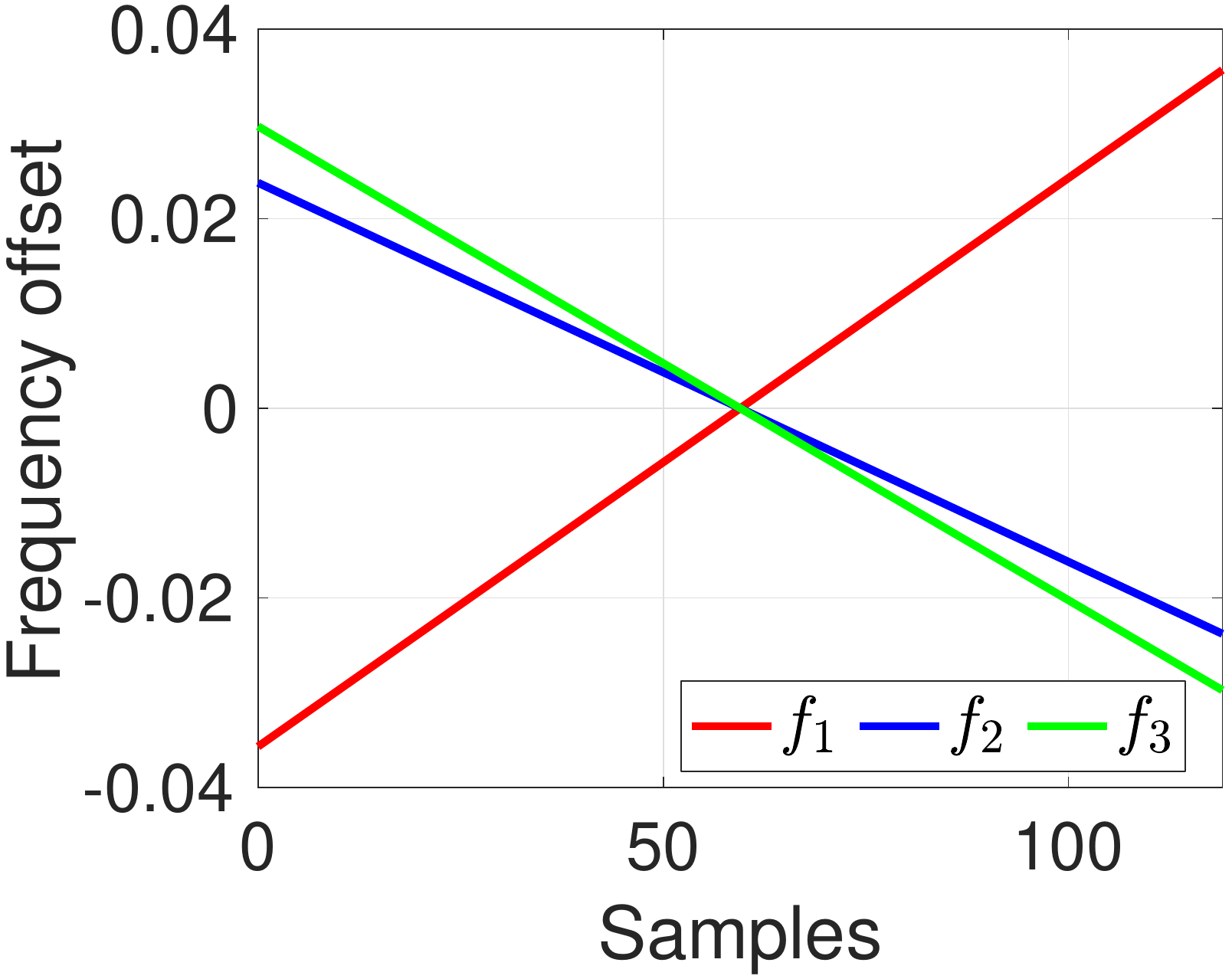}
\centerline{\footnotesize{(b) Linear frequency offset}}
\end{minipage}
\\
\begin{minipage}{0.48\linewidth}
\centering
\includegraphics[width=1.67in]{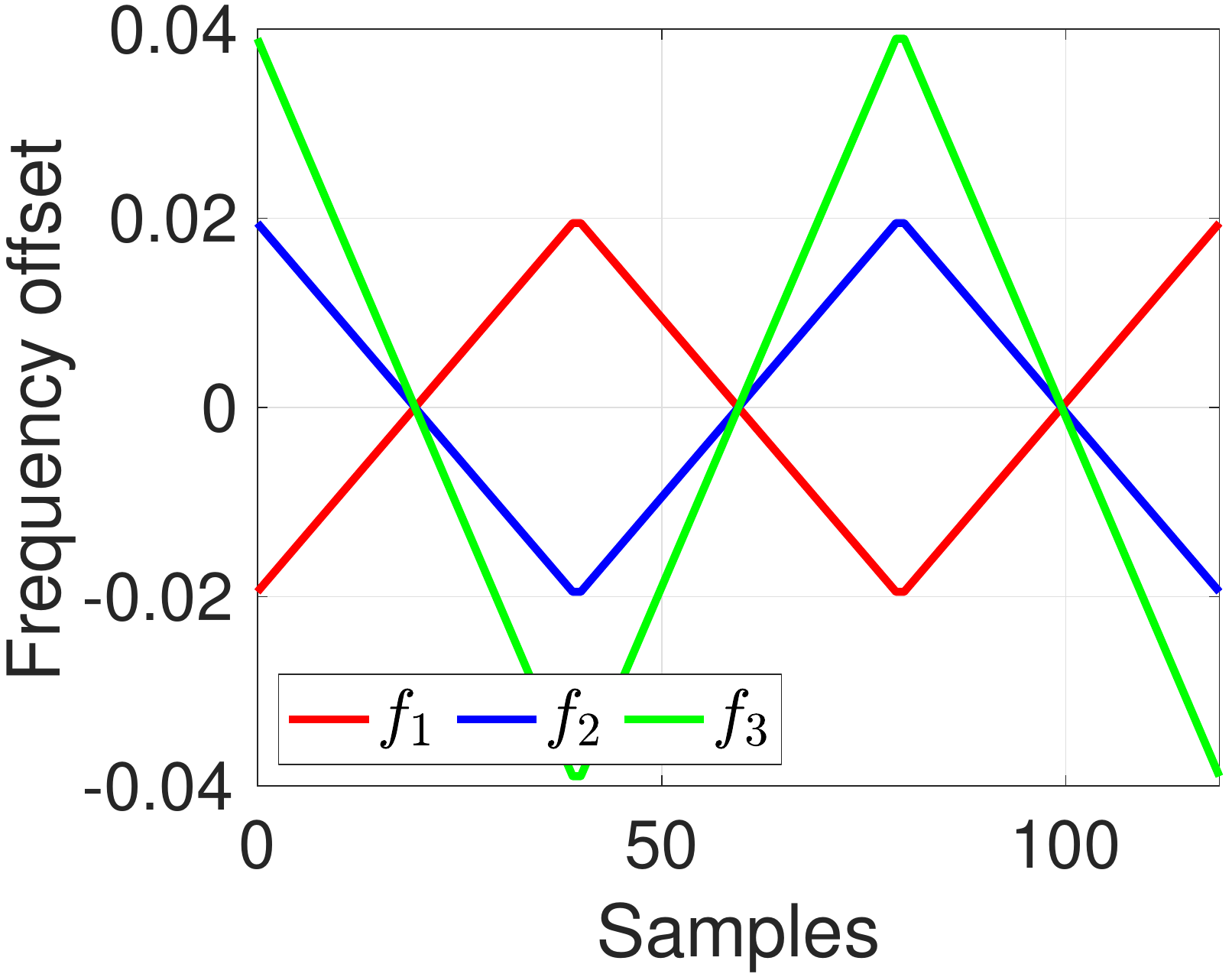}
\centerline{\footnotesize{(c) Zigzag frequency offset}}
\end{minipage}
~
\begin{minipage}{0.48\linewidth}
\centering
\includegraphics[width=1.67in]{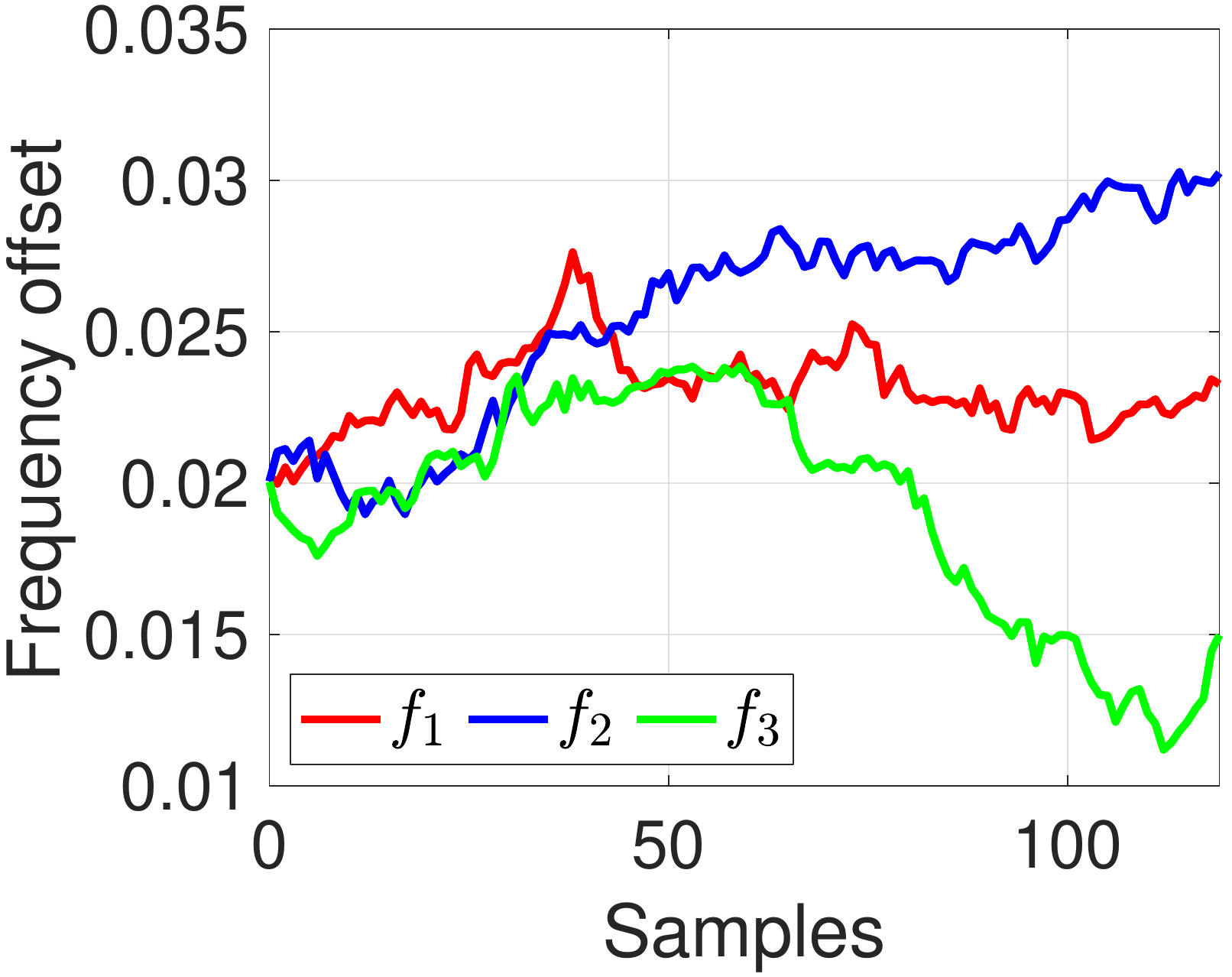}
\centerline{\footnotesize{(d) Random frequency offset}}
\end{minipage}
\caption{Four types of frequency offsets used in the experiments. ($M=120$)}
\label{fig:test_ANM_IVDST_DPSS_SMI_freqdrift}
\end{figure}

\begin{figure}[ht]
\begin{minipage}{0.48\linewidth}
\centering
\includegraphics[width=1.67in]{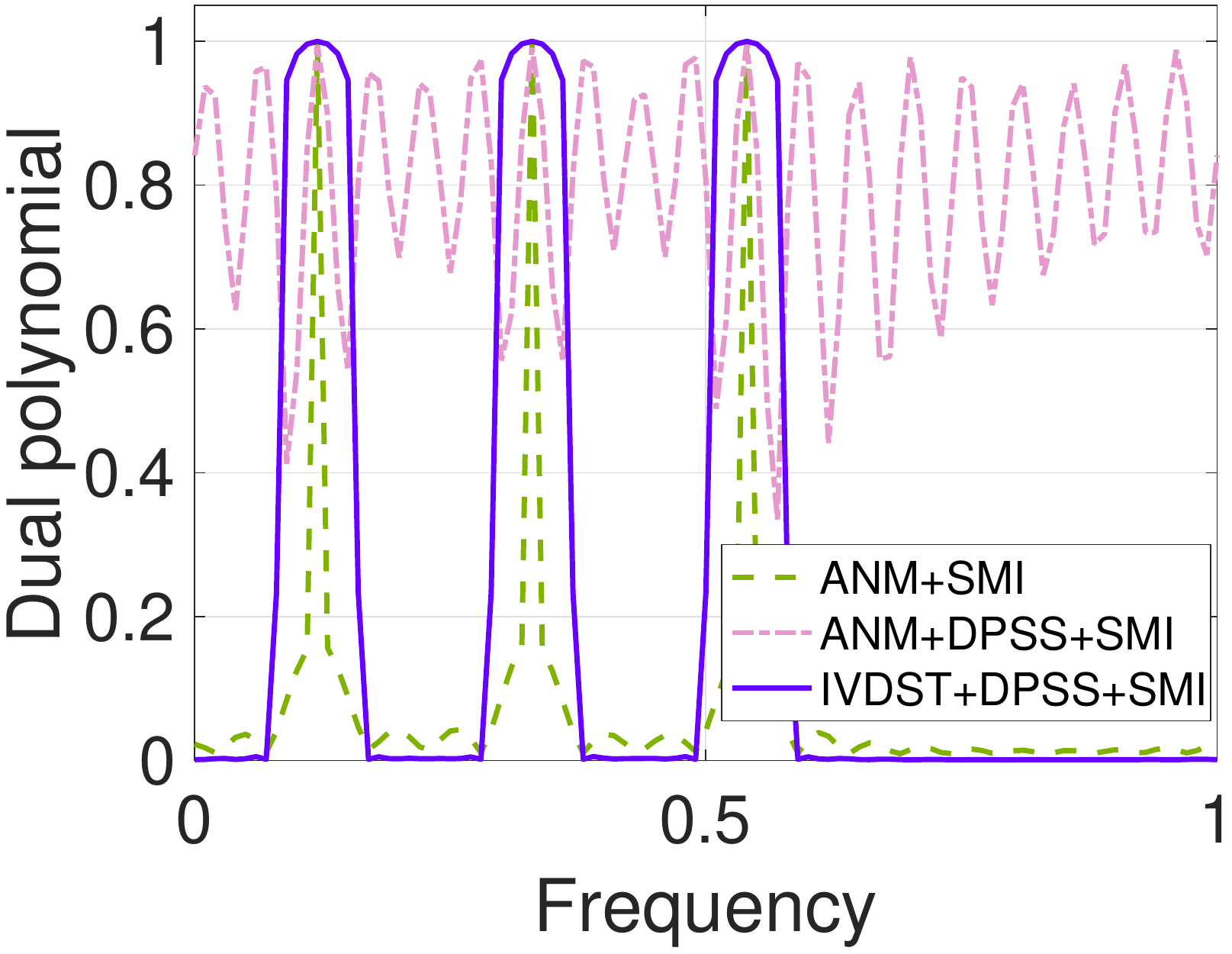}
\centerline{\footnotesize{(a) Static frequency offset}}
\end{minipage}
~
\begin{minipage}{0.48\linewidth}
\centering
\includegraphics[width=1.67in]{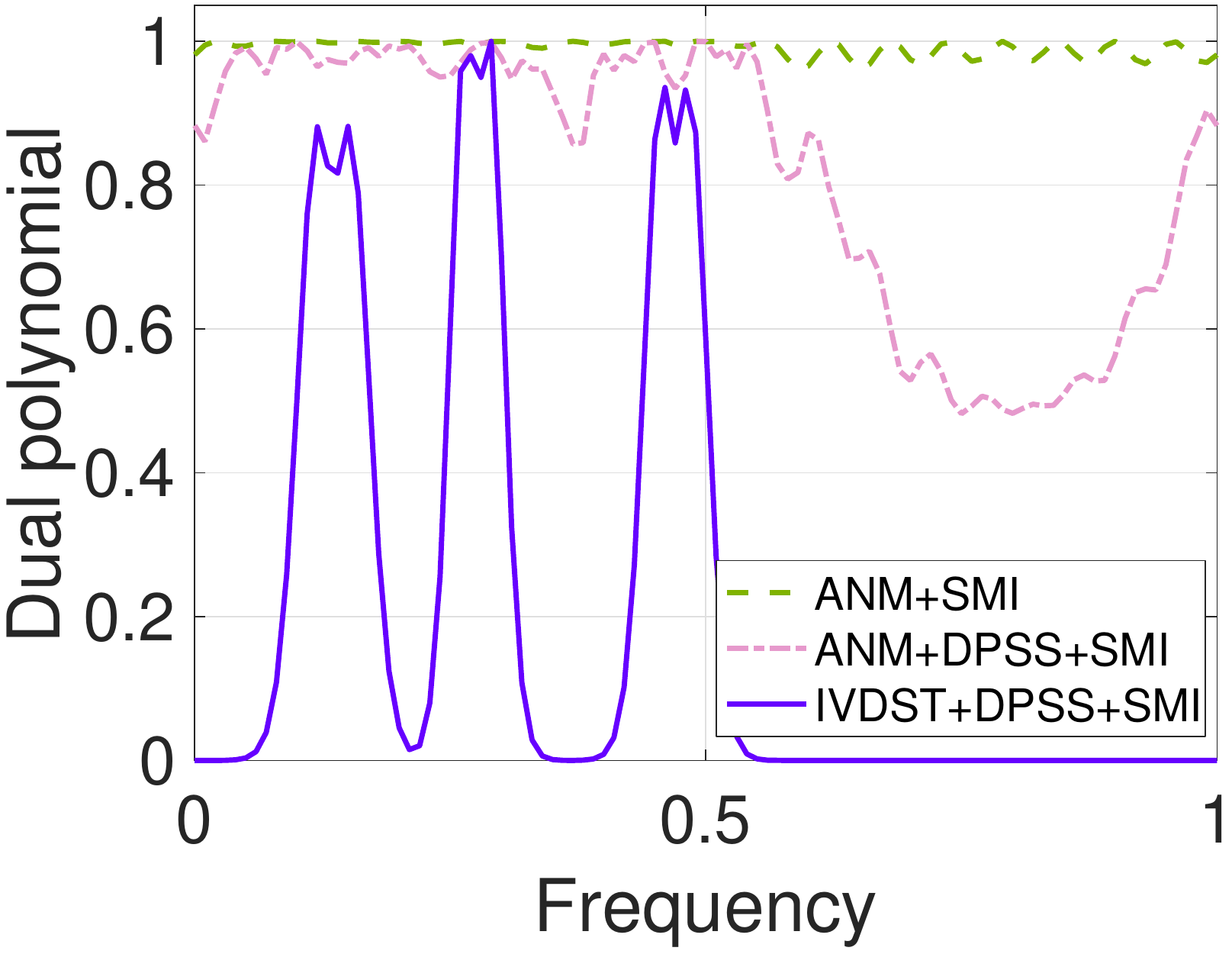}
\centerline{\footnotesize{(b) Linear frequency offset}}
\end{minipage}
\\
\begin{minipage}{0.48\linewidth}
\centering
\includegraphics[width=1.67in]{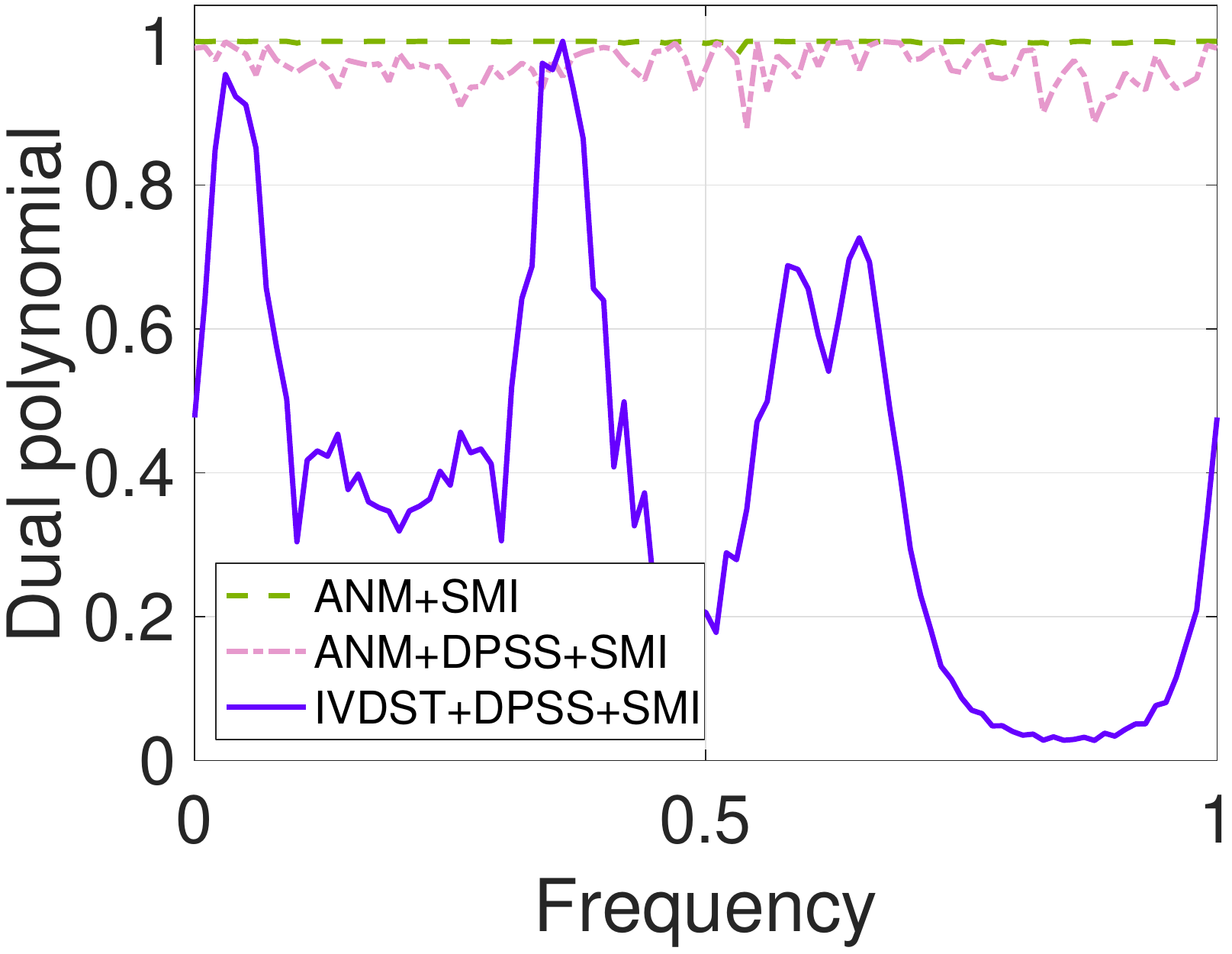}
\centerline{\footnotesize{(c) Zigzag frequency offset}}
\end{minipage}
~
\begin{minipage}{0.48\linewidth}
\centering
\includegraphics[width=1.67in]{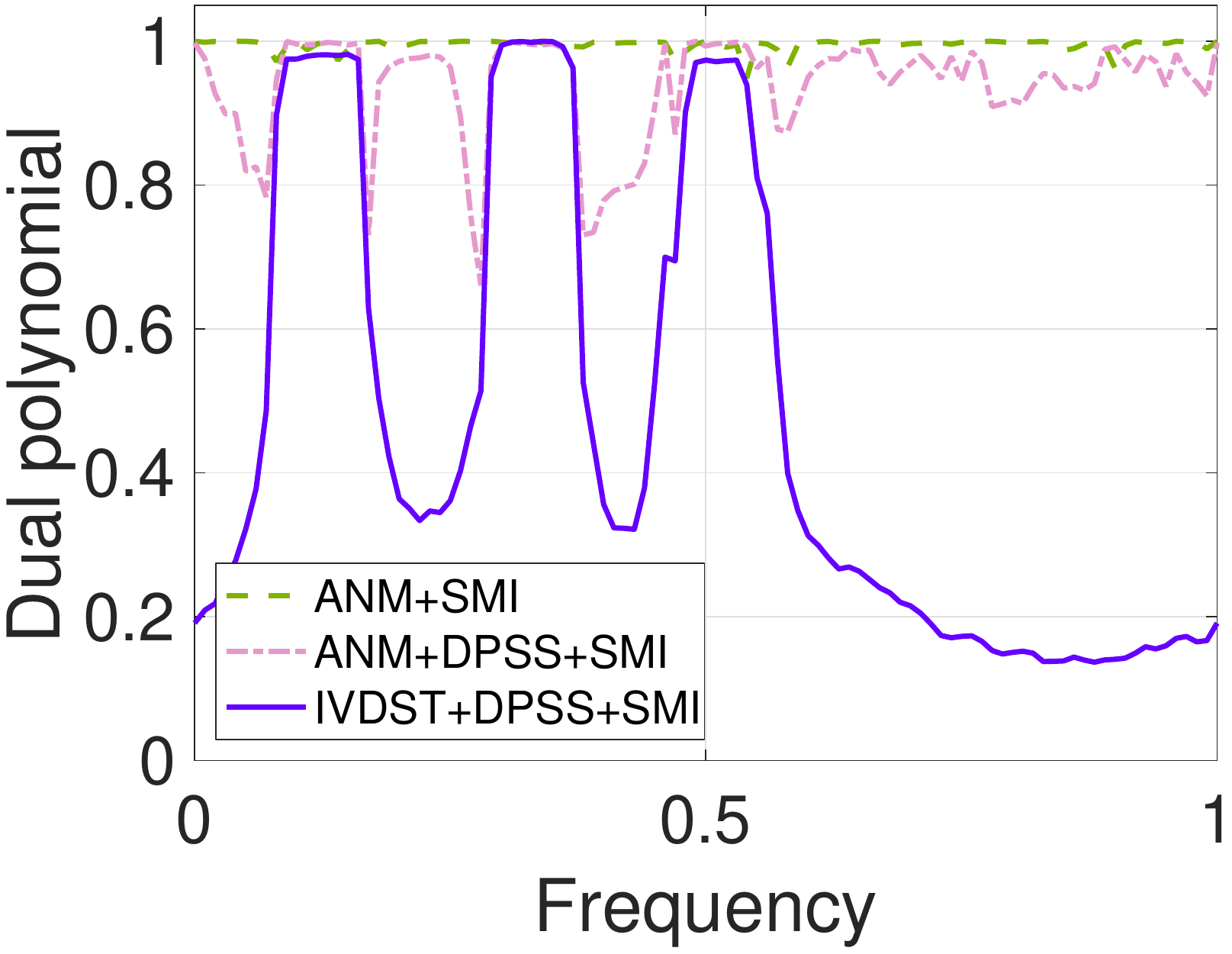}
\centerline{\footnotesize{(d) Random frequency offset}}
\end{minipage}
\caption{Dual polynomial obtained from the three methods in the four types of frequency offsets. ($M=120$)}
\label{fig:test_ANM_IVDST_DPSS_SMI_dual}
\end{figure}

\begin{figure}[ht]
\begin{minipage}{0.48\linewidth}
\centering
\includegraphics[width=1.67in]{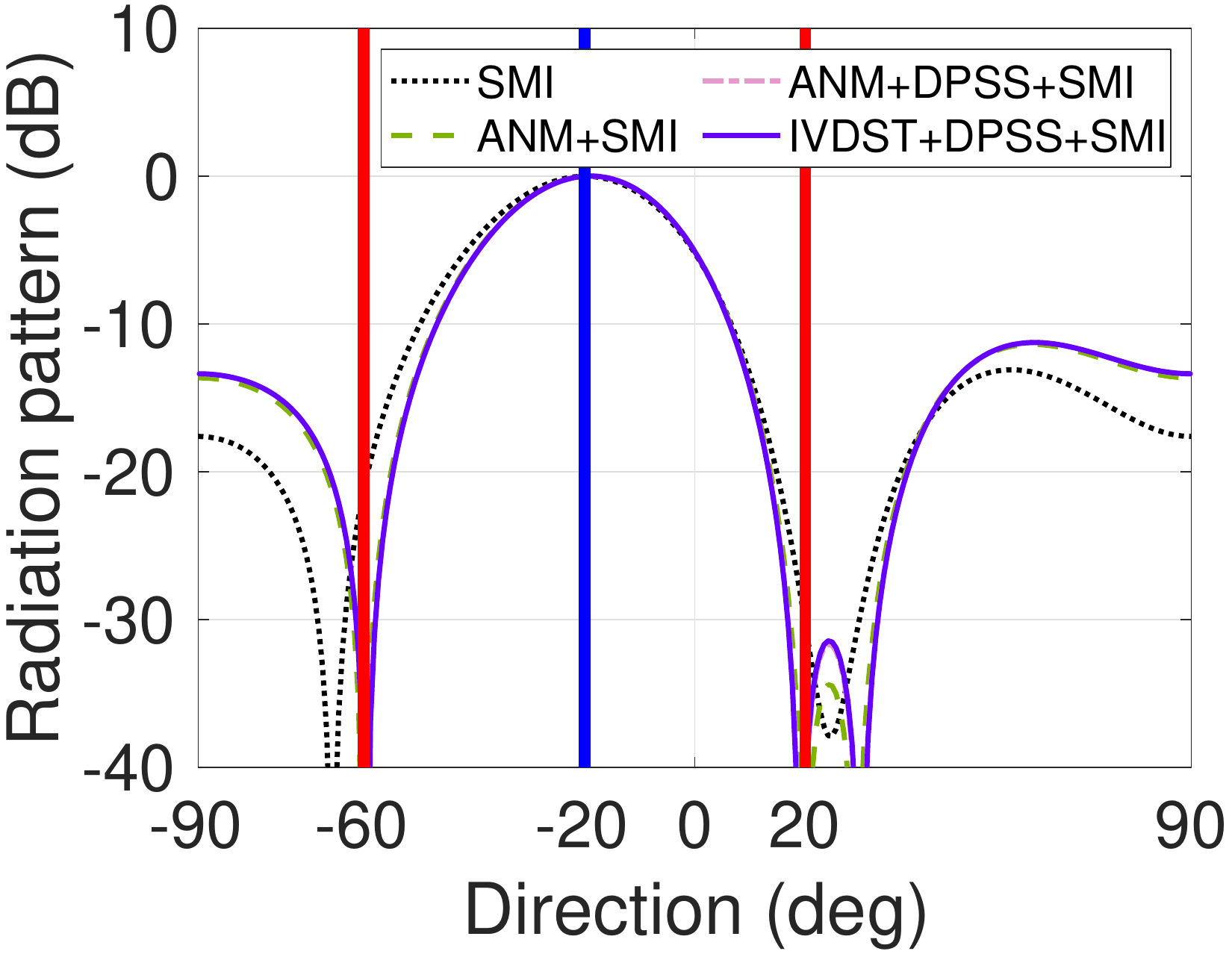}
\centerline{\footnotesize{(a) Static frequency offset}}
\end{minipage}
~
\begin{minipage}{0.48\linewidth}
\centering
\includegraphics[width=1.67in]{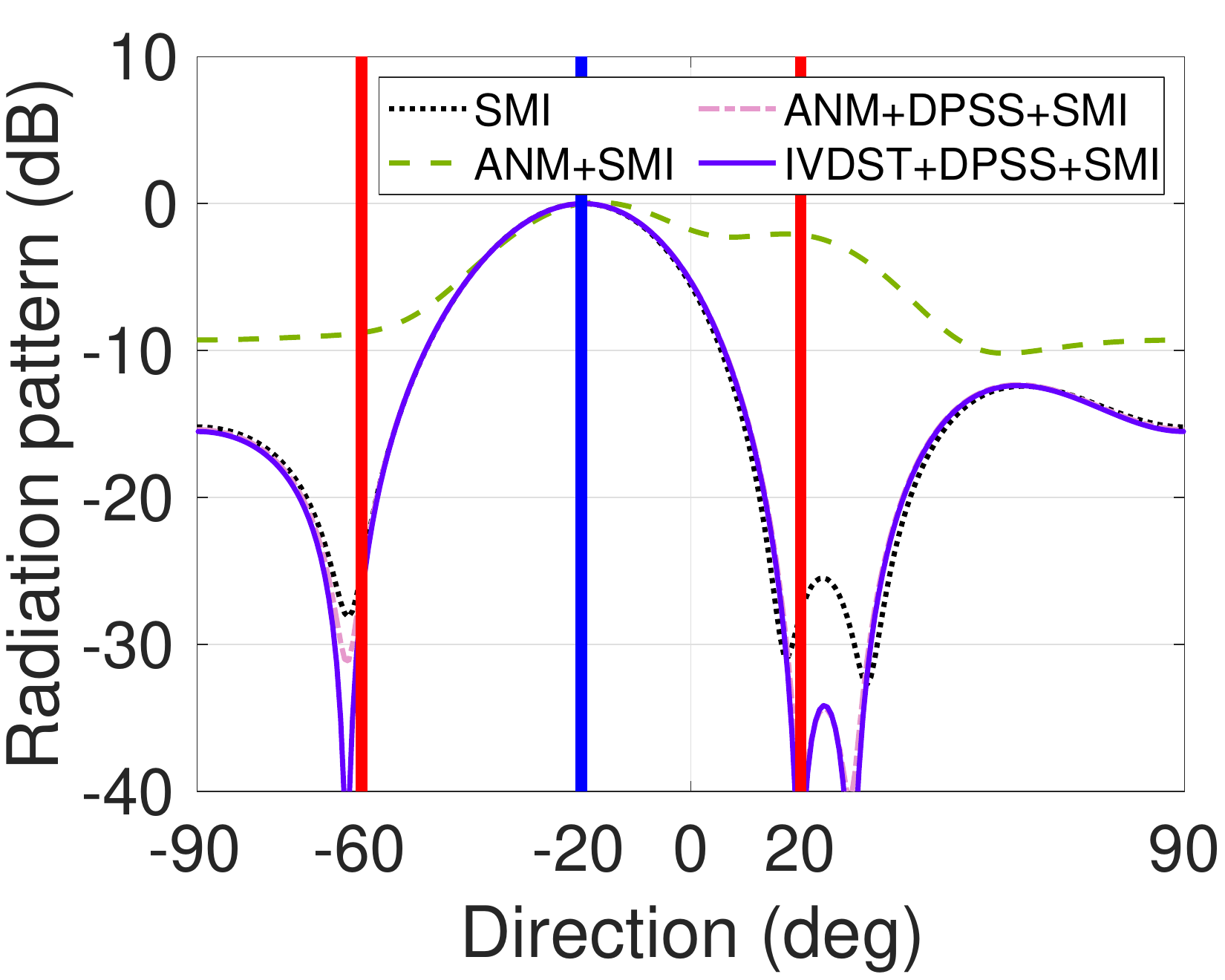}
\centerline{\footnotesize{(b) Linear frequency offset}}
\end{minipage}
\\
\begin{minipage}{0.48\linewidth}
\centering
\includegraphics[width=1.67in]{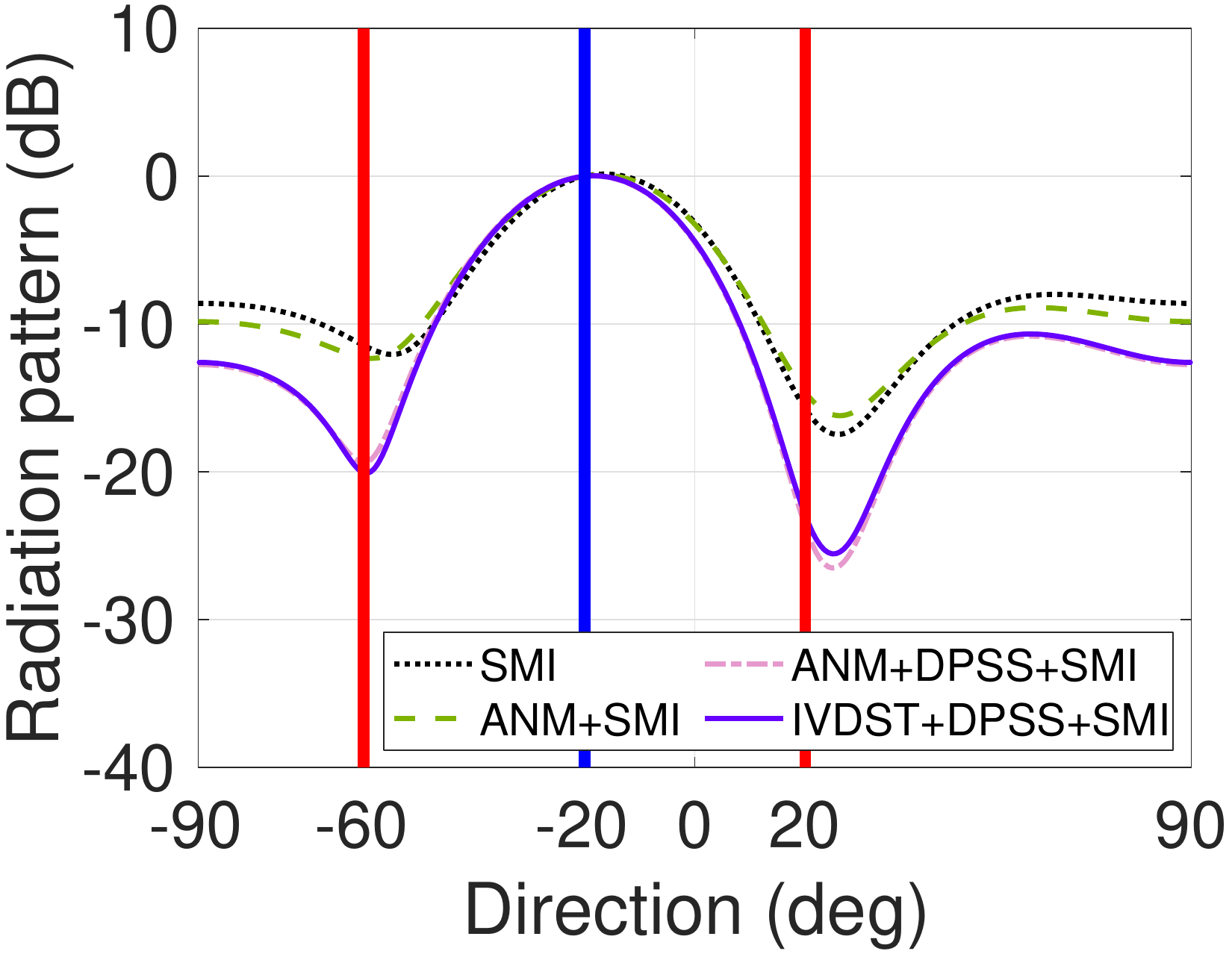}
\centerline{\footnotesize{(c) Zigzag frequency offset}}
\end{minipage}
~
\begin{minipage}{0.48\linewidth}
\centering
\includegraphics[width=1.67in]{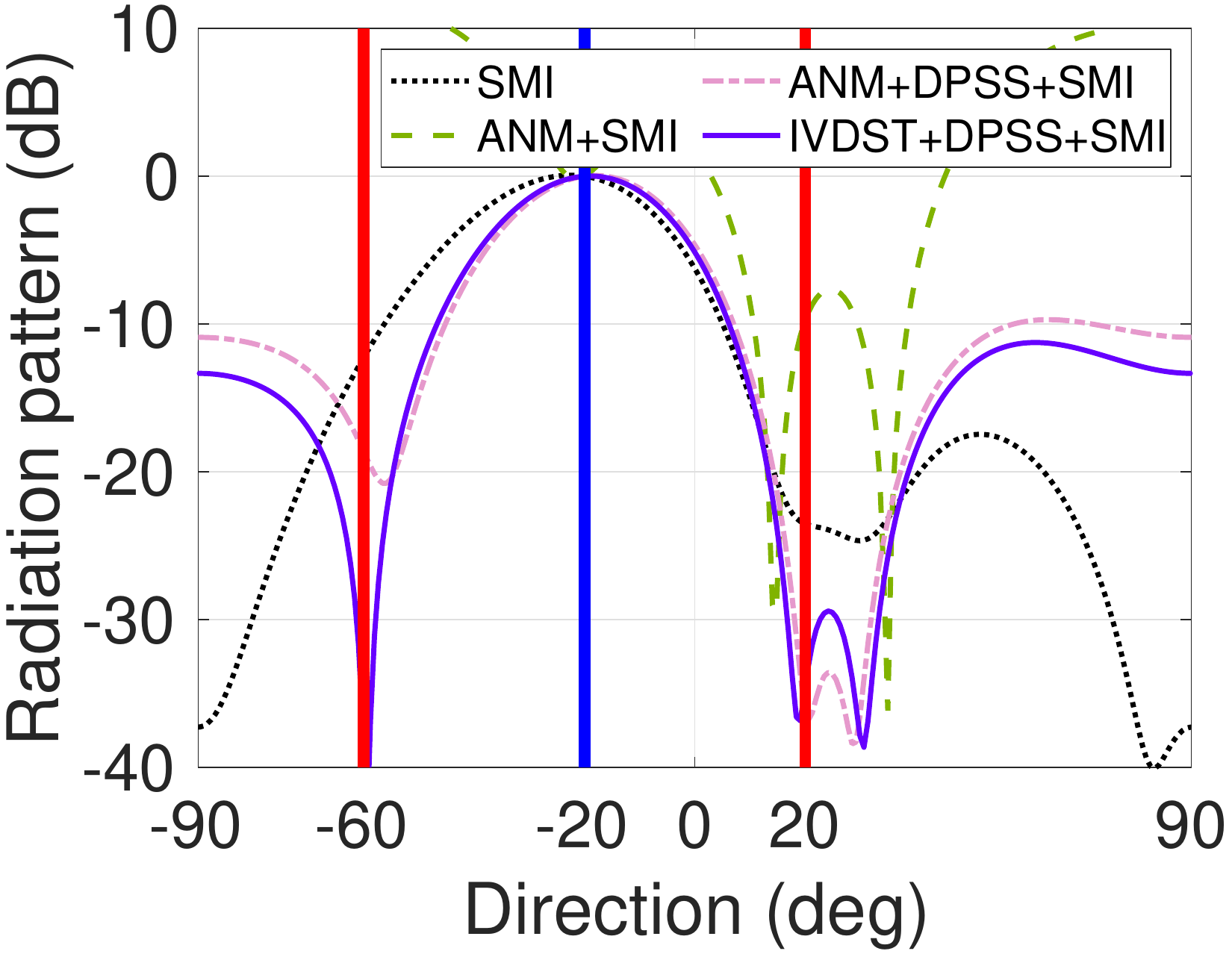}
\centerline{\footnotesize{(d) Random frequency offset}}
\end{minipage}
\caption{Interference cancellation with DBF in the four types of frequency offsets. ($M=120$)}
\label{fig:test_ANM_IVDST_DPSS_SMI_radi}
\end{figure}

\begin{table}[ht]
\caption{The value of $L$ used to generate the DPSS basis.} \label{TB:L}
\begin{center}
\begin{tabular}{|c|c|c|c|c|}
\hline
    &  Static & Linear& Zigzag & Random \\
\hline
      IVDST+DPSS+SMI &7&2&2&13\\
\hline
     ANM+DPSS+SMI &7&10&10&13\\
\hline
     2D-ANM+DPSS+SMI &4 & 4&5 & 4\\
\hline
\end{tabular}
\end{center}
\end{table}

\begin{table}[ht]
\caption{Running time (in seconds) of the ``IVDST+DPSS+SMI'' method and the ``ANM+DPSS+SMI'' method used in the first experiment.} \label{TB:time}
\begin{center}
\begin{tabular}{|c|c|c|c|c|}
\hline
    &  Static & Linear& Zigzag & Random \\
\hline
      IVDST+DPSS+SMI &20.25&19.81&19.07&19.57\\
\hline
     ANM+DPSS+SMI &470.23&506.52&479.42&941.30\\
\hline
\end{tabular}
\end{center}
\end{table}

To show that the proposed ``IVDST+DPSS+SMI'' method is suitable for problems with larger size, we repeat the above experiment with $M=300$, $(f_1^o, f_2^o, f_3^o) = (0.2,0.24,0.3)$, and $L = 2$. Other parameters are set the same as in the above experiment.  Since the SDP-based methods run slowly in this case, we compare only the ``IVDST+DPSS+SMI'' method with the SMI method. We present the time-varying frequency offsets, the dual polynomials, and the radiation pattern in Figures~\ref{fig:test_ANM_IVDST_DPSS_SMI_freqdrifth}-\ref{fig:test_ANM_IVDST_DPSS_SMI_radih}. It can be seen that proposed ``IVDST+DPSS+SMI'' method still works very well in this high-dimensional case and significantly outperforms SMI.

\begin{figure}[ht]
\begin{minipage}{0.48\linewidth}
\centering
\includegraphics[width=1.67in]{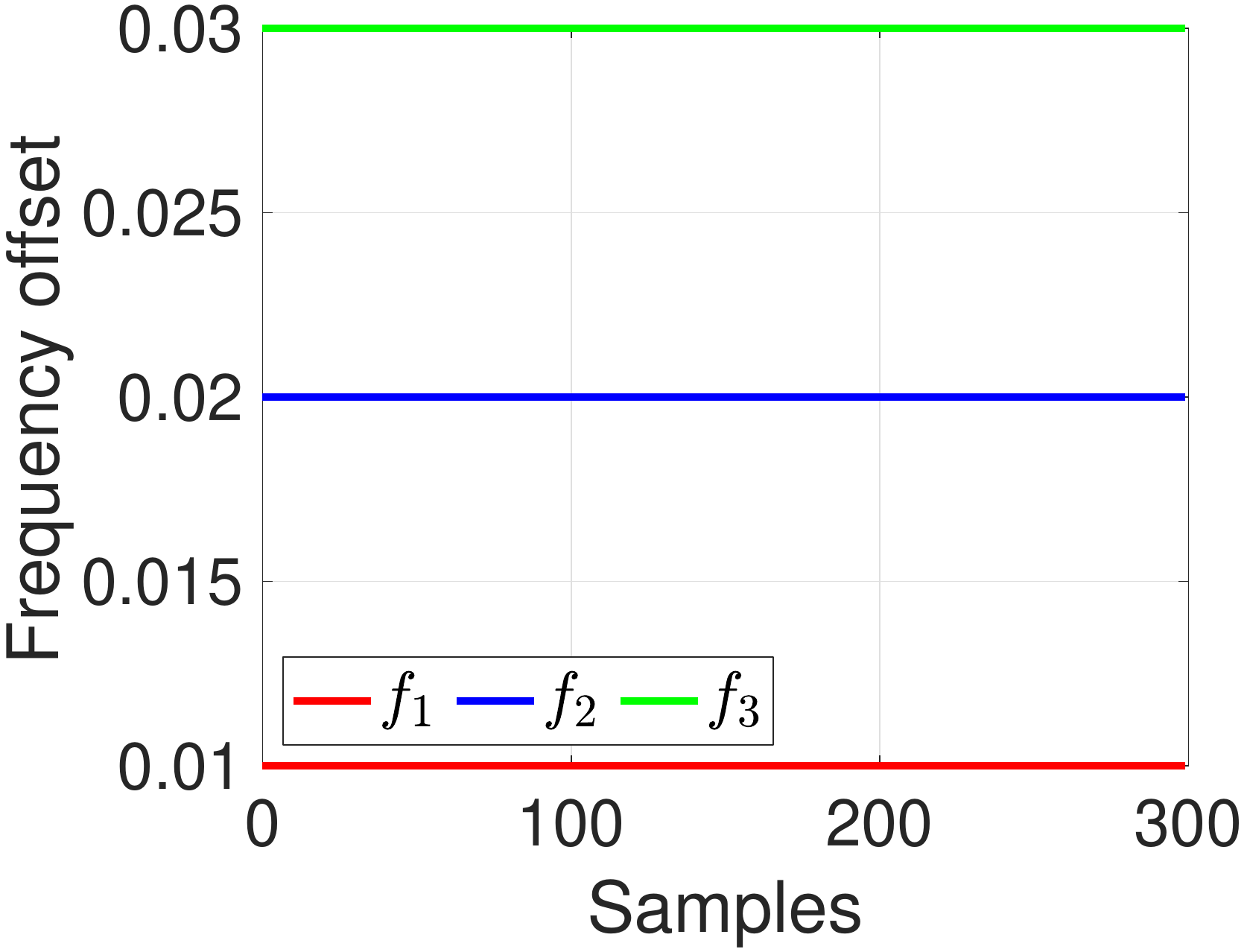}
\centerline{\footnotesize{(a) Static frequency offset}}
\end{minipage}
~
\begin{minipage}{0.48\linewidth}
\centering
\includegraphics[width=1.67in]{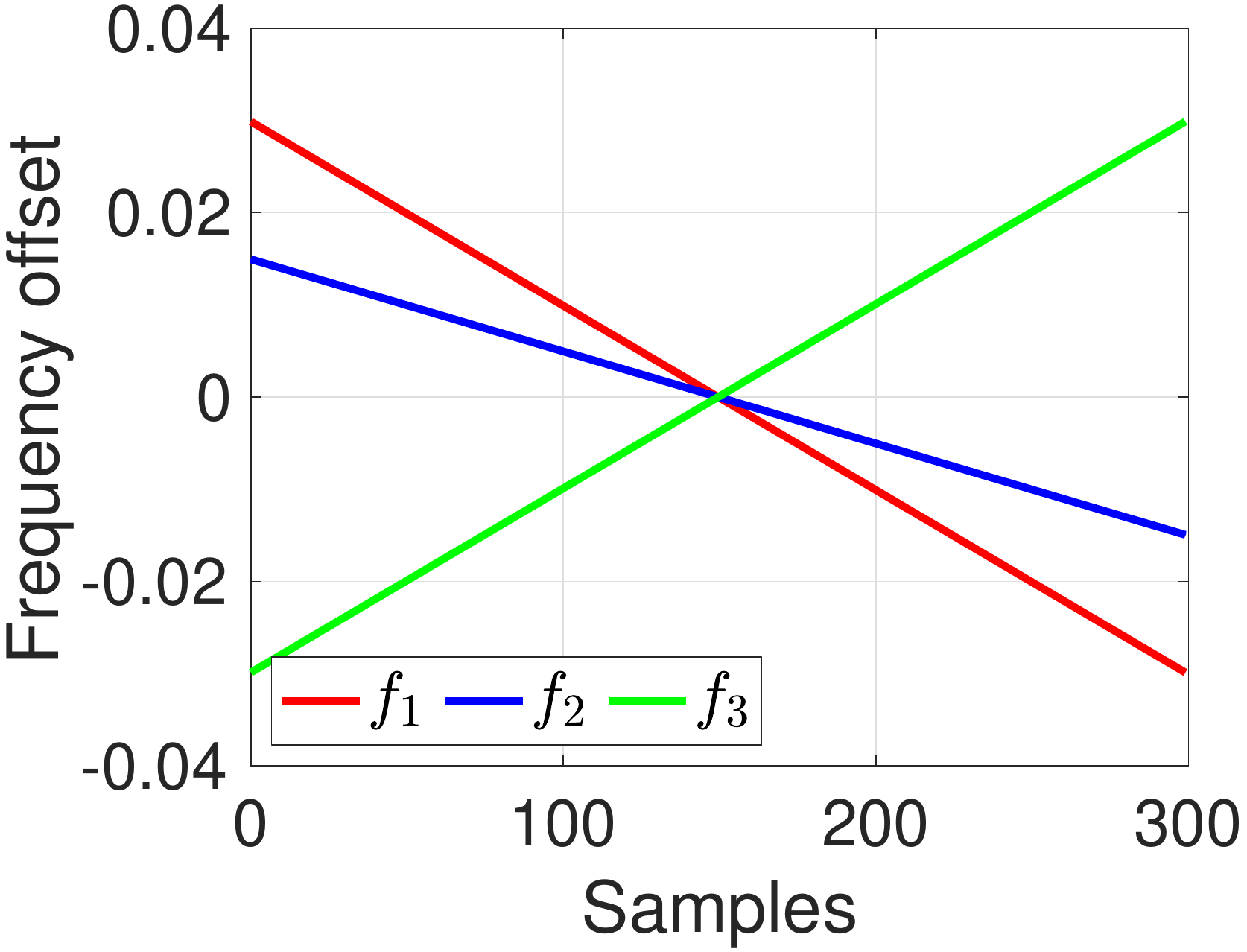}
\centerline{\footnotesize{(b) Linear frequency offset}}
\end{minipage}
\\
\begin{minipage}{0.48\linewidth}
\centering
\includegraphics[width=1.67in]{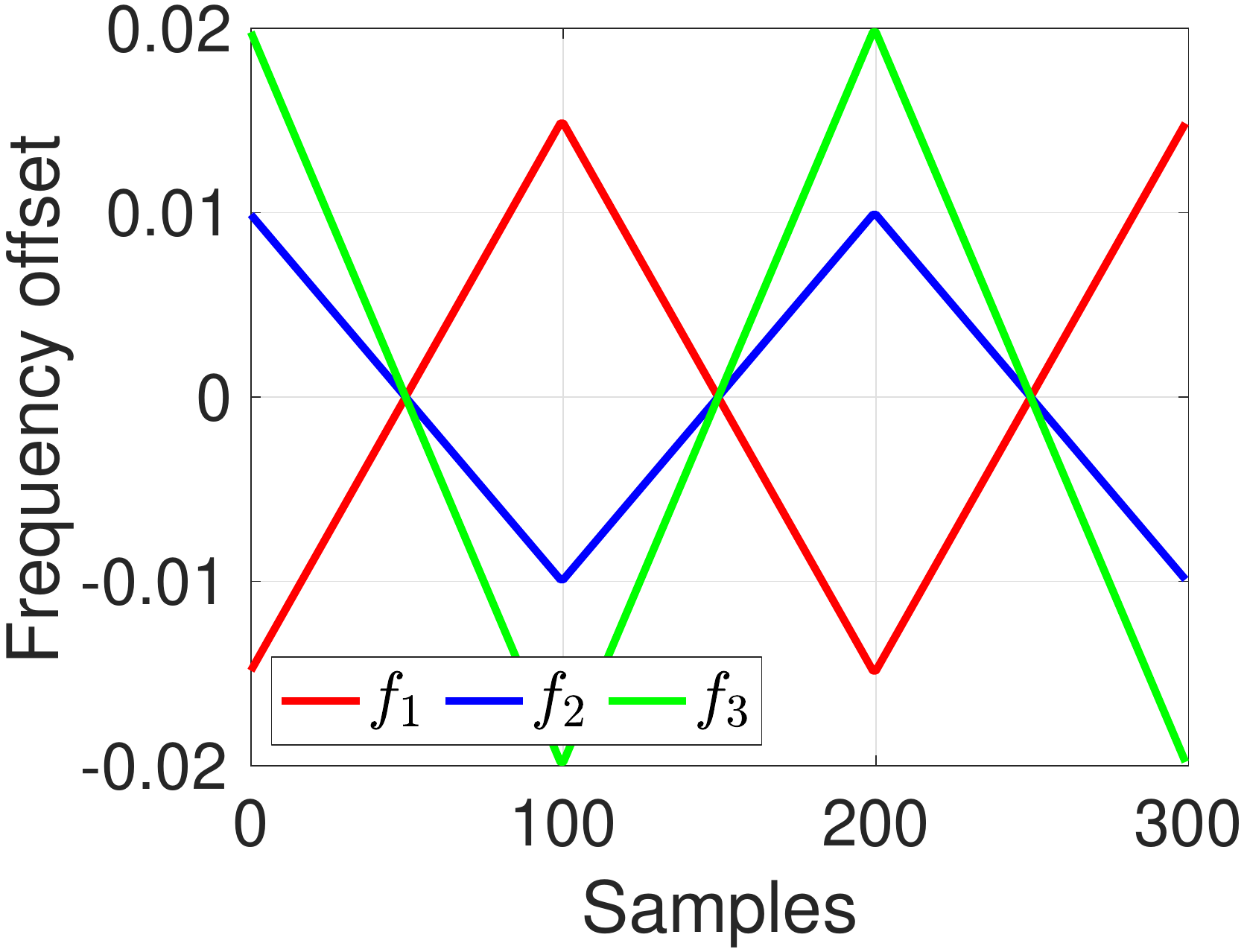}
\centerline{\footnotesize{(c) Zigzag frequency offset}}
\end{minipage}
~
\begin{minipage}{0.48\linewidth}
\centering
\includegraphics[width=1.67in]{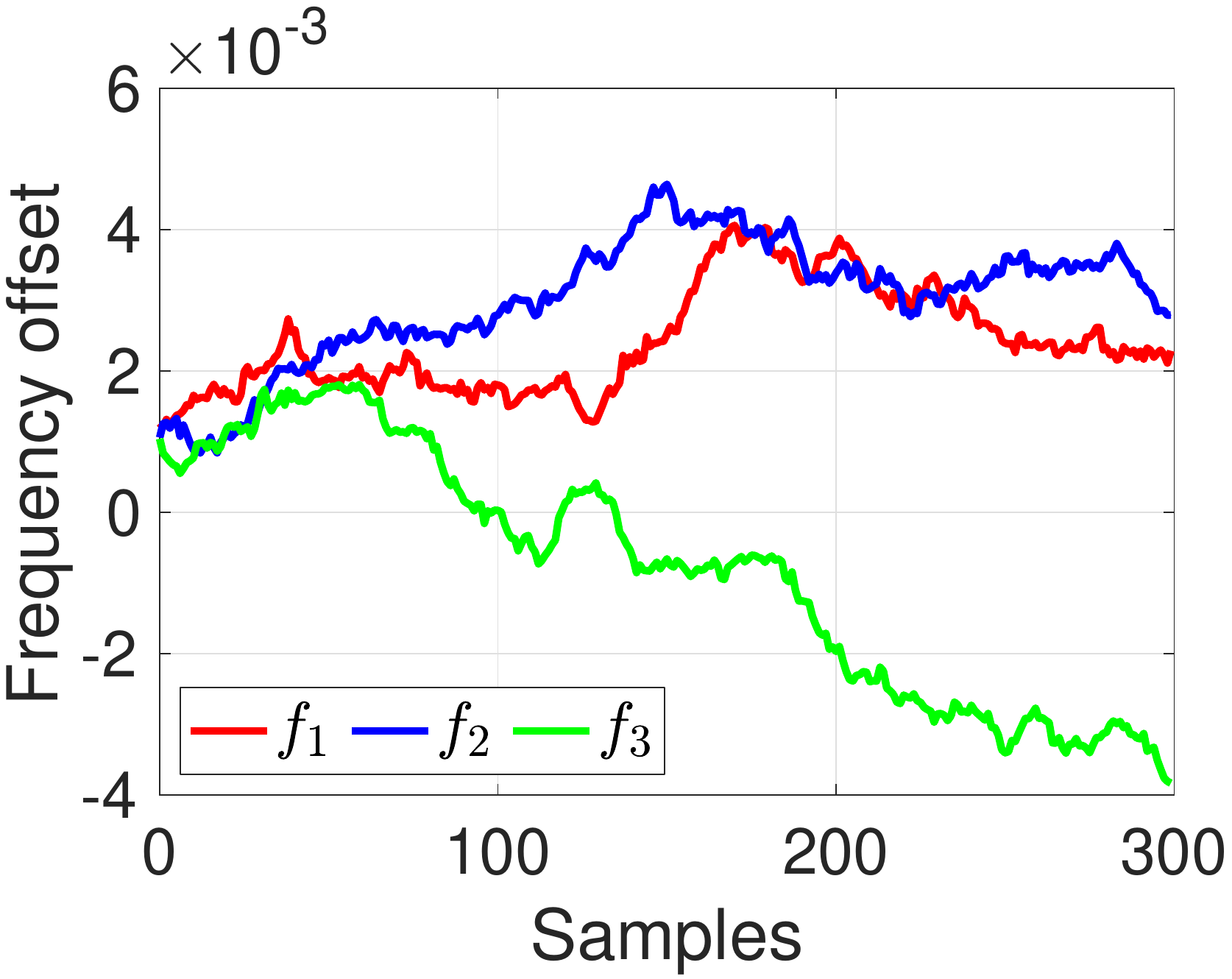}
\centerline{\footnotesize{(d) Random frequency offset}}
\end{minipage}
\caption{Four types of frequency offsets used in the experiments. ($M=300$)}
\label{fig:test_ANM_IVDST_DPSS_SMI_freqdrifth}
\end{figure}

\begin{figure}[ht]
\begin{minipage}{0.48\linewidth}
\centering
\includegraphics[width=1.67in]{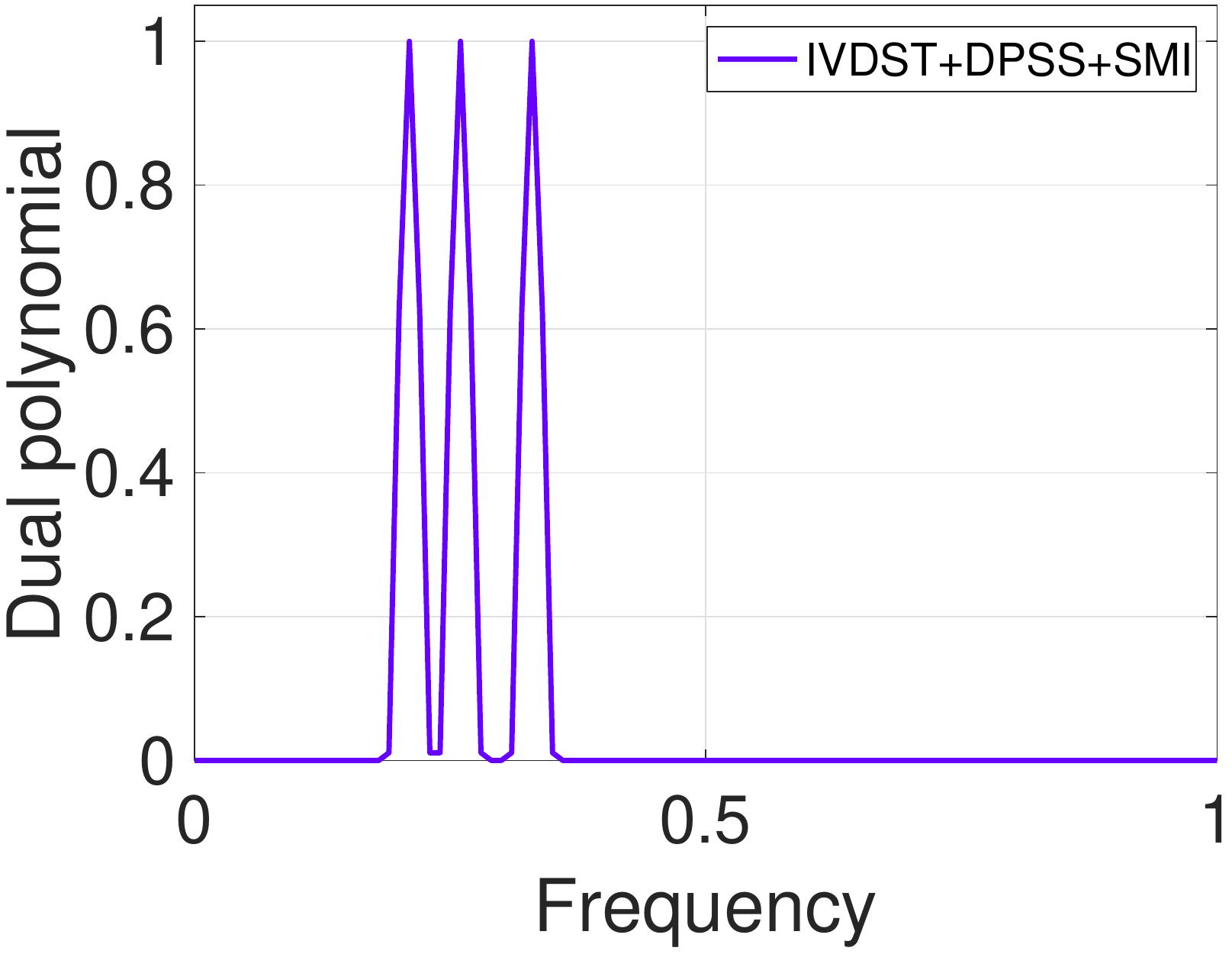}
\centerline{\footnotesize{(a) Static frequency offset}}
\end{minipage}
~
\begin{minipage}{0.48\linewidth}
\centering
\includegraphics[width=1.67in]{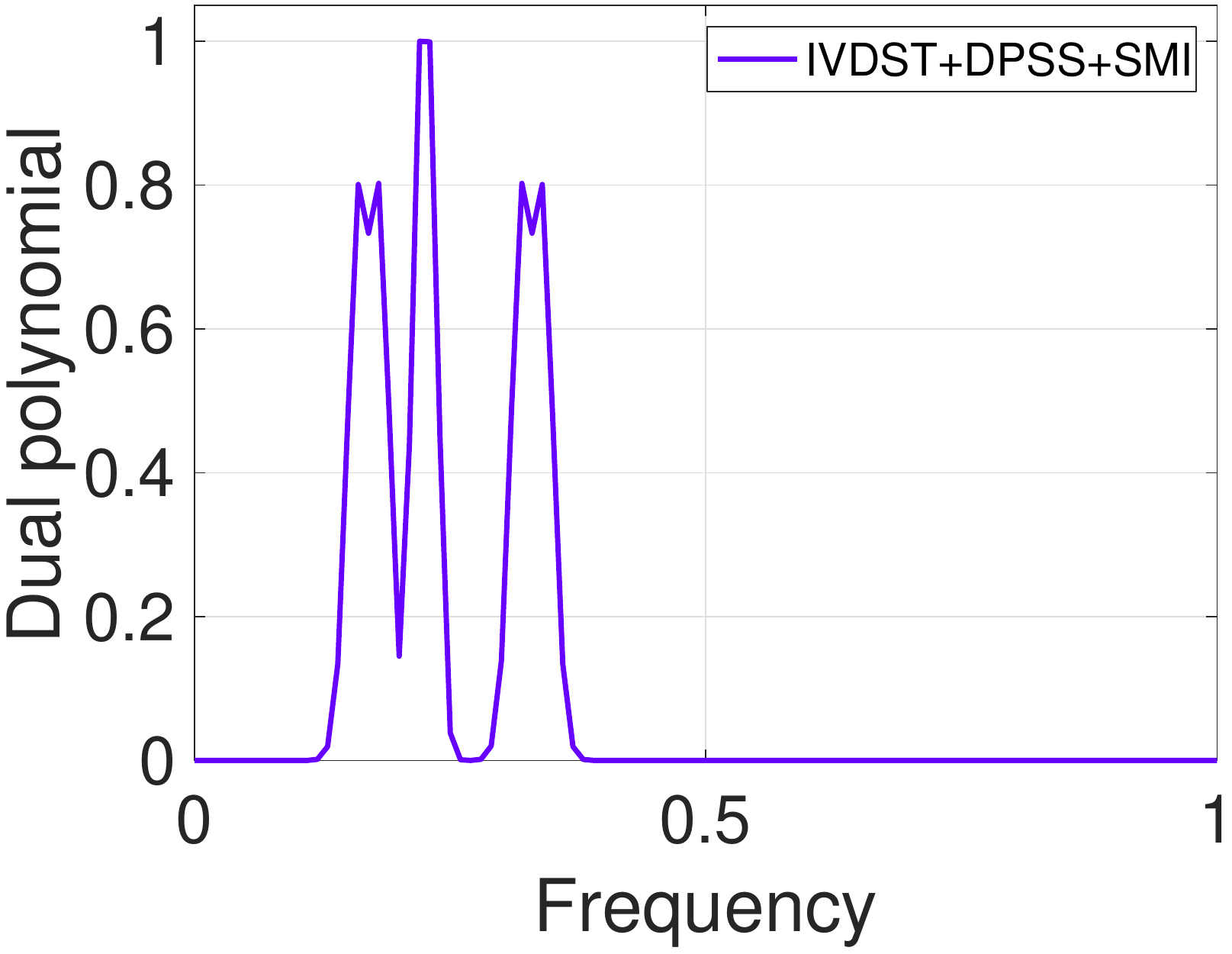}
\centerline{\footnotesize{(b) Linear frequency offset}}
\end{minipage}
\\
\begin{minipage}{0.48\linewidth}
\centering
\includegraphics[width=1.67in]{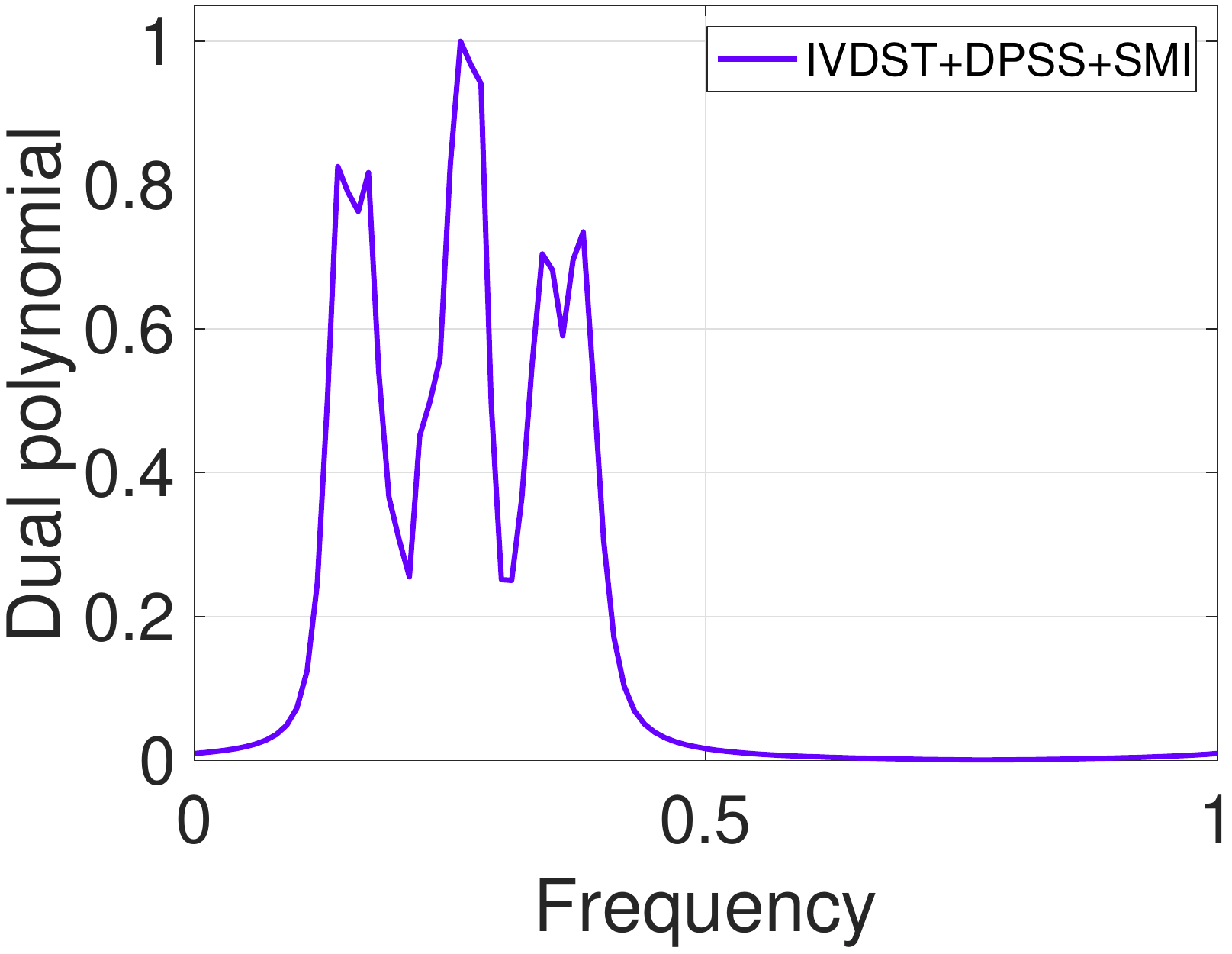}
\centerline{\footnotesize{(c) Zigzag frequency offset}}
\end{minipage}
~
\begin{minipage}{0.48\linewidth}
\centering
\includegraphics[width=1.67in]{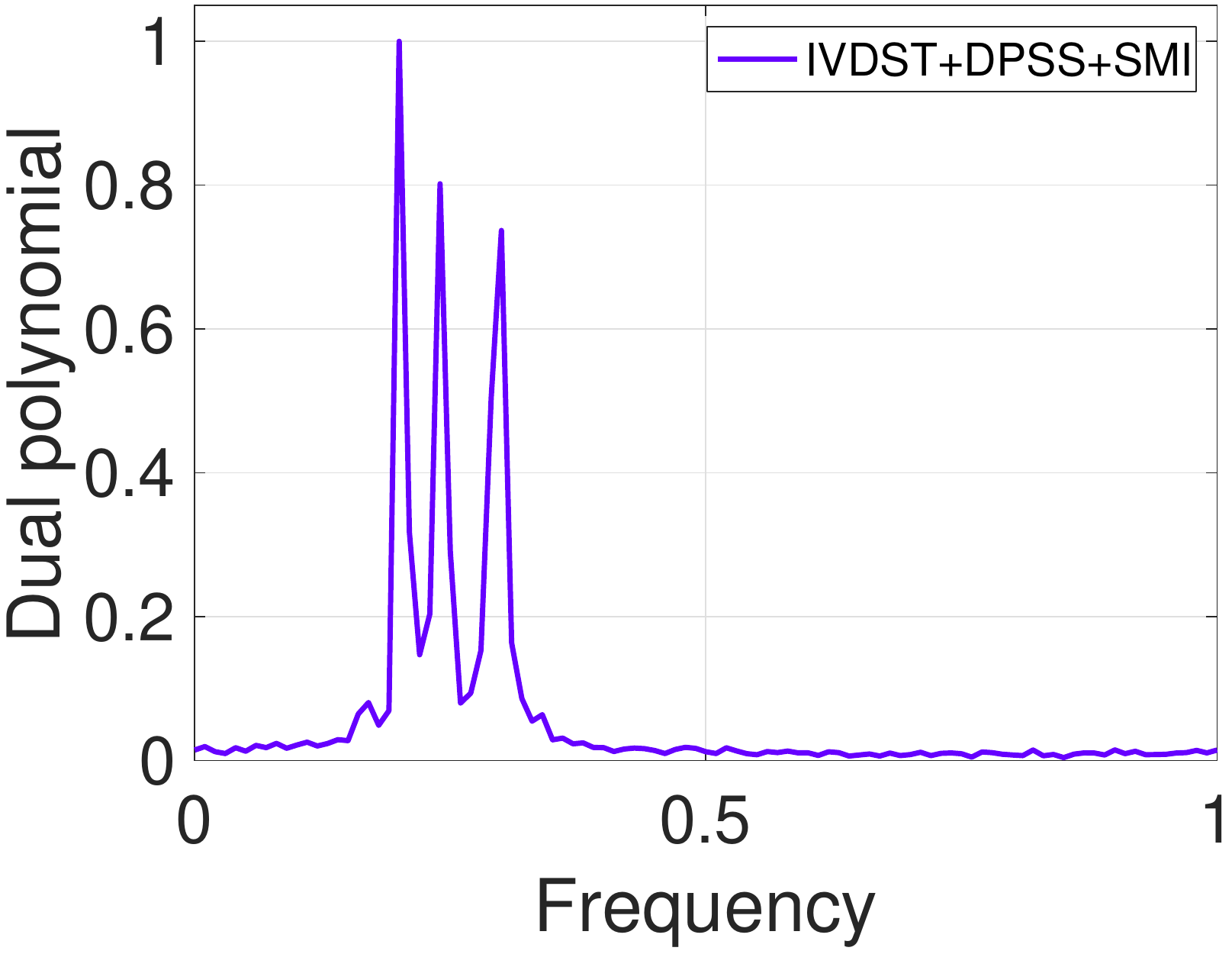}
\centerline{\footnotesize{(d) Random frequency offset}}
\end{minipage}
\caption{Dual polynomial obtained from the three methods in the four types of frequency offsets. ($M=300$)}
\label{fig:test_ANM_IVDST_DPSS_SMI_dualh}
\end{figure}

\begin{figure}[ht]
\begin{minipage}{0.48\linewidth}
\centering
\includegraphics[width=1.67in]{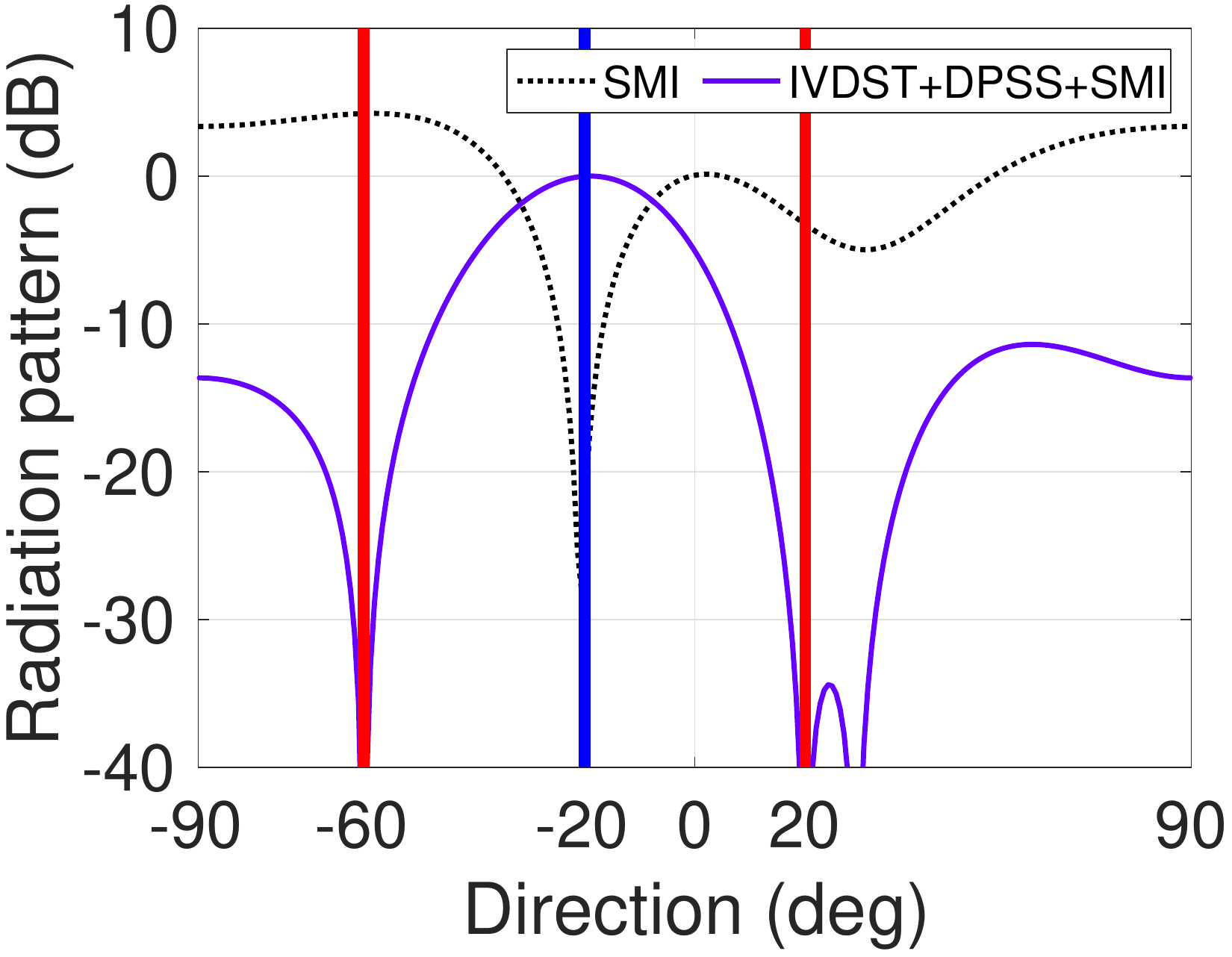}
\centerline{\footnotesize{(a) Static frequency offset}}
\end{minipage}
~
\begin{minipage}{0.48\linewidth}
\centering
\includegraphics[width=1.67in]{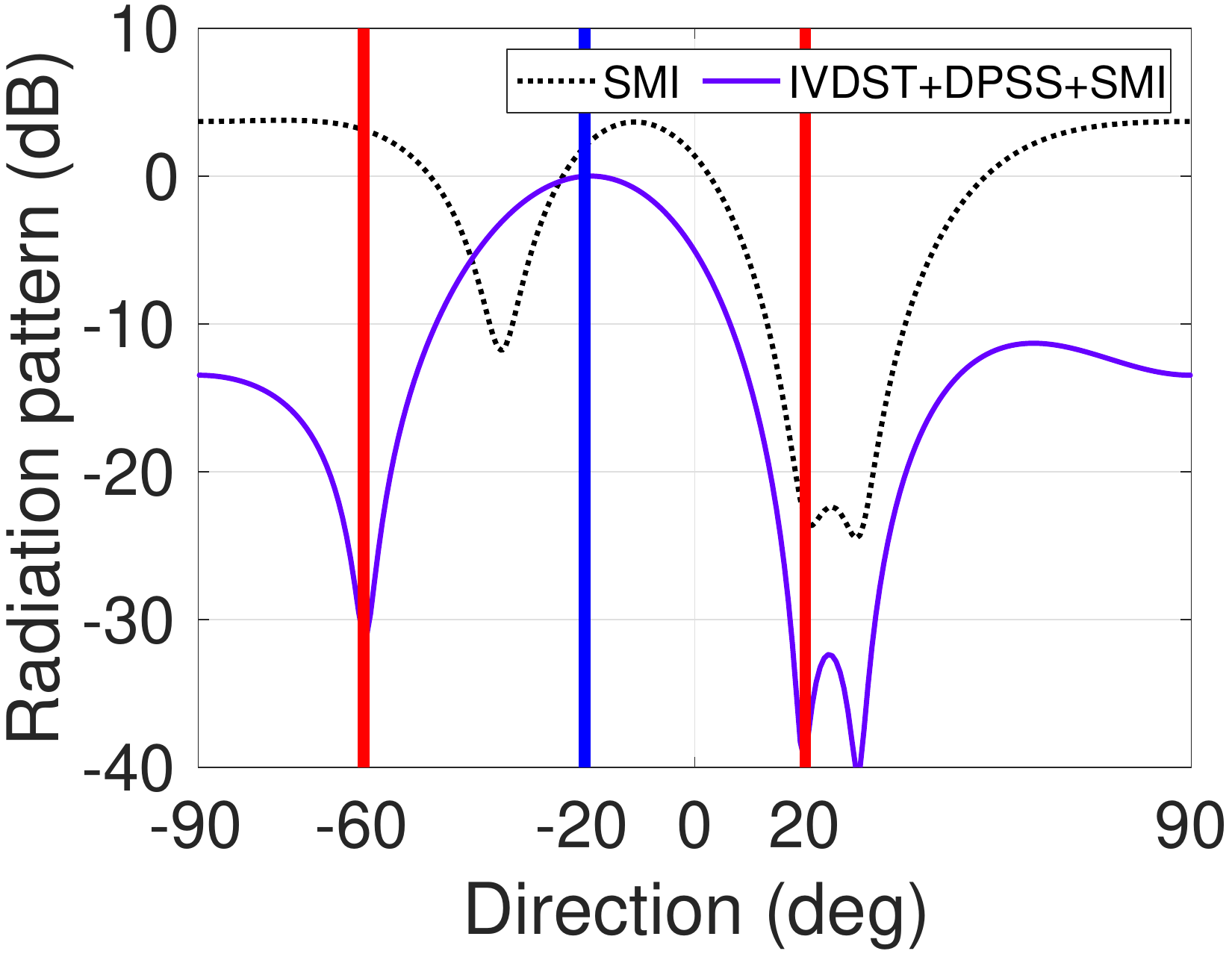}
\centerline{\footnotesize{(b) Linear frequency offset}}
\end{minipage}
\\
\begin{minipage}{0.48\linewidth}
\centering
\includegraphics[width=1.67in]{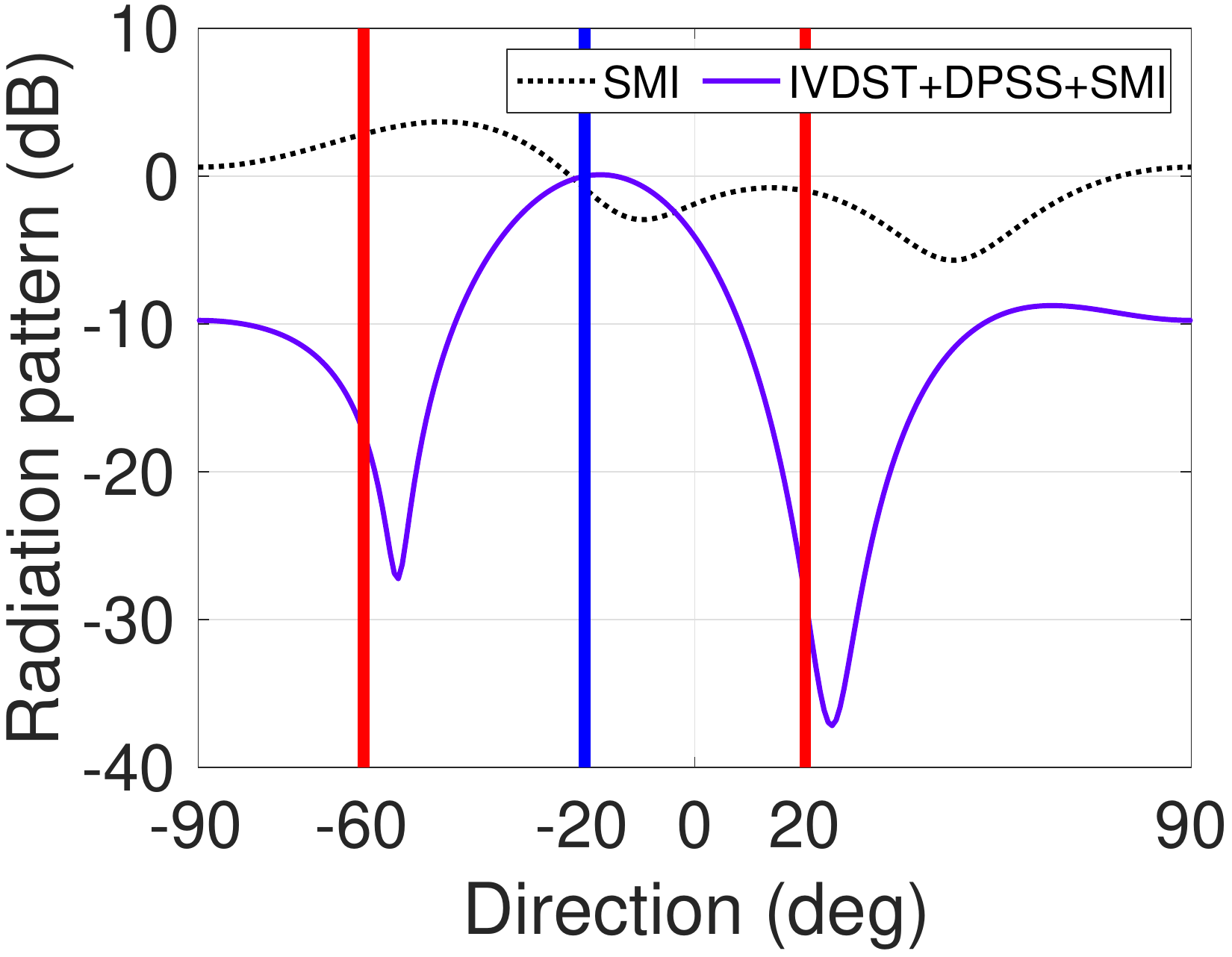}
\centerline{\footnotesize{(c) Zigzag frequency offset}}
\end{minipage}
~
\begin{minipage}{0.48\linewidth}
\centering
\includegraphics[width=1.67in]{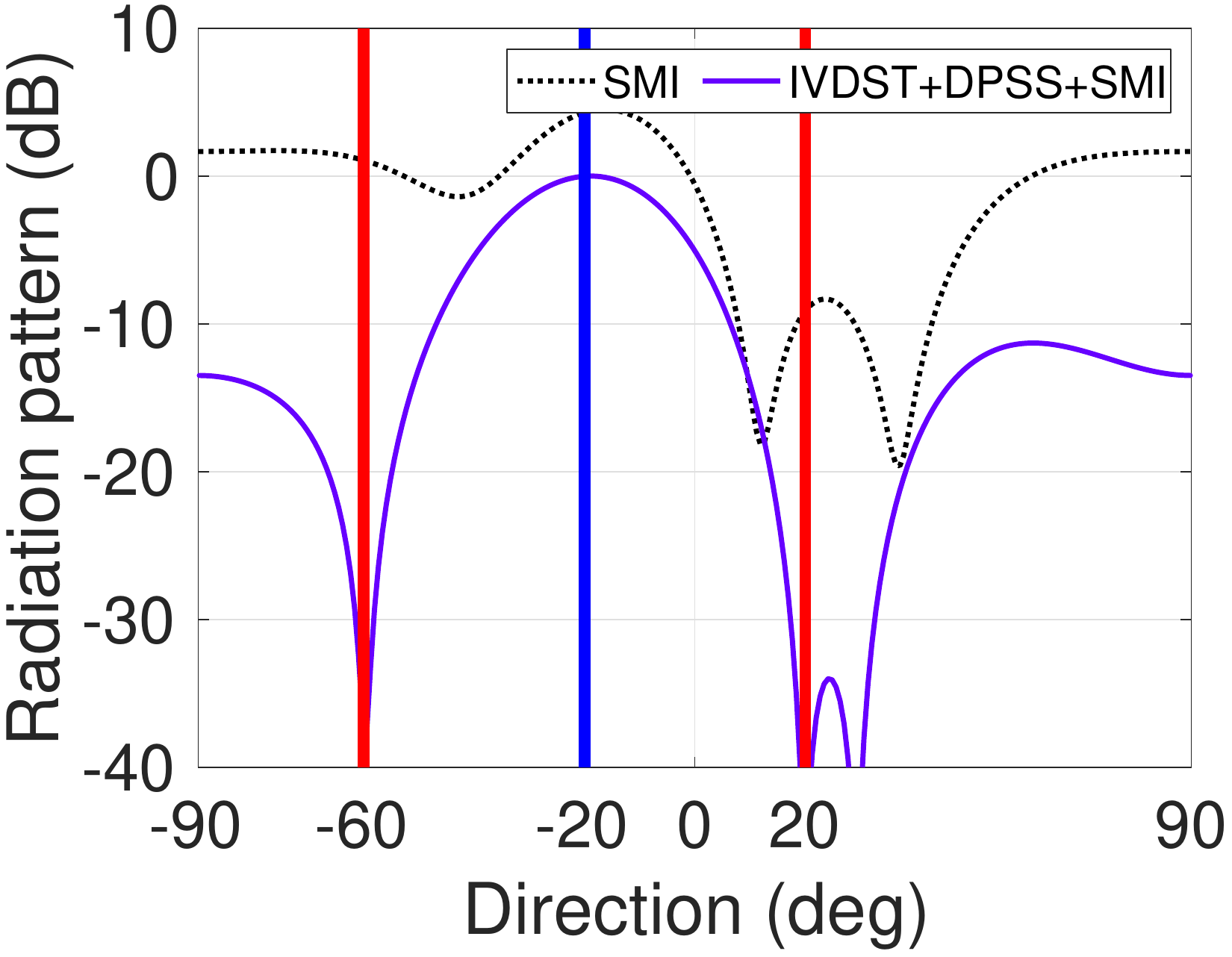}
\centerline{\footnotesize{(d) Random frequency offset}}
\end{minipage}
\caption{Interference cancellation with DBF in the four types of frequency offsets. ($M=300$)}
\label{fig:test_ANM_IVDST_DPSS_SMI_radih}
\end{figure}

Next, we compare the proposed ``2D-ANM+DPSS+SMI'' method with the SMI method. We repeat the first experiment with $M=15$\footnote{For zigzag frequency offset, we set $M=30$.} and $(f_1^o, f_2^o, f_3^o) = (0.2,0.7,0.7)$. Other parameters are set same as the first experiment. We present the four types of frequency offsets in Figure~\ref{fig:test_ANM_DPSS_SMI_2D_freqdrift} and the four values of $L$ in Table~\ref{TB:L}. We present the dual polynomials obtained from the ``2D-ANM+DPSS+SMI'' method in Figure~\ref{fig:test_ANM_DPSS_SMI_2D_dual} and the corresponding radiation pattern in Figure~\ref{fig:test_ANM_DPSS_SMI_2D_radi}. It can be seen that the ``2D-ANM+DPSS+SMI'' method still significantly outperforms the SMI method.

\begin{figure}[ht]
\begin{minipage}{0.48\linewidth}
\centering
\includegraphics[width=1.67in]{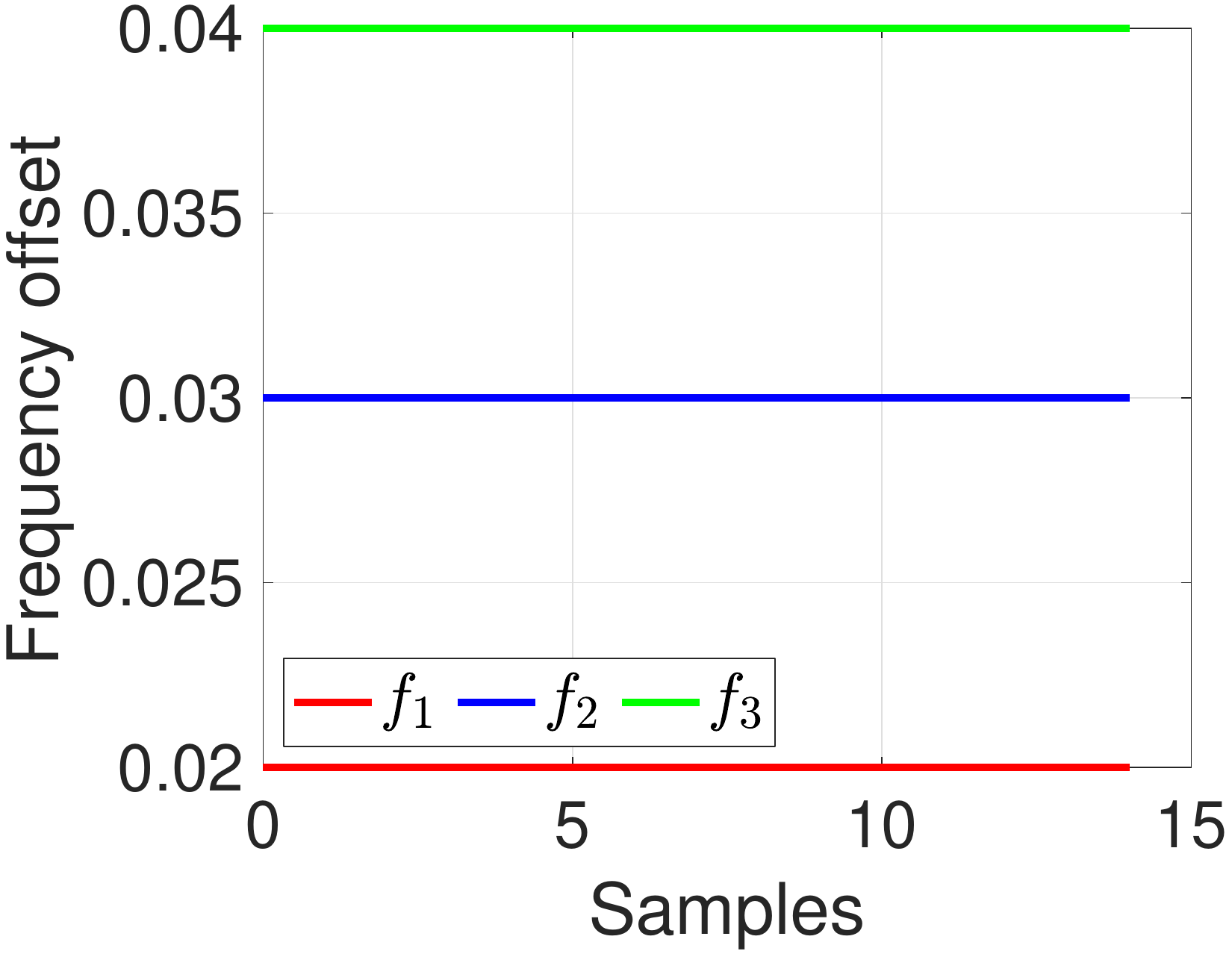}
\centerline{\footnotesize{(a) Static frequency offset}}
\end{minipage}
~
\begin{minipage}{0.48\linewidth}
\centering
\includegraphics[width=1.67in]{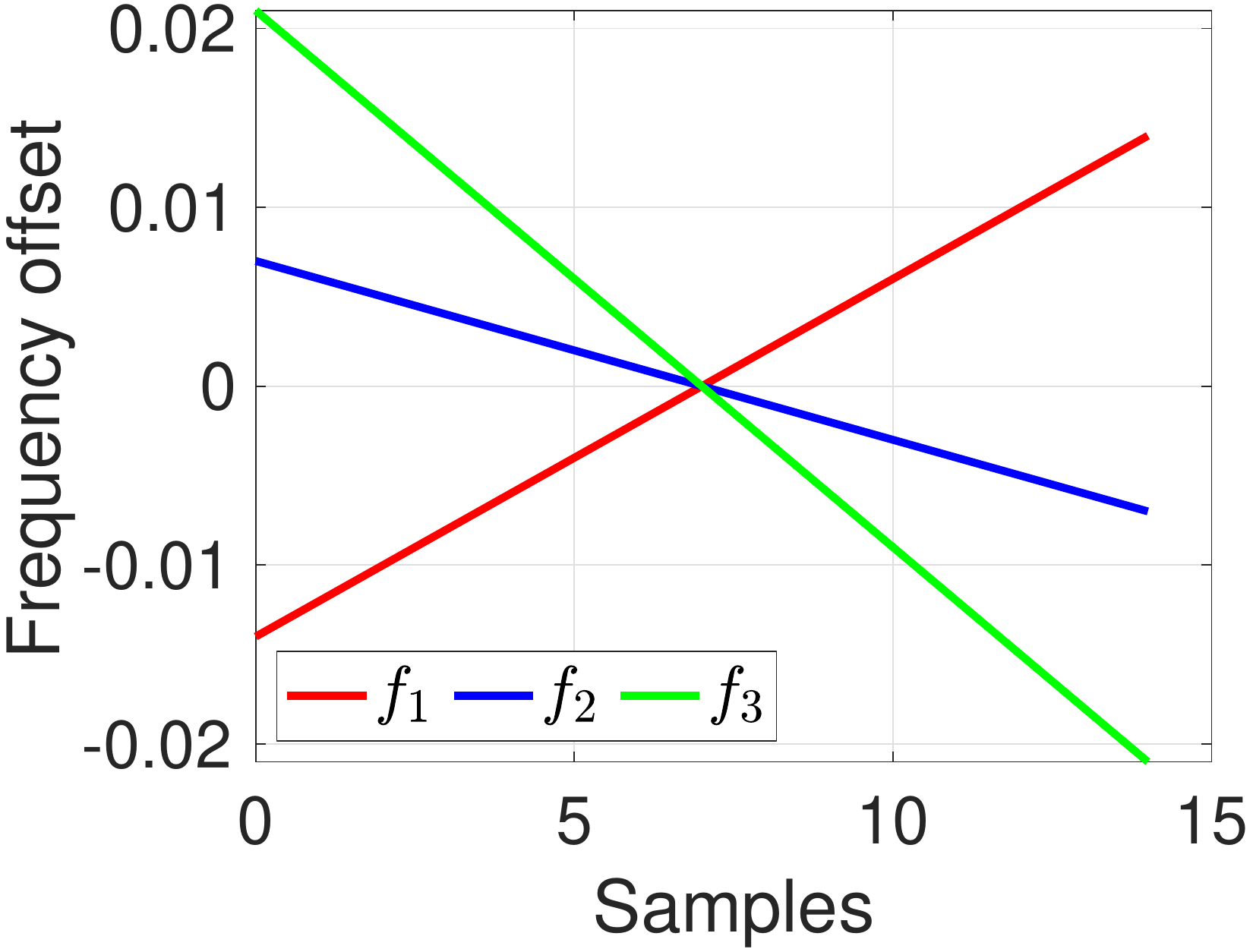}
\centerline{\footnotesize{(b) Linear frequency offset}}
\end{minipage}
\\
\begin{minipage}{0.48\linewidth}
\centering
\includegraphics[width=1.67in]{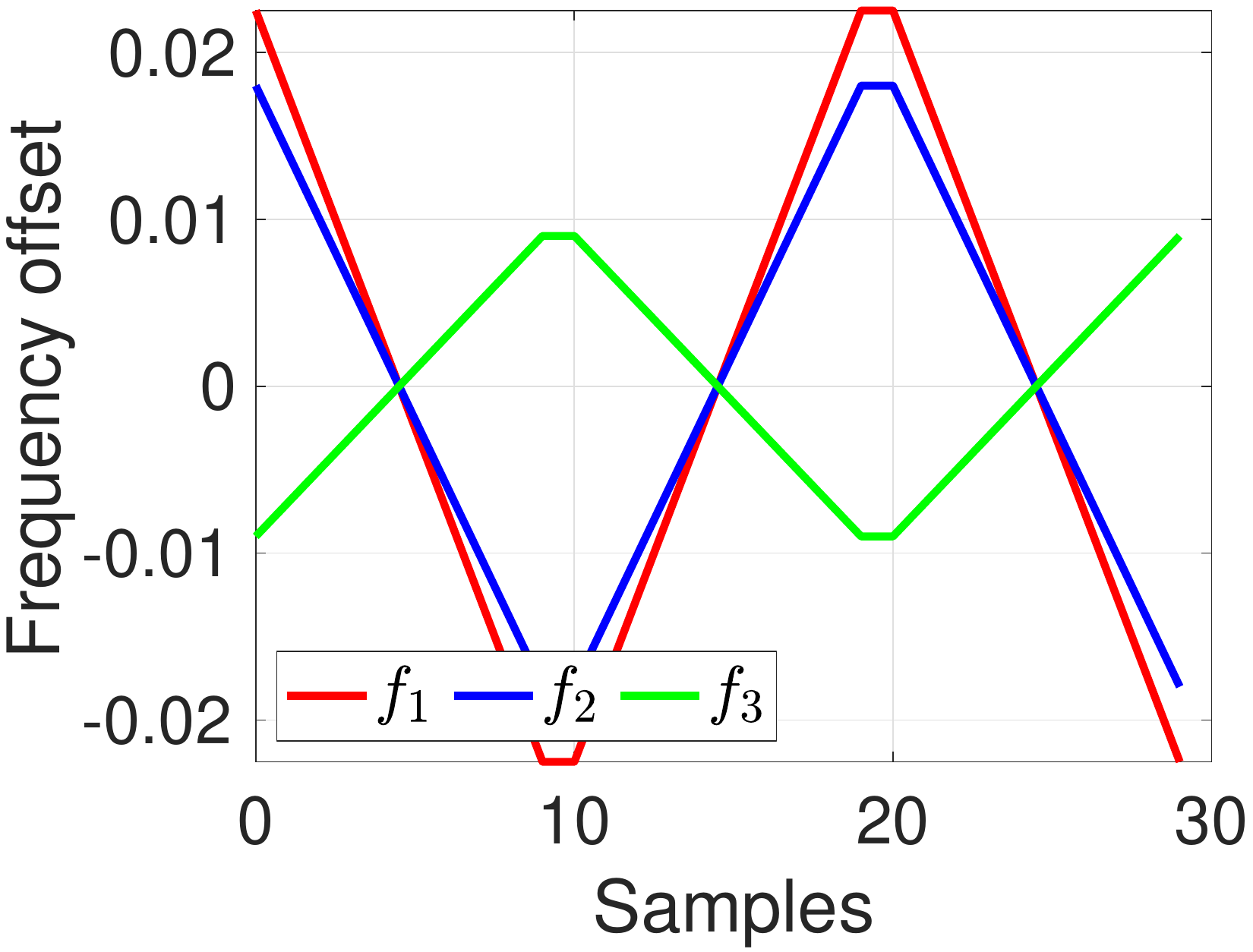}
\centerline{\footnotesize{(c) Zigzag frequency offset}}
\end{minipage}
~
\begin{minipage}{0.48\linewidth}
\centering
\includegraphics[width=1.67in]{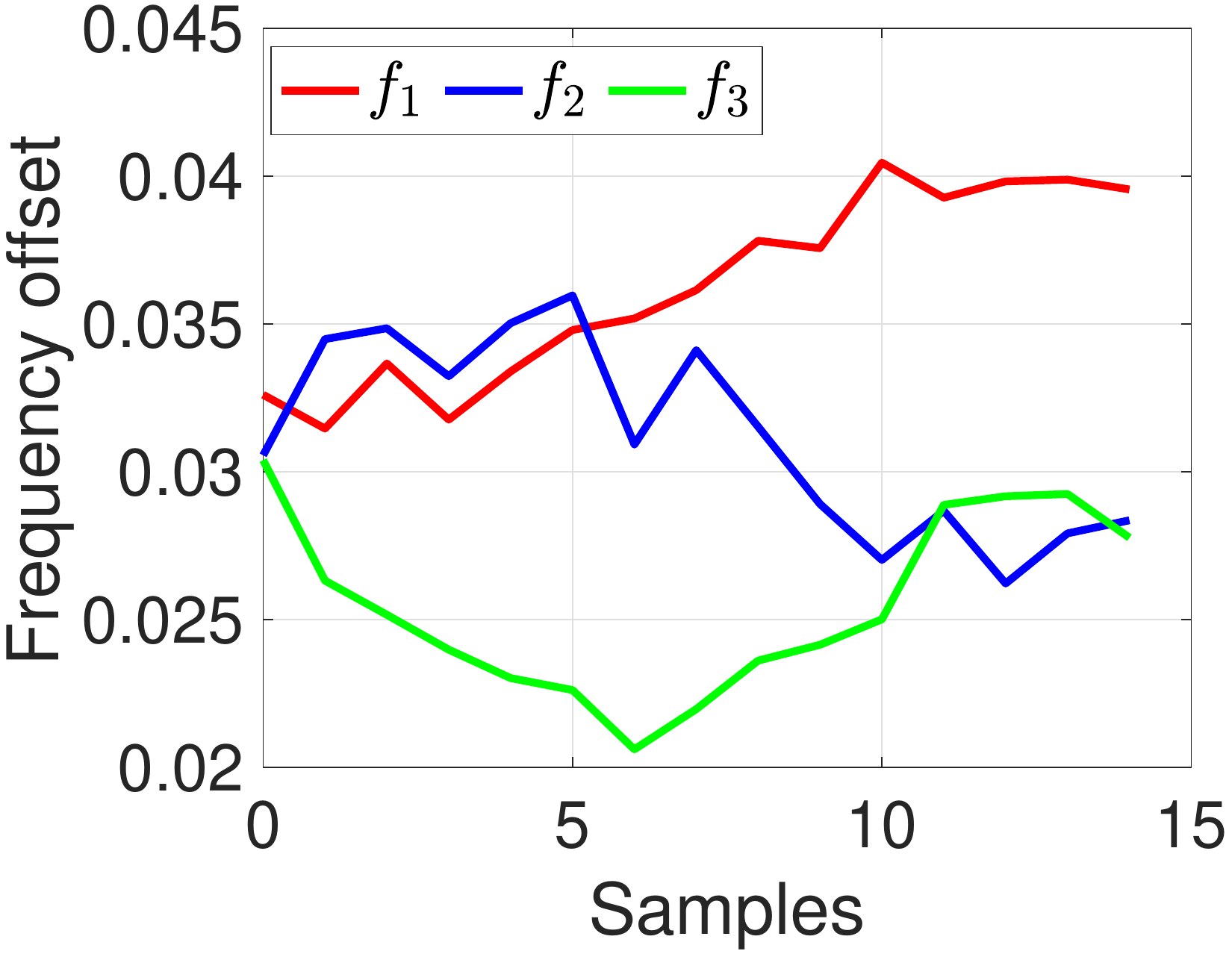}
\centerline{\footnotesize{(d) Random frequency offset}}
\end{minipage}
\caption{Four types of frequency offsets used in the experiments. ($M=15$ or 30)}
\label{fig:test_ANM_DPSS_SMI_2D_freqdrift}
\end{figure}

\begin{figure}[ht]
\begin{minipage}{0.48\linewidth}
\centering
\includegraphics[width=1.67in]{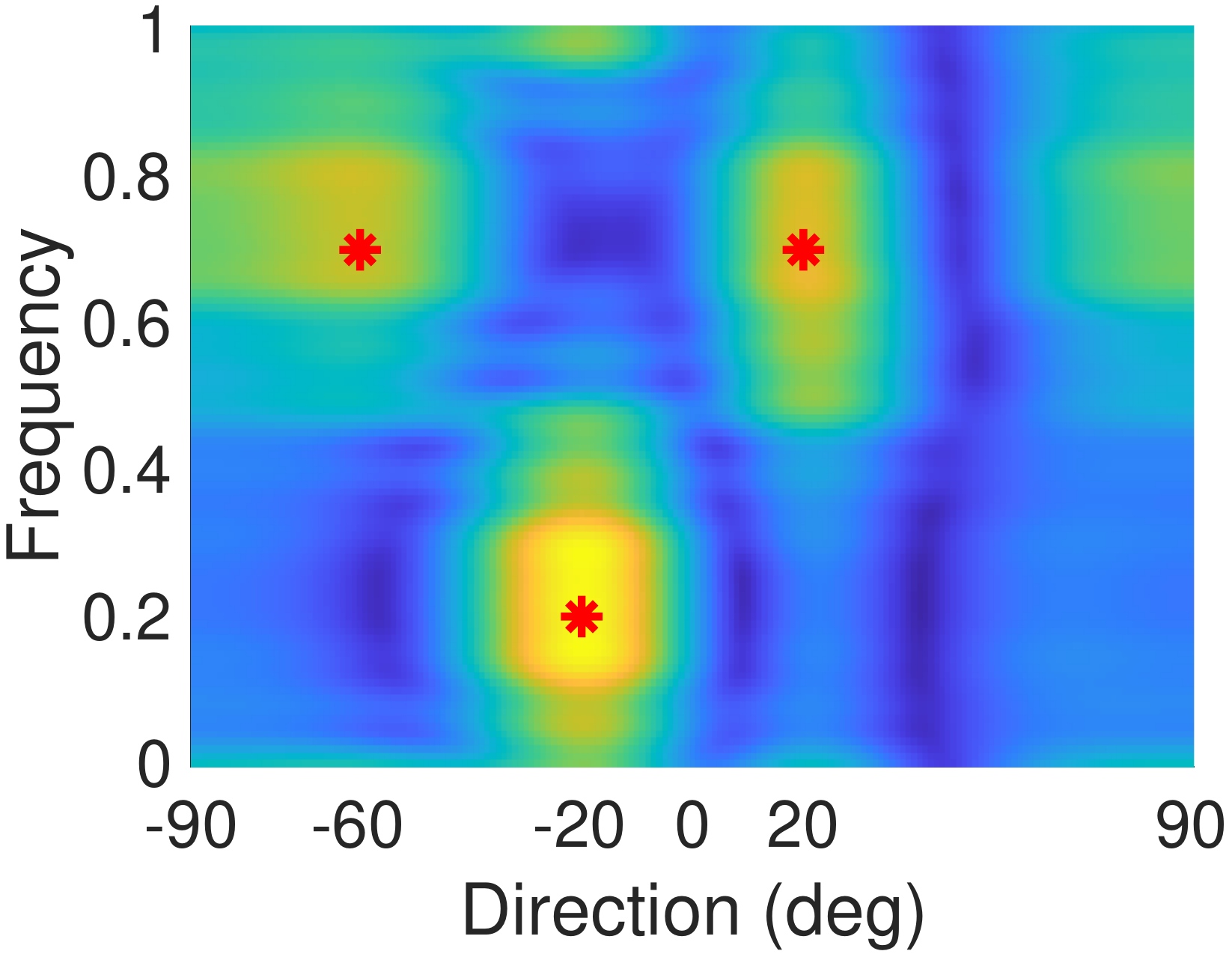}
\centerline{\footnotesize{(a) Static frequency offset}}
\end{minipage}
~
\begin{minipage}{0.48\linewidth}
\centering
\includegraphics[width=1.67in]{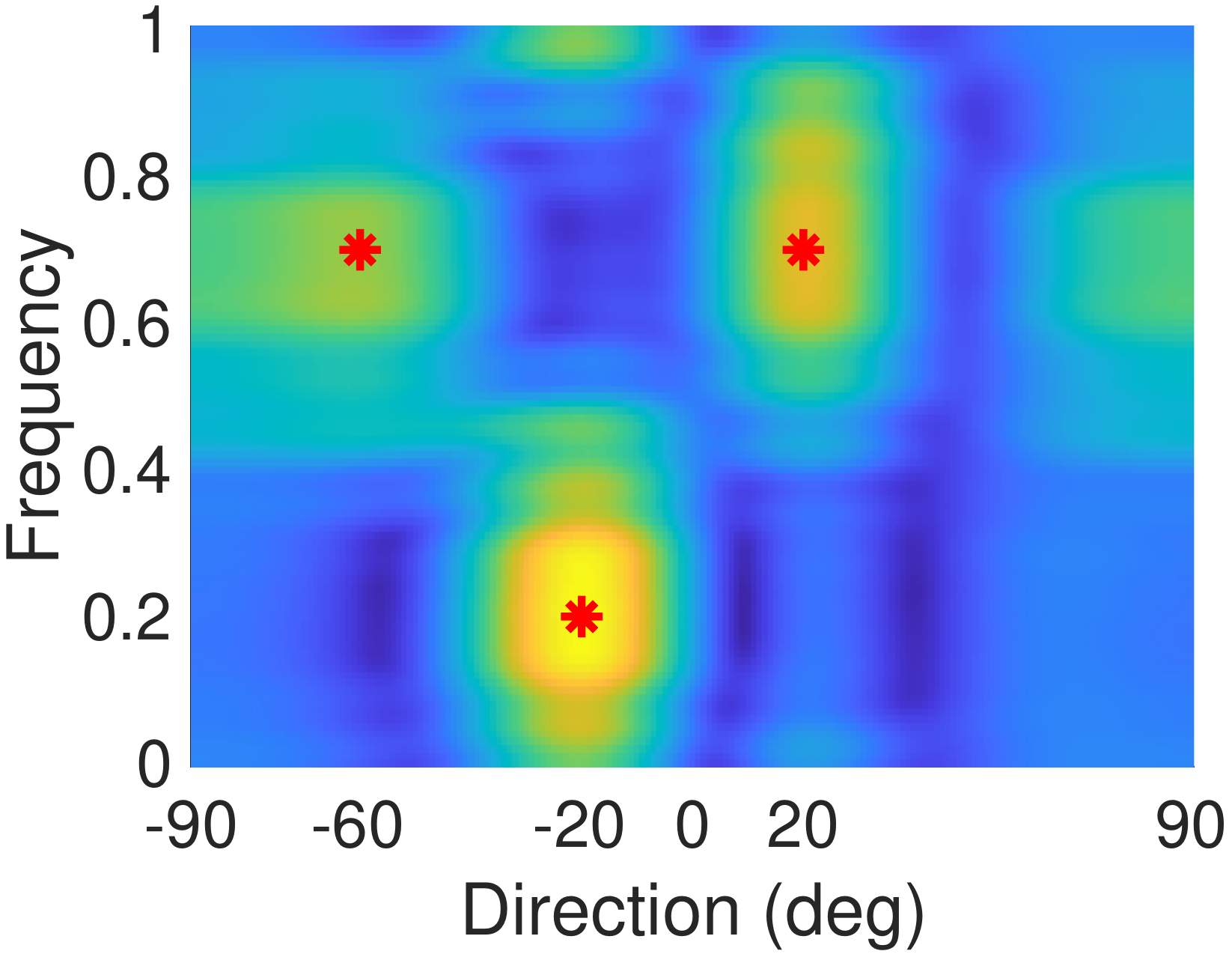}
\centerline{\footnotesize{(b) Linear frequency offset}}
\end{minipage}
\\
\begin{minipage}{0.48\linewidth}
\centering
\includegraphics[width=1.67in]{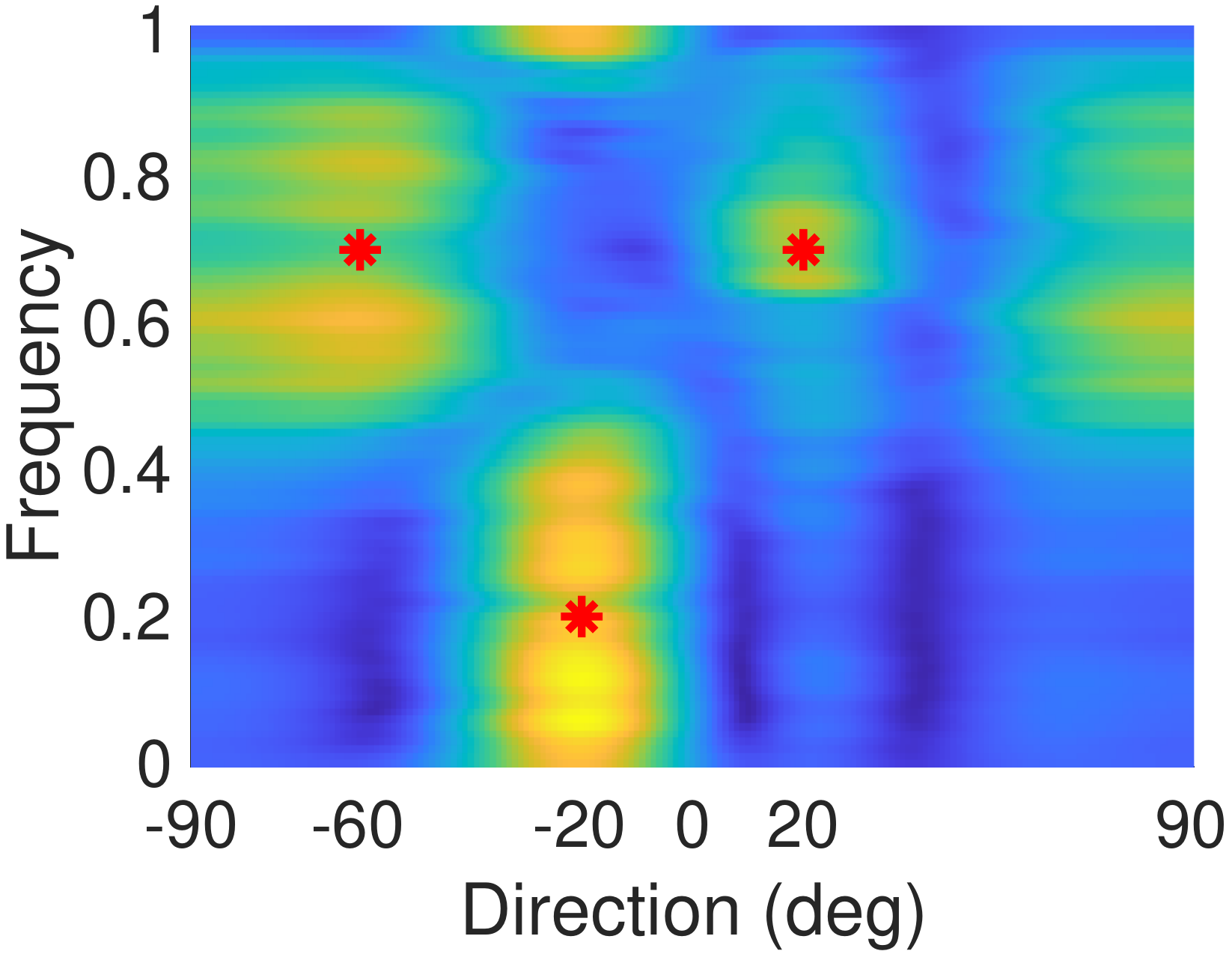}
\centerline{\footnotesize{(c) Zigzag frequency offset}}
\end{minipage}
~
\begin{minipage}{0.48\linewidth}
\centering
\includegraphics[width=1.67in]{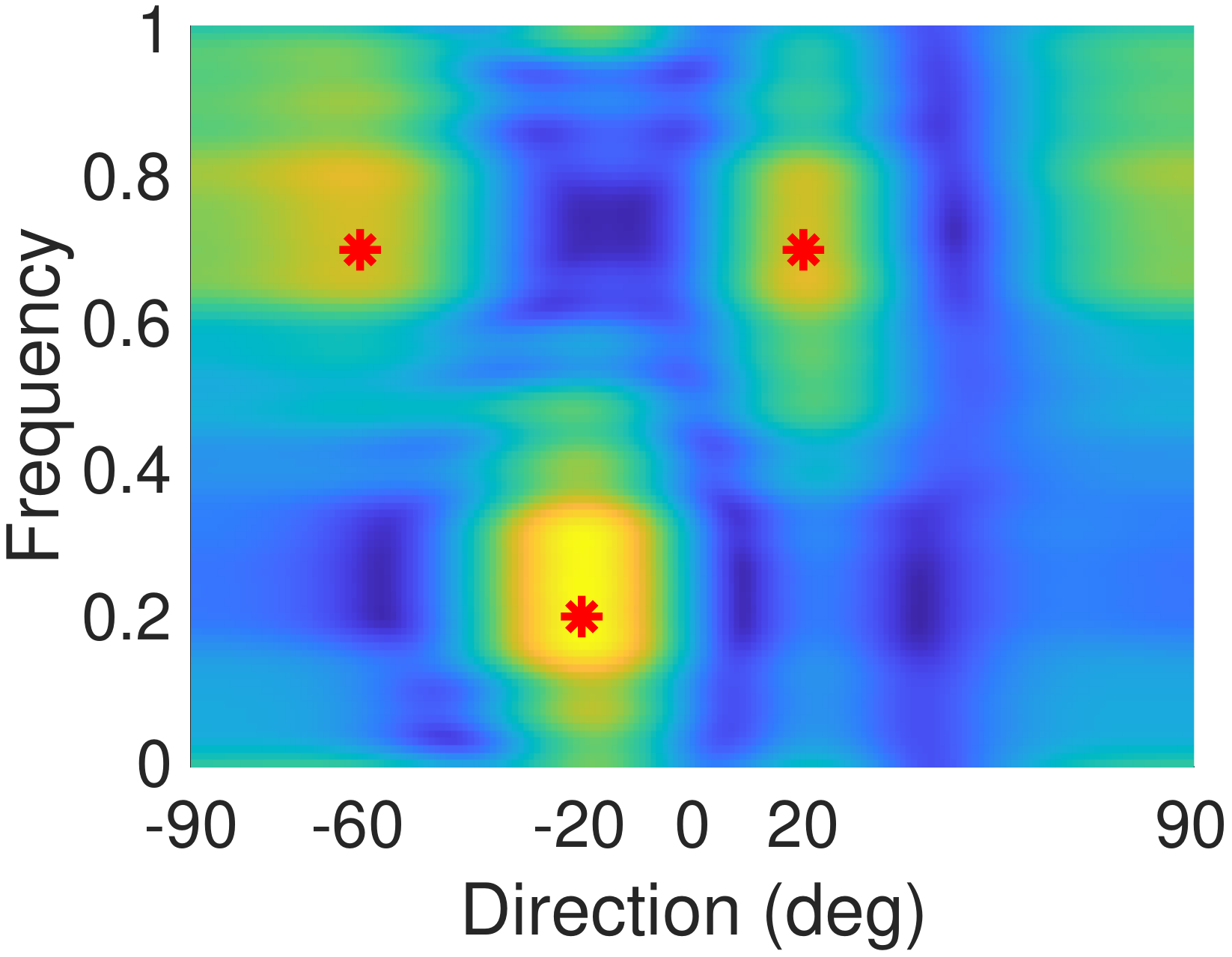}
\centerline{\footnotesize{(d) Random frequency offset}}
\end{minipage}
\caption{Dual polynomials obtained from the ``2D-ANM+DPSS+SMI'' method in the four types of frequency offsets. The red stars denote the true frequencies and angles in the desired signal and interferers.}
\label{fig:test_ANM_DPSS_SMI_2D_dual}
\end{figure}

\begin{figure}[ht]
\begin{minipage}{0.48\linewidth}
\centering
\includegraphics[width=1.67in]{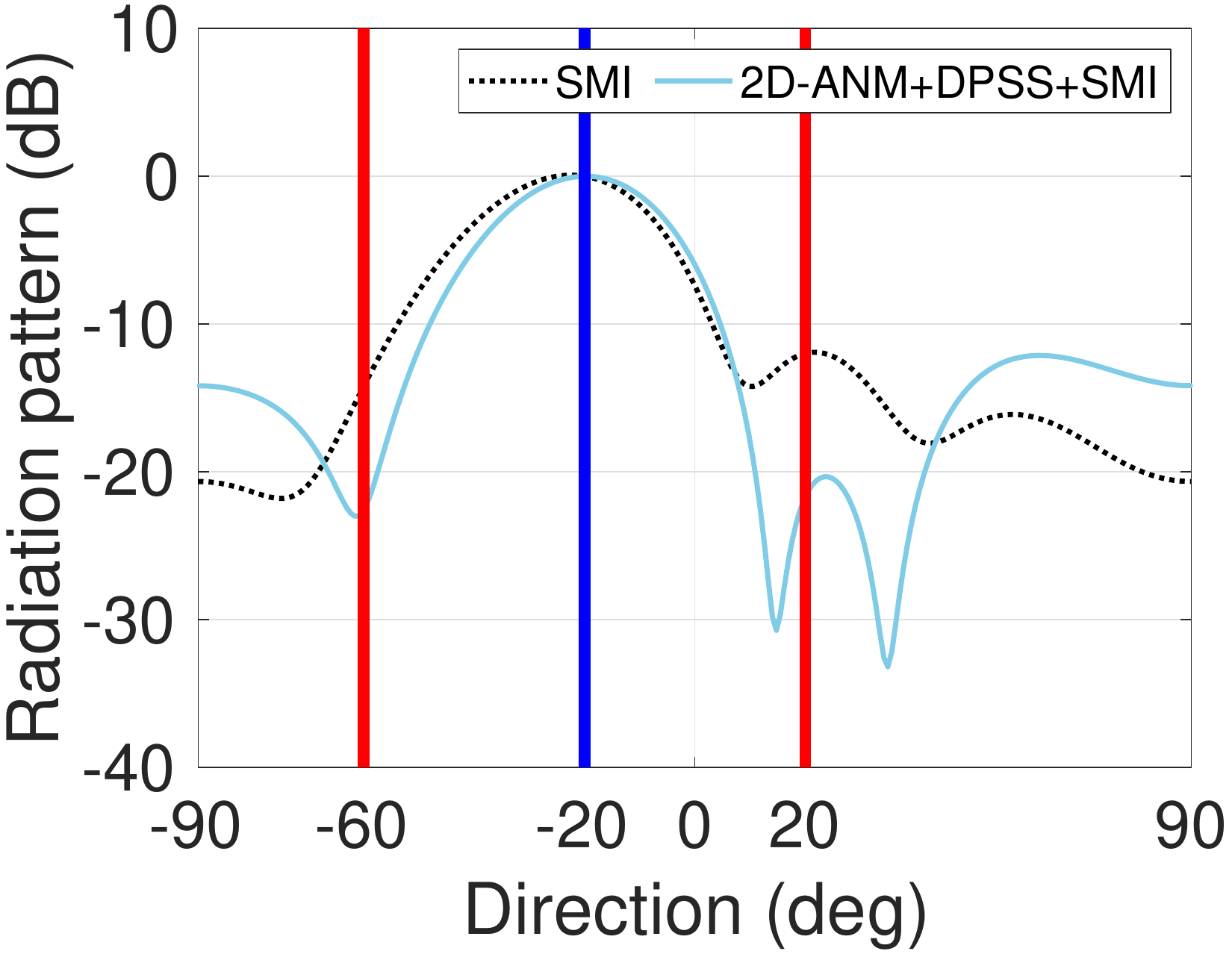}
\centerline{\footnotesize{(a) Static frequency offset}}
\end{minipage}
~
\begin{minipage}{0.48\linewidth}
\centering
\includegraphics[width=1.67in]{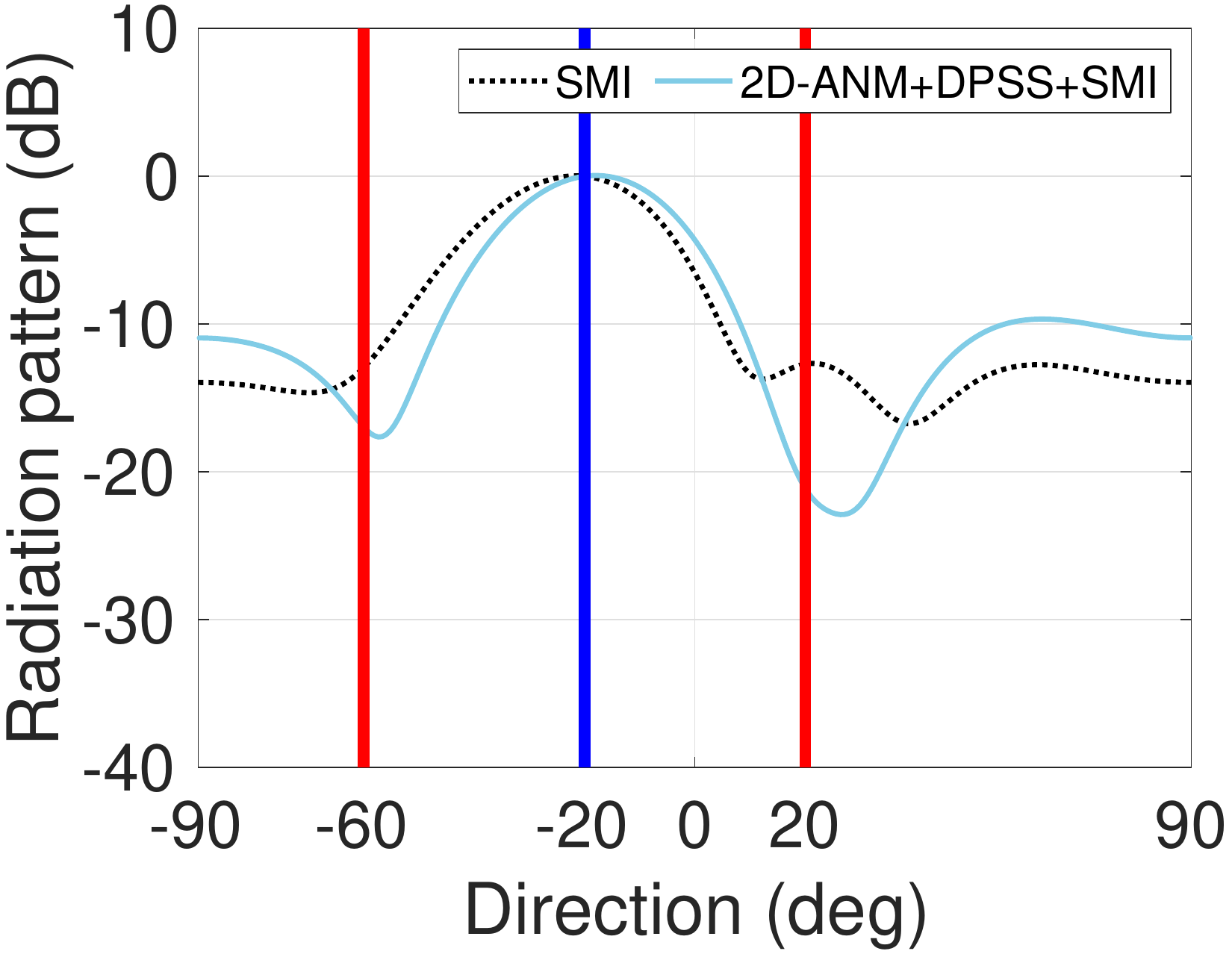}
\centerline{\footnotesize{(b) Linear frequency offset}}
\end{minipage}
\\
\begin{minipage}{0.48\linewidth}
\centering
\includegraphics[width=1.67in]{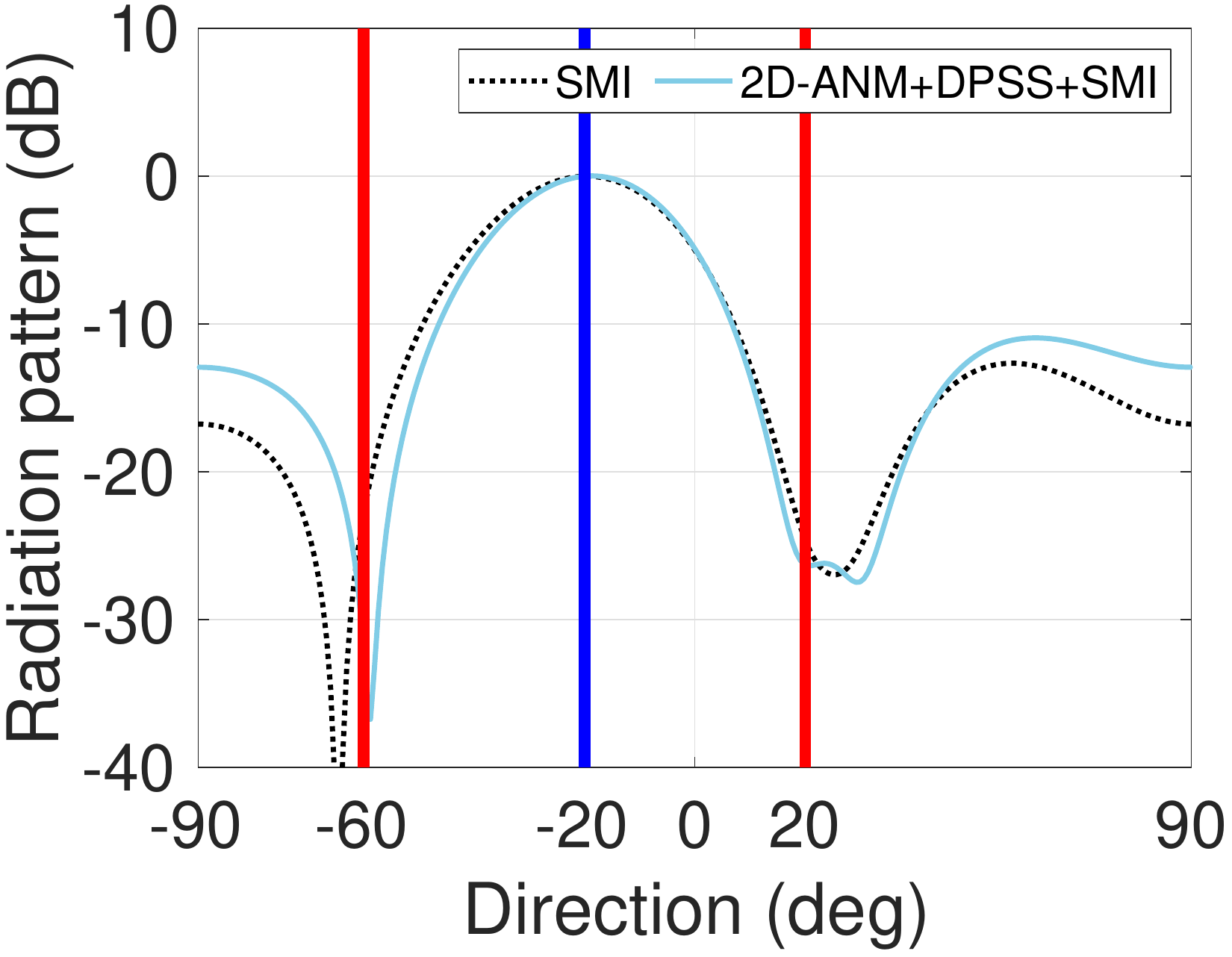}
\centerline{\footnotesize{(c) Zigzag frequency offset}}
\end{minipage}
~
\begin{minipage}{0.48\linewidth}
\centering
\includegraphics[width=1.67in]{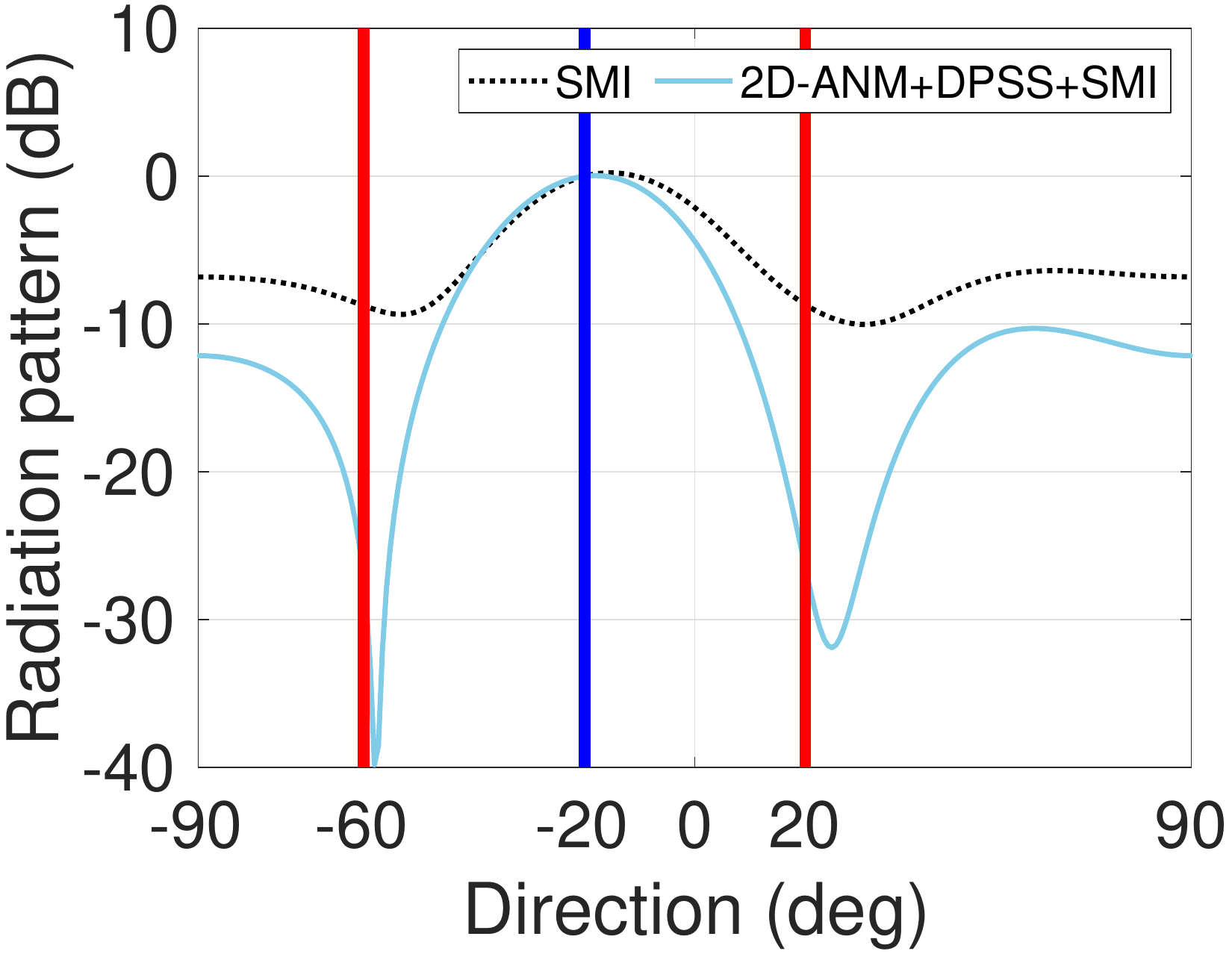}
\centerline{\footnotesize{(d) Random frequency offset}}
\end{minipage}
\caption{Interference cancellation with DBF in the four types of frequency offsets.}
\label{fig:test_ANM_DPSS_SMI_2D_radi}
\end{figure}

Finally, we repeat the experiments from the previous paragraph (comparing ``2D-ANM+DPSS+SMI'' with SMI) for 20 trials each and present the histogram of the radiation pattern (dB) evaluated at the two interferer directions ($-60^\circ$ and $20^\circ$) in Figure~\ref{fig:test_ANM_DPSS_SMI_2D_hg}. In each trial, we generate the four types of frequencies as follows: (a) Static frequency offset: the values are set as three random integers between 1 and 6 scaled by a factor of 0.01 or $-0.01$. (b) Linear frequency offset: the three slopes are set as three random integers between 1 and 6 scaled by a factor of 0.001 or $-0.001$. (c) Zigzag frequency offset: the three slopes are set as three random integers between 1 and 6 scaled by a factor of 0.001 or $-0.001$. (d) Random frequency offset: an independent random copy of the offsets shown in Figure~\ref{fig:test_ANM_DPSS_SMI_2D_freqdrift} (d). The other settings are same as above. It can be seen from Figure~\ref{fig:test_ANM_DPSS_SMI_2D_hg} that in most cases the proposed ``2D-ANM+DPSS+SMI'' method can achieve a lower radiation pattern at both directions $-60^\circ$ and $20^\circ$ with high probability.

\begin{figure}[ht]
\begin{minipage}{0.48\linewidth}
\centering
\includegraphics[width=1.67in]{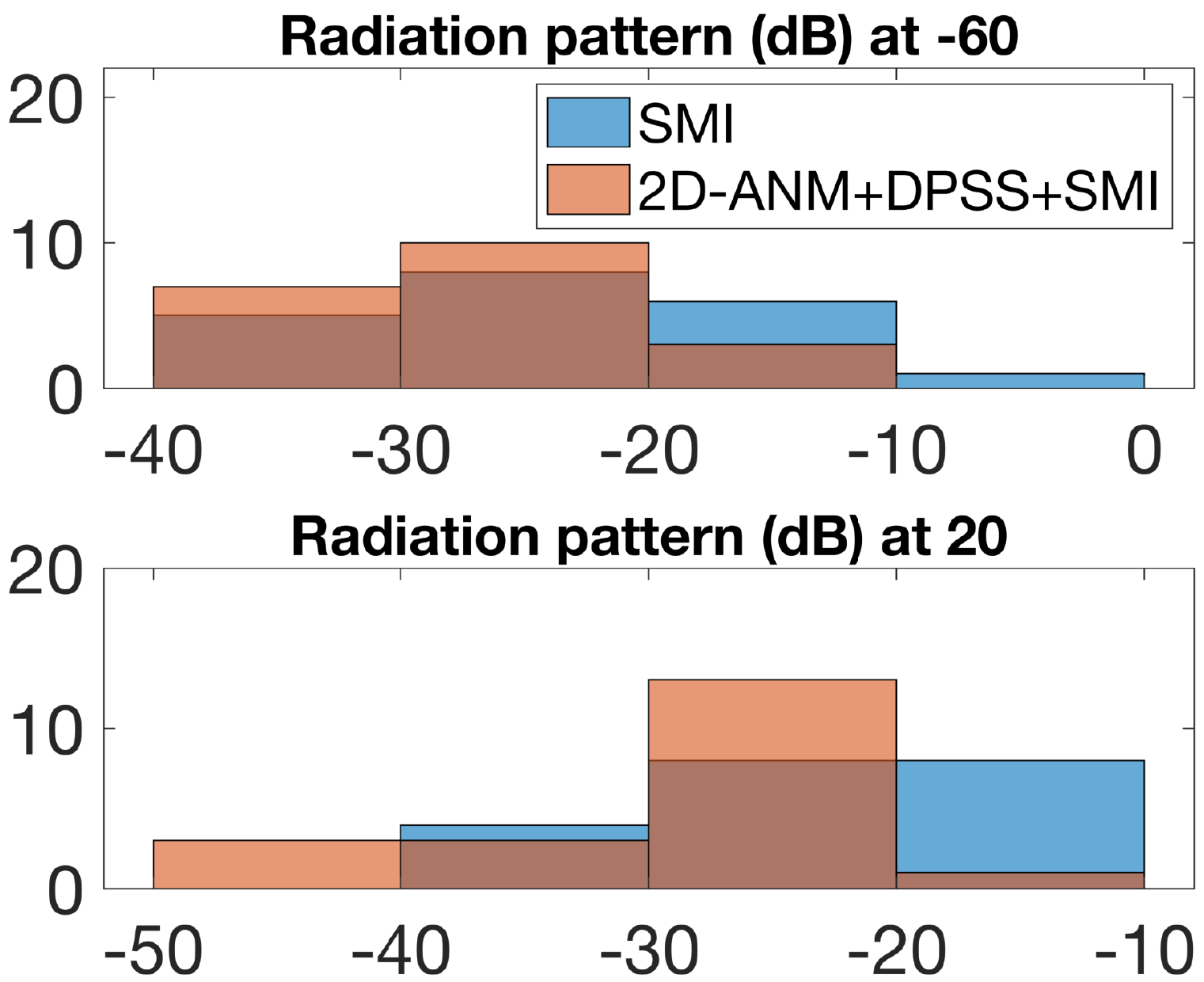}
\centerline{\footnotesize{(a) Static frequency offset}}
\end{minipage}
~
\begin{minipage}{0.48\linewidth}
\centering
\includegraphics[width=1.67in]{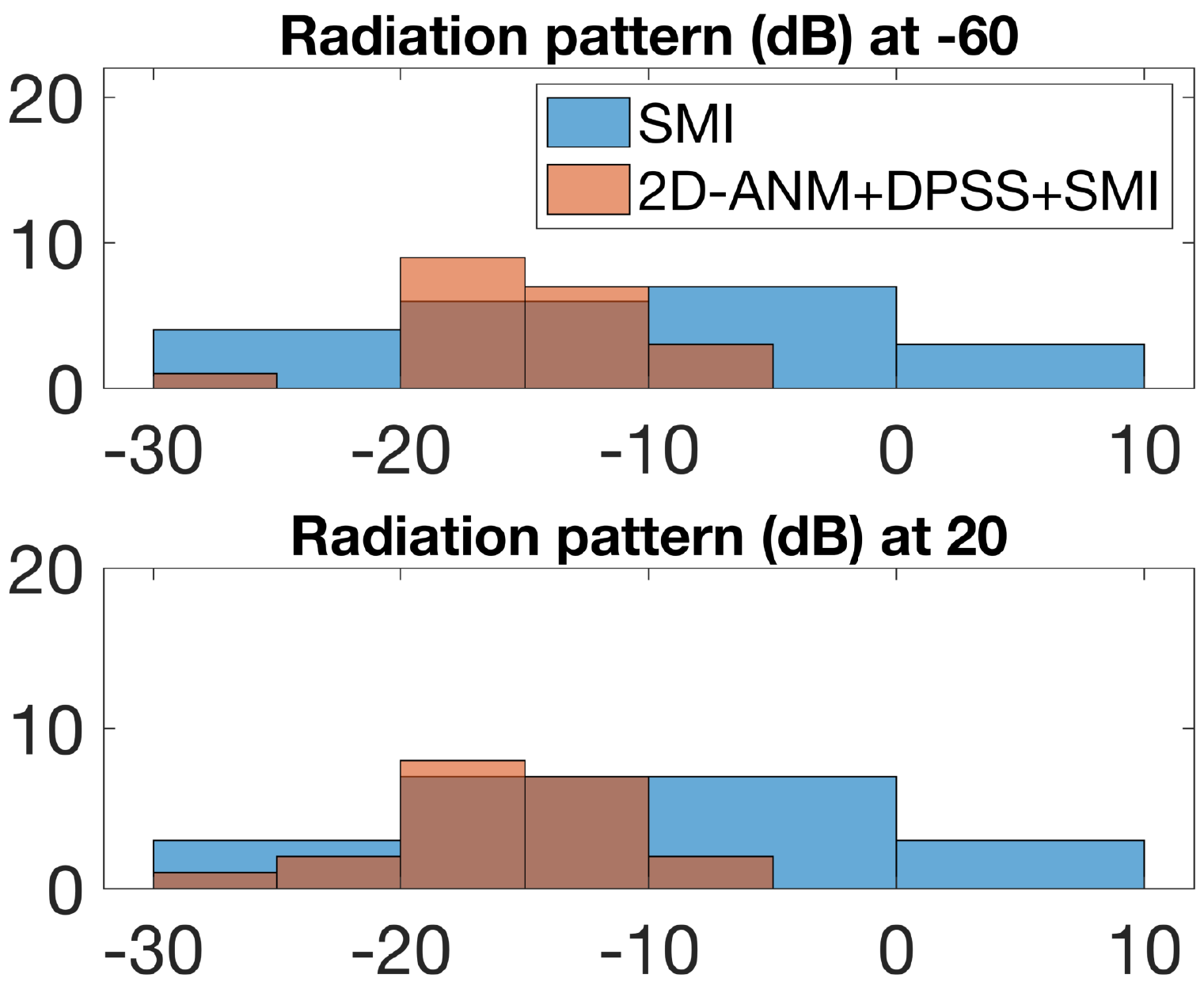}
\centerline{\footnotesize{(b) Linear frequency offset}}
\end{minipage}
\\
\begin{minipage}{0.48\linewidth}
\centering
\includegraphics[width=1.67in]{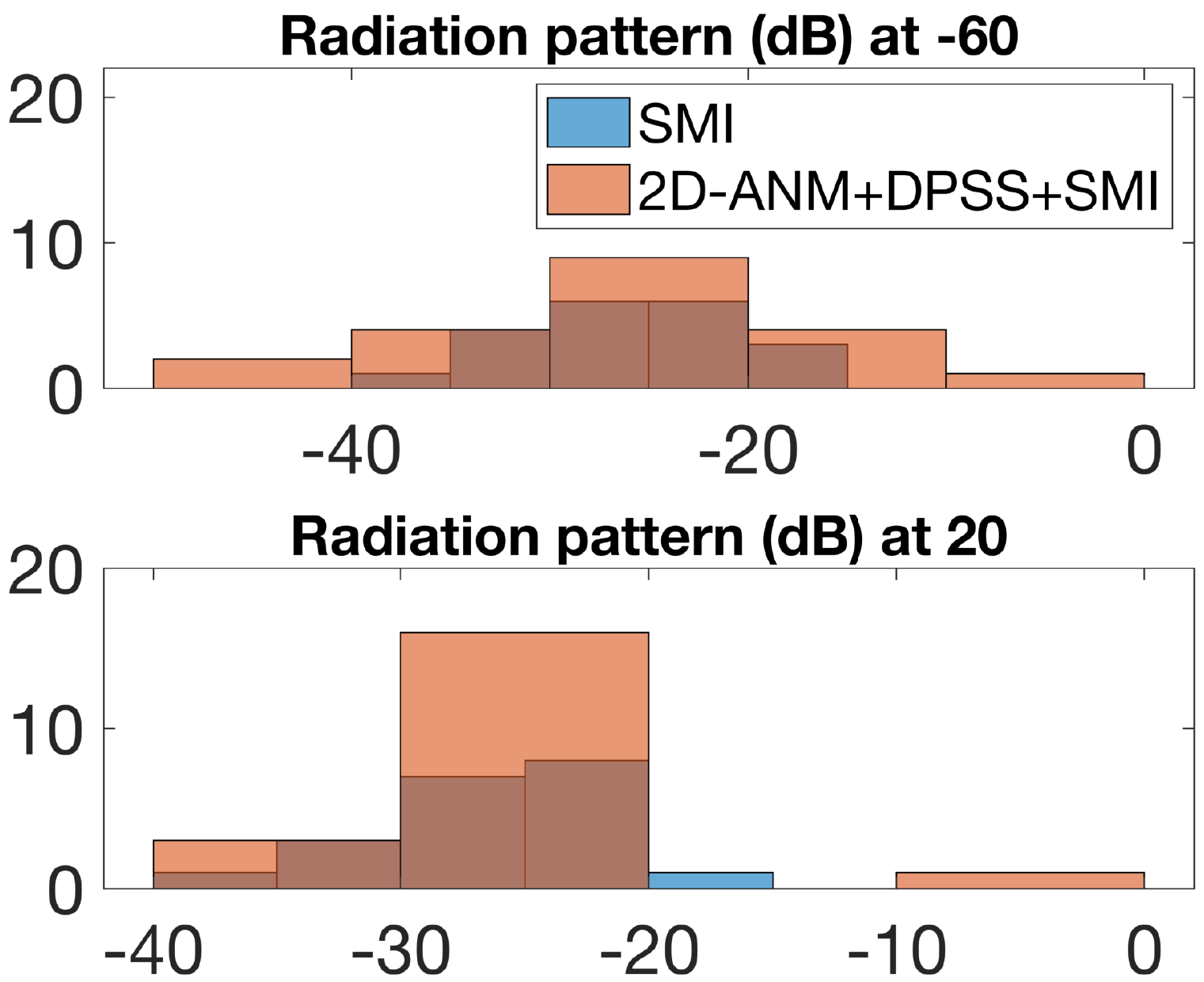}
\centerline{\footnotesize{(c) Zigzag frequency offset}}
\end{minipage}
~
\begin{minipage}{0.48\linewidth}
\centering
\includegraphics[width=1.67in]{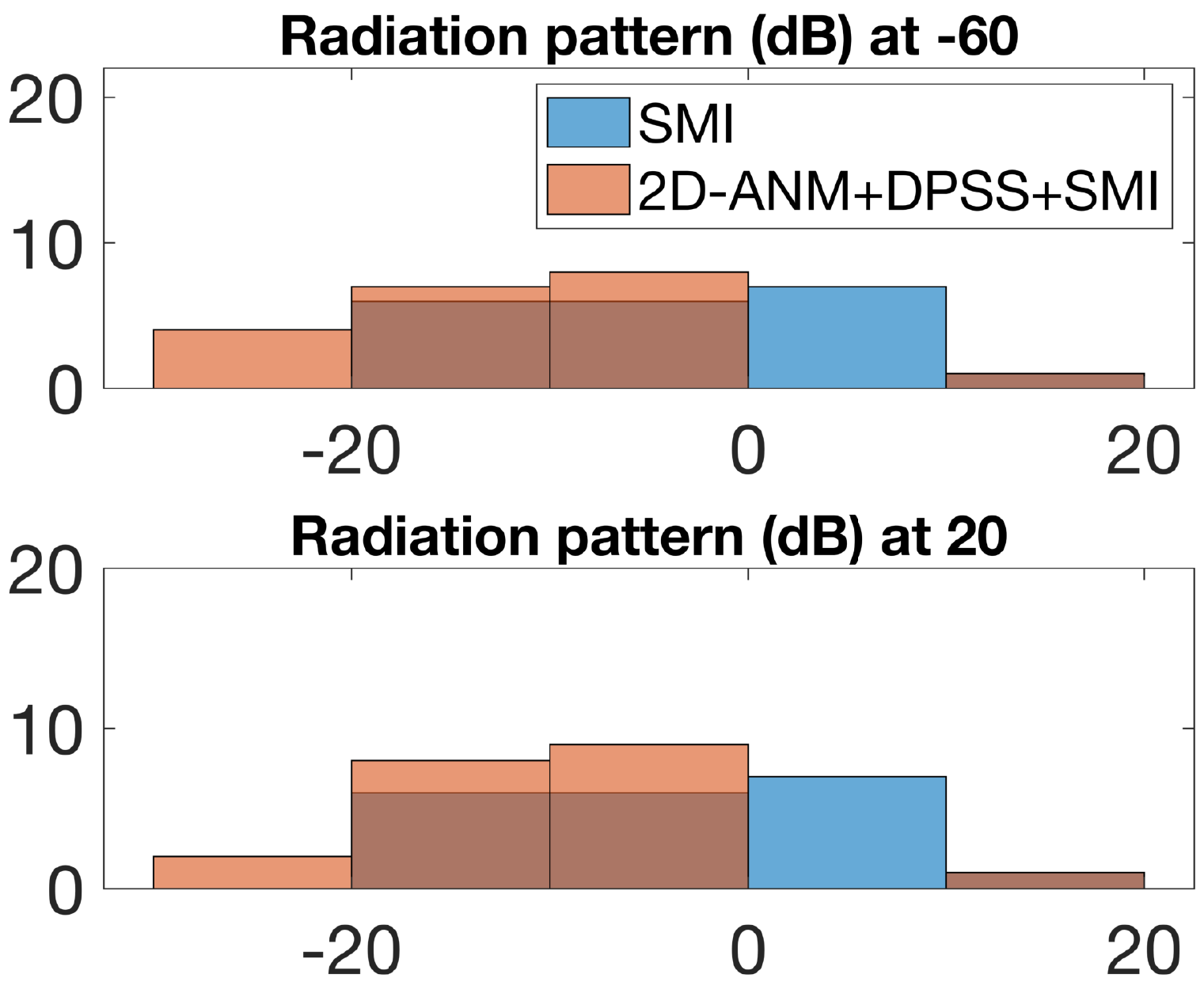}
\centerline{\footnotesize{(d) Random frequency offset}}
\end{minipage}
\caption{The histogram plots of radiation pattern (in dB) evaluated at $-60^\circ$ and $20^\circ$ in the four types of frequency offsets with 20 trials.}
\label{fig:test_ANM_DPSS_SMI_2D_hg}
\end{figure}


\section{Conclusion}
\label{sec:conc}

In this paper, we present two novel ANM-based methods in both 1D and 2D frameworks to address the dynamic problem of interference cancellation with time-varying frequency offset. By noting that solving the corresponding SDP is a computational burden in high-dimensional dataset, we also present a novel fast algorithm to approximately solve the 1D ANM optimization problem. Finally, we conduct a series of experiments to confirm the benefits of our 1D and 2D ANM frameworks compared to the conventional SMI, and we illustrate the computational speedup provided by our fast 1D algorithm. We leave the work of developing a fast algorithm for the 2D ANM method as our future work.

\section*{Acknowledgement}

MW and SL were supported by NSF grant CCF-1704204.

\ifCLASSOPTIONcaptionsoff
  \newpage
\fi

\bibliographystyle{ieeetr}
\bibliography{DBF_dynamics}

\begin{thebibliography}{10}

\bibitem{ZLiIoT2019}
Z.~Li, Y.~Liu, K.~G. Shin, J.~Liu, and Z.~Yan, ``Interference steering to
  manage interference in iot,'' {\em IEEE Internet of Things Journal}, vol.~6,
  no.~6, pp.~10458--10471, 2019.

\bibitem{ChettriIoT2020}
L.~Chettri and R.~Bera, ``A comprehensive survey on {I}nternet of {T}hings
  ({IoT}) toward 5{G} wireless systems,'' {\em IEEE Internet of Things
  Journal}, vol.~7, no.~1, pp.~16--32, 2019.

\bibitem{widrow1967adaptive}
B.~Widrow, P.~Mantey, L.~Griffiths, and B.~Goode, ``Adaptive antenna systems,''
  {\em Proceedings of the IEEE}, vol.~55, no.~12, pp.~2143--2159, 1967.

\bibitem{van2004optimum}
H.~L. Van~Trees, {\em Optimum array processing: Part {IV} of detection,
  estimation, and modulation theory}.
\newblock John Wiley \& Sons, 2004.

\bibitem{Mailloux2017}
R.~J. Mailloux, {\em Phased array antenna handbook}.
\newblock Artech house, 2017.

\bibitem{DBFYang2018}
B.~Yang, Z.~Yu, J.~Lan, R.~Zhang, J.~Zhou, and W.~Hong, ``Digital
  beamforming-based massive {MIMO} transceiver for 5{G} millimeter-wave
  communications,'' {\em IEEE Transactions on Microwave Theory and Techniques},
  vol.~66, no.~7, pp.~3403--3418, 2018.

\bibitem{DBFRoh2014}
W.~Roh, J.-Y. Seol, J.~Park, B.~Lee, J.~Lee, Y.~Kim, J.~Cho, K.~Cheun, and
  F.~Aryanfar, ``Millimeter-wave beamforming as an enabling technology for 5{G}
  cellular communications: Theoretical feasibility and prototype results,''
  {\em IEEE Communications Magazine}, vol.~52, no.~2, pp.~106--113, 2014.

\bibitem{gaydos2018SDR}
D.~Gaydos, P.~Nayeri, and R.~Haupt, ``Experimental demonstration of a
  software-defined-radio adaptive beamformer,'' in {\em 2018 15th European
  Radar Conference (EuRAD)}, pp.~561--564, IEEE, 2018.

\bibitem{monzingo2011introduction}
B.~Monzingo, R.~Haupt, and T.~Miller, {\em Introduction to adaptive arrays}.
\newblock The Institution of Engineering and Technology, 2011.

\bibitem{Blind3}
A.-J. Van Der~Veen, ``Algebraic methods for deterministic blind beamforming,''
  {\em Proceedings of the IEEE}, vol.~86, no.~10, pp.~1987--2008, 1998.

\bibitem{Blind1}
Q.~Wu and K.~M. Wong, ``Blind adaptive beamforming for cyclostationary
  signals,'' {\em IEEE Transactions on Signal Processing}, vol.~44, no.~11,
  pp.~2757--2767, 1996.

\bibitem{Blind2}
E.~Gonen and J.~M. Mendel, ``Applications of cumulants to array processing.
  {III}. {B}lind beamforming for coherent signals,'' {\em IEEE Transactions on
  Signal Processing}, vol.~45, no.~9, pp.~2252--2264, 1997.

\bibitem{Blind4}
C.~Coviello and L.~Sibul, ``Blind source separation and beamforming: algebraic
  technique analysis,'' {\em IEEE Transactions on Aerospace and Electronic
  Systems}, vol.~40, no.~1, pp.~221--235, 2004.

\bibitem{PSMI}
J.~Liu, W.~Liu, H.~Liu, B.~Chen, X.-G. Xia, and F.~Dai, ``Average {SINR}
  calculation of a persymmetric sample matrix inversion beamformer,'' {\em IEEE
  Transactions on Signal Processing}, vol.~64, no.~8, pp.~2135--2145, 2015.

\bibitem{gaydos2019adaptive}
D.~Gaydos, P.~Nayeri, and R.~Haupt, ``Adaptive beamforming in high-interference
  environments using a software-defined radio array,'' in {\em 2019 IEEE
  International Symposium on Antennas and Propagation and USNC-URSI Radio
  Science Meeting}, pp.~1501--1502, IEEE, 2019.

\bibitem{SMIDL1}
M.~W. Ganz, S.~L. Wilson, and R.~L. Moses, ``Convergence of the {SMI} and the
  diagonally loaded {SMI} algorithms with weak interference,'' {\em IEEE
  Transactions on Antennas and Propagation}, vol.~38, pp.~394--399, 1990.

\bibitem{SMIDL2}
R.~L. Dilsavor and R.~L. Moses, ``Analysis of modified {SMI} method for
  adaptive array weight control,'' {\em IEEE Transactions on Signal
  Processing}, vol.~41, no.~2, pp.~721--726, 1993.

\bibitem{SMIVL}
J.~Gu, ``Robust beamforming based on variable loading,'' {\em Electronics
  Letters}, vol.~41, no.~2, pp.~55--56, 2005.

\bibitem{SMIIVL}
X.~Li, D.-W. Wang, X.~Ma, and Z.~Xiong, ``Robust adaptive beamforming using
  iterative variable loaded sample matrix inverse,'' {\em Electronics Letters},
  vol.~54, no.~9, pp.~546--548, 2018.

\bibitem{li2020adaptiveACES}
S.~Li, D.~Gaydos, P.~Nayeri, and M.~B. Wakin, ``Adaptive interference
  cancellation using atomic norm minimization,'' in {\em 2020 International
  Applied Computational Electromagnetics Society Symposium (ACES)}, pp.~1--2,
  IEEE, 2020.

\bibitem{li2020adaptiveAWPL}
S.~Li, D.~Gaydos, P.~Nayeri, and M.~B. Wakin, ``Adaptive interference
  cancellation using atomic norm minimization and denoising,'' {\em IEEE
  Antennas and Wireless Propagation Letters}, vol.~19, no.~12, pp.~2349--2353,
  2020.

\bibitem{candes2006compressive}
E.~J. Cand{\`e}s {\em et~al.}, ``Compressive sampling,'' in {\em Proceedings of
  the International Congress of Mathematicians}, vol.~3, pp.~1433--1452,
  Madrid, Spain, 2006.

\bibitem{candes2008introduction}
E.~J. Cand{\`e}s and M.~B. Wakin, ``An introduction to compressive sampling,''
  {\em IEEE Signal Processing Magazine}, vol.~25, no.~2, pp.~21--30, 2008.

\bibitem{tang2013compressed}
G.~Tang, B.~N. Bhaskar, P.~Shah, and B.~Recht, ``Compressed sensing off the
  grid,'' {\em IEEE Transactions on Information Theory}, vol.~59, no.~11,
  pp.~7465--7490, 2013.

\bibitem{li2015off}
Y.~Li and Y.~Chi, ``Off-the-grid line spectrum denoising and estimation with
  multiple measurement vectors,'' {\em IEEE Transactions on Signal Processing},
  vol.~64, no.~5, pp.~1257--1269, 2015.

\bibitem{yang2016exact}
Z.~Yang and L.~Xie, ``Exact joint sparse frequency recovery via optimization
  methods,'' {\em IEEE Transactions on Signal Processing}, vol.~64, no.~19,
  pp.~5145--5157, 2016.

\bibitem{slepian1978prolate}
D.~Slepian, ``Prolate spheroidal wave functions, {F}ourier analysis, and
  uncertainty—{V}: The discrete case,'' {\em Bell System Technical Journal},
  vol.~57, no.~5, pp.~1371--1430, 1978.

\bibitem{davenport2012compressive}
M.~A. Davenport and M.~B. Wakin, ``Compressive sensing of analog signals using
  discrete prolate spheroidal sequences,'' {\em Applied and Computational
  Harmonic Analysis}, vol.~33, no.~3, pp.~438--472, 2012.

\bibitem{zhu2017approximating}
Z.~Zhu and M.~B. Wakin, ``Approximating sampled sinusoids and multiband signals
  using multiband modulated {DPSS} dictionaries,'' {\em Journal of Fourier
  Analysis and Applications}, vol.~23, no.~6, pp.~1263--1310, 2017.

\bibitem{chandrasekaran2012convex}
V.~Chandrasekaran, B.~Recht, P.~A. Parrilo, and A.~S. Willsky, ``The convex
  geometry of linear inverse problems,'' {\em Foundations of Computational
  mathematics}, vol.~12, no.~6, pp.~805--849, 2012.

\bibitem{candes2014towards}
E.~J. Cand{\`e}s and C.~Fernandez-Granda, ``Towards a mathematical theory of
  super-resolution,'' {\em Communications on Pure and Applied Mathematics},
  vol.~67, no.~6, pp.~906--956, 2014.

\bibitem{li2018atomic}
S.~Li, D.~Yang, G.~Tang, and M.~B. Wakin, ``Atomic norm minimization for modal
  analysis from random and compressed samples,'' {\em IEEE Transactions on
  Signal Processing}, vol.~66, no.~7, pp.~1817--1831, 2018.

\bibitem{xie2017radar}
Y.~Xie, S.~Li, G.~Tang, and M.~B. Wakin, ``Radar signal demixing via convex
  optimization,'' in {\em 2017 22nd International Conference on Digital Signal
  Processing (DSP)}, pp.~1--5, IEEE, 2017.

\bibitem{li2019atomic}
S.~Li, M.~B. Wakin, and G.~Tang, ``Atomic norm denoising for complex
  exponentials with unknown waveform modulations,'' {\em IEEE Transactions on
  Information Theory}, vol.~66, no.~6, pp.~3893--3913, 2020.

\bibitem{wakin2012study}
M.~B. Wakin, ``A study of the temporal bandwidth of video and its implications
  in compressive sensing,'' {\em Colorado School of Mines Technical Report},
  pp.~1--50, 2012.

\bibitem{helland2019super}
J.~Helland, M.~B. Wakin, and G.~Tang, ``A super-resolution algorithm for
  extended target localization,'' in {\em 2019 IEEE 8th International Workshop
  on Computational Advances in Multi-Sensor Adaptive Processing (CAMSAP)},
  pp.~386--390, IEEE, 2019.

\bibitem{grant2008cvx}
M.~Grant, S.~Boyd, and Y.~Ye, ``{CVX}: Matlab software for disciplined convex
  programming,'' 2008.

\bibitem{tutuncu2001sdpt3}
R.~T{\"u}t{\"u}nc{\"u}, K.~Toh, and M.~Todd, ``{SDPT}3—a {M}atlab software
  package for semidefinite-quadratic-linear programming, version 3.0,'' {\em
  Web page http://www. math. nus. edu. sg/mattohkc/sdpt3. html}, 2001.

\bibitem{bhaskar2013atomic}
B.~N. Bhaskar, G.~Tang, and B.~Recht, ``Atomic norm denoising with applications
  to line spectral estimation,'' {\em IEEE Transactions on Signal Processing},
  vol.~61, no.~23, pp.~5987--5999, 2013.

\bibitem{bertsekas1989parallel}
D.~P. Bertsekas and J.~N. Tsitsiklis, {\em Parallel and distributed
  computation: numerical methods}, vol.~23.
\newblock Prentice Hall Englewood Cliffs, NJ, 1989.

\bibitem{boyd2011distributed}
S.~Boyd, N.~Parikh, and E.~Chu, {\em Distributed optimization and statistical
  learning via the alternating direction method of multipliers}.
\newblock Now Publishers Inc, 2011.

\bibitem{hansen2019fast}
T.~L. Hansen and T.~L. Jensen, ``A fast interior-point method for atomic norm
  soft thresholding,'' {\em Signal Processing}, vol.~165, pp.~7--19, 2019.

\bibitem{wang2018ivdst}
Y.~Wang and Z.~Tian, ``{IVDST}: A fast algorithm for atomic norm minimization
  in line spectral estimation,'' {\em IEEE Signal Processing Letters}, vol.~25,
  no.~11, pp.~1715--1719, 2018.

\bibitem{liu2020iterative}
Y.~Liu, Z.~Chu, and Y.~Yang, ``Iterative vandermonde decomposition and
  shrinkage-thresholding based two-dimensional grid-free compressive
  beamforming,'' {\em The Journal of the Acoustical Society of America},
  vol.~148, no.~3, pp.~EL301--EL306, 2020.

\end{thebibliography}

%
%
%
%

\end{document}